\documentclass{article}

\usepackage{times}
\usepackage{amssymb}
\usepackage{amsmath}
\usepackage{amsthm}
\usepackage{bm}
\usepackage{natbib}
\usepackage{tikz}
\usepackage{tkz-euclide}
\usetkzobj{all}
\usepackage{chngcntr}
\usepackage{verbatim}

\usepackage[accepted]{icml2016}

\newcommand{\methodname}{{\textsc{dynaSAGA}}}
\newcommand{\x}{{\bm{x}}} 
\newcommand{\ts}{{\bm{t}}} 
\newcommand{\y}{{\bm{y}}} 
\newcommand{\X}{\mathcal{X}} 
\newcommand{\w}{{\bm{w}}} 
\newcommand{\risk}{{\cal R}} 
\newcommand{\E}{{\mathbf{E}}} 
\newcommand{\bound}{{\cal H}} 
\renewcommand{\S}{{\cal S}} 
\newcommand{\T}{{\cal T}} 
\newcommand{\Pdata}{{\mathcal P}} 
\renewcommand{\Re}{{\mathbb R}} 
\newcommand{\fclass}{{\mathcal F}} 
\newcommand{\bigO}{O} 
\newcommand{\U}{{\mathbf U}} 
\newcommand{\initerror}{\xi}

\renewcommand{\a}{\bm{a}}

\DeclareMathOperator*{\argmin}{\arg\min}

\newtheorem{theorem}{Theorem}
\newtheorem{lemma}[theorem]{Lemma}
\newtheorem{proposition}[theorem]{Proposition}
\newtheorem{corollary}[theorem]{Corollary}

\graphicspath{ {./images/} }

\author{}

\icmltitlerunning{Starting Small -- Learning with Adaptive Sample Sizes}

\begin{document}

\twocolumn[
\icmltitle{Starting Small -- Learning with Adaptive Sample Sizes}
\icmlauthor{Hadi Daneshmand}{hadi.daneshmand@inf.ethz.ch}
\icmlauthor{Aurelien Lucchi}{aurelien.lucchi@inf.ethz.ch}
\icmlauthor{Thomas Hofmann}{thomas.hofmann@inf.ethz.ch}
\icmladdress{Department of Computer Science, ETH Zurich, Switzerland}

\icmlkeywords{Stochastic Optimization, }

\vskip 0.3in
]

\begin{abstract}
For many machine learning problems, data is abundant and it may be prohibitive to make multiple passes through the full training set. In this context, we investigate strategies for dynamically increasing the effective sample size, when using iterative methods such as stochastic gradient descent. Our interest is motivated by the rise of variance-reduced methods, which achieve linear convergence rates that scale favorably for smaller sample sizes. Exploiting this feature, we show -- theoretically and empirically -- how to obtain significant speed-ups with a novel algorithm that reaches statistical accuracy on an $n$-sample in $2n$, instead of $n \log n$ steps. 
\end{abstract}


\section{Introduction}

In empirical risk minimization (ERM) \cite{vapnik1998statistical} the training set $\S$ is used to define a sample risk $\risk_\S$, which is then minimized with regard to a pre-defined function class. One effectively equates learning algorithms with optimization algorithms. However, for all practical purposes an approximate solution of $\risk_\S$ will be sufficient, as long as the optimization error is small relative to the statistical accuracy at sample size $n := |\S|$. This is important for massive data sets, where optimization to numerical precision is infeasible. Instead of performing early stopping on black-box optimization, one ought to understand the trade-offs between statistical and computational accuracy, cf.~\cite{chandrasekaran2013computational}. In this paper, we investigate a much neglected facet of this topic, namely how to dynamically control the effective sample size in optimization. 

Many  large-scale optimization algorithms are iterative: they use sampled or aggregated data to perform a sequence of update steps. This includes the popular family of gradient descent methods. Often, the computational complexity increases with the size of the training sample, e.g.~in steepest-descent, where the cost of a gradient computation scales with $n$. Does one really need a highly accurate gradient though, in particular in the early phase of optimization? Why not use subsets $\T_t \subseteq \S$ which are increased in size with the iteration count $t$, matching-up statistical accuracy with optimization accuracy in a dynamic manner? This is the general program we pursue in this paper. In order to make this idea concrete and to reach competitive results, we focus on a recent variant of stochastic gradient descent (SGD), which is known as SAGA \cite{defazio2014saga}. As we will show, this algorithm has a particularly interesting property in how its convergence rate depends on $n$.

\subsection{Empirical Risk Minimization} 

Formally, we assume that training examples $\x \in \S \subseteq \X$ have been drawn i.i.d.~from some underlying, but unknown probability distribution $\Pdata$. We fix a function class ${\cal F}$ parametrized by weight vectors $\w \in \Re^d$ and define the expected risk as $\risk(\w) :=\E f_{\x}(\w)$, where $f$ is an $\x$-indexed family of loss functions, often convex. We denote the minimum and the minimizer of $\risk(\w)$ over $\fclass$ by $\risk^*$ and $\w^*$, respectively. Given that $\Pdata$ is unknown, ERM suggests to rely on the empirical (or sample) risk with regard to $\S$
\begin{align}\label{eq:ERM}
& \risk_\S(\w) := \frac {1}{n} \sum_{\x \in \S} f_{\x}(\w), \;\; \w^*_\S := \argmin_{\w \in \fclass} \risk_{\S}(\w) \,.
\end{align}
Note that one may absorb a regularizer in the definition of the loss $f_{\x}$. 

\subsection{Generalization bounds}
The relation between $\w^*$ and $\w_\S^*$  has been widely studied in the literature on learning theory. It is usually analysed with the help of uniform convergence bounds that take the generic form  \cite{boucheron2005theory}
\begin{align}
\E_\S \left[ \sup_{\w \in {\cal F}} \left| \risk(\w) - \risk_\S(\w) \right| \right]  \leq
\bound(n)\,,
\label{eq:bound}
\end{align} 
where the expectation is over a random $n$-sample $\S$. Here $\bound$ is a bound
that depends on $n$, usually through a ratio $n/d$, where $d$ is the
capacity of $\fclass$ (e.g.~VC dimension). This \textit{fast} convergence rate
has been shown to hold for a class of strictly convex loss functions such as quadratic, and
logistic loss~\cite{bartlett2006convexity,bartlett2005local}. In the realizable
case, we may be able to observe a favorable $\bound(n) \propto d/n$, whereas in the pessimistic case, we may only be able to establish weaker bounds such as $\bound(n) \propto \sqrt{d/n}$ (e.g.~for linear function classes); see also \cite{bousquet2008tradeoffs}. We ignore additional $\log$ factors that can be eliminated using the "chaining" technique \cite{bousquet2002concentration, bousquet2008tradeoffs}.

\subsection{Statistical efficiency}

Assume now that we have some approximate optimization algorithm, which given $\S$ produces solutions $\w_\S$ that  are on average $\epsilon(n)$ optimal, i.e.~$\E_\S \left[ \risk_\S(\w_\S) - \risk_\S^*\right] \le \epsilon(n)$. One can then provide the following quality guarantee in expectation over sample sets  $\S$ \cite{bousquet2008tradeoffs} 
\begin{align}
\label{eq:exprisk-bound}
 \E_\S \risk(\w_\S) - \risk^* 
 & \le \bound(n) + \epsilon(n)\,,
\end{align}
which is an additive decomposition of the expected solution suboptimality into an estimation (or statistical) error $\bound(n)$ and an optimization (or computational) error $\epsilon(n)$. For a given computational budget, one typically finds that $\epsilon(n)$ is increasing with $n$, whereas $\bound(n)$ is always decreasing. This hints at a trade-off, which may suggest to chose a sample size $m < n$. 
Intuitively speaking, concentrating the computational budget on fewer data may be better than spreading computations too thinly. 


\subsection{Stochastic Gradient Optimization} 

For large scale problems, stochastic gradient descent is a method of choice in order to optimize problems of the form given in Eq.~\eqref{eq:ERM}. Yet, while SGD update directions equal the true (negative) gradient direction in expectation, high variance typically leads to sub-linear convergence. This is where variance-reducing methods for ERM such as SAG \cite{roux2012stochastic}, SVRG \cite{johnson2013accelerating}, and SAGA \cite{defazio2014saga} come into play. We focus on the latter here, where  one can establish the following result on the convergence rate (see appendix).
\begin{lemma}
\label{lemma:saga}
Let all $f_{\x}$ be convex with  $L$-Lipschitz continuous gradients and assume that $\risk_\S$ is $\mu$-strongly convex. Then the suboptimality of the SAGA iterate $\w^t$ after $t$ steps is over a randomly sampled $\S$ bounded by
\begin{align*}
& \E_{\cal A} \left[ \risk_{\S}(\w^{t}) - \risk_{\S}^* \right] \leq \rho_n^{t}  C_{\S} , \; \; 
\rho_n = 1 - \min\left(\frac{1}{n}, \frac{\mu}{L}\right),
\end{align*}
where the expectation is over the algorithmic randomness.
\end{lemma}

This highlights two different regimes: For small $n$, the condition number $\kappa := \frac{L}{\mu}$ dictates how fast the optimization algorithm converges. On the other hand, for large $n$, the convergence rate of SAGA becomes $\rho_n = 1 - \frac{1}{n}$.

\subsection{Contributions}

Our main question is: can we obtain faster convergence to a statistically accurate solution by running SAGA on an initially smaller sample, whose size is then gradually increased? Motivated by a simple, yet succinct analysis, we present a novel algorithm, called \methodname{} that implements this idea and achieves $\epsilon(n) \le \bound(n)$ after only $2n$ iterations.


\section{Related Work}
\label{sec:related_work}

\comment{
\textit{Curriculum learning} paradigm criticizes the classical learning methods
where learning process starts from the whole training
set~\cite{bengio2009curriculum}.
Indeed, a learning machine, same as a human, requires to learn simple concepts
firstly and then it could learn more complicated concepts better. From optimization
perspective, curriculum learning is a generalization of \textit{continuation
methods} \cite{allgower1993continuation} for non-convex optimization and
improves generalization in non-convex setting \cite{bengio2009curriculum}. Here
our focus is computational aspect of learning for large scale data sets. We
indeed propose a curriculum to improve computational complexity of a learning
method.
}

Stochastic approximation is  a powerful tool for minimizing objective Eq.~\eqref{eq:ERM} for convex loss functions. The pioneering work of~\cite{robbins1951stochastic} is essentially a streaming SGD method where each observation is used only once. Another major milestones has been the idea of iterate averaging \cite{polyak1992acceleration}. A thorough theoretical analysis of asymptotic convergence of SGD can be found in \cite{kushner2003stochastic}, whereas some non-asymptotic results have been presented in \cite{moulines2011non}.

A line of recent work known as variance-reduced SGD, e.g.
\cite{roux2012stochastic, shalev2013stochastic, johnson2013accelerating,
defazio2014saga,defazio2014finito,konevcny2013semi,zhang2013linear}, has
exploited the finite sum structure of the empirical risk to establish linear convergence for strongly convex objectives and also a
a better convergence rate for purely convex objectives~\cite{mahdavi2013mixed}.
There is also evidence of slightly improved statistical efficiency \cite{babanezhad2015stop}. \cite{frostig15} provides a non-asymptotic analysis of a streaming SVRG algorithm (SSVRG), for which a convergence rate approaching that of the ERM is established.

There have also been related data-adaptive sampling approaches, e.g. in the context of unsupervised learning~\cite{lucic2015tradeoffs} or for non-uniform sampling of data points~\cite{schmidt2013minimizing,HeT15} with the goal of sampling important data points more often. This direction is largely orthogonal to
our dynamic sizing of the sample, which is purely based on random subsampling.
Our sampling strategy is instead based on revisiting samples which has also been explored in~\cite{wang2016reducing} to empirically improve the convergence of certain variance-reduced methods.


\section{Methodology}

\subsection{Setting and Assumptions} 

We work under the assumptions made in Lemma \ref{lemma:saga} and focus on the large data regime, where $n \geq \kappa$ and the geometric rate of convergence of SAGA depends on $n$ through $\rho_n  = 1 - 1/n$.  This is an interesting regime as the guaranteed progress per update is larger for smaller samples. 

This form of $\rho_n$ implies for the case of performing $t=n$  iterations, i.e.~performing one pass\footnote{The SAGA analysis holds for i.i.d.~sampling, so strictly speaking this is not a pass, but corresponds to $n$ update steps.}:
\begin{align}
\label{eq:saga-accuracy}
\E_{\cal A} \left[ \risk_{\S}(\w^{n}) - \risk_{\S}^* \right] \leq \left(1-\frac 1n \right)^{n} C_{\S} \leq \frac{C_{\S}}{e} \,.
\end{align}
So we are guaranteed to improve the solution suboptimality on average by a factor $1/e$ per pass. This in turn implies that in order to get to a guaranteed accuracy $O(n^{-\alpha})$, we need $O(\alpha n \log n)$ update steps. 

\begin{figure}
    \begin{center}
        \begin{tikzpicture}[thick]
        \draw[->] (-3.5,-2) -- (-3.5,2);
	  \draw[->] (-3.5,-2) -- (4.2,-2) node[below] {$m$};
	  \draw[scale=1,domain=-3.5:4,smooth,variable=\x,green, thick, smooth] plot
  ({\x},{1.5/(\x+4)-2});
   \draw[scale=1,domain=-3.5:4,smooth,variable=\x,blue, thick, smooth]
   plot ({\x},{(\x+4)/(4)-2}) ; 
   \draw[scale=1,domain=-3.5:4,smooth,variable=\x,black, thick, smooth]
   plot ({\x},{(\x+4)/(4)+1.5/(\x+4)-2}) ;
   \node (A) at (3,-0.5) {$\epsilon(m)$};
   \node (B) at (3,-1.5) {$\bound(m)$};
   \node (C) at (-1.55,-2.2) {$m^*$};
   \node (D) at (-1.55,-0.7) {};
   \node (E) at (-1.55,-0.3) {$\epsilon(m^*) + \bound(m^*)$};
   \draw[dashed] (C) -- (D);
    \end{tikzpicture}
    \end{center}
    \caption{Tradeoff between sample statistical accuracy term $\bound(m)$ and
    optimization suboptimality $\epsilon(m)$ using sample size $m<n$. Note that $\epsilon(m)$ is drawn by taking the first order approximation of the upper bound $C e^{-\frac n m}$. Here, $m^* = \bigO(n/\log n)$ yields the best balance between these two terms.}
  \label{fig:statsitcal_evolving}
\end{figure}
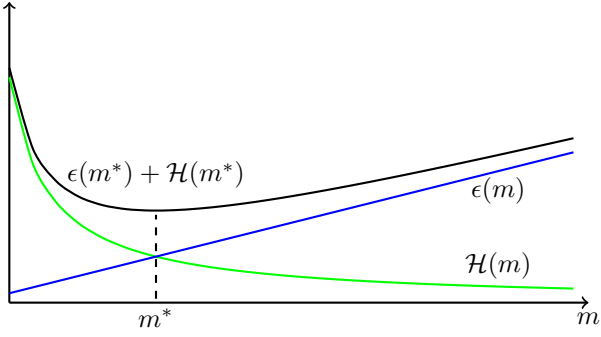

\subsection{Sample Size Optimization} 
\label{sec:sample_size_opt}

For illustrative purposes, let us use the above result to select a sample size for SAGA, which yields the best guarantees. 
\begin{proposition}
\label{proposition:optsample}
Assume $\bound(m) = D/m$ and $n$ is given. Define $C$ to be an upper-bound on $C_\S, \forall \S$ (from Lemma \ref{lemma:saga}), then  for $m \geq \kappa$, $V(m) := \frac{D}{m} + C e^{-\frac n m}$ provides a bound on the expected suboptimality of SAGA. It is minimized for the choice 
\begin{align*}
m^* =  \max\left\{ \kappa, \frac{n}{\log n + \log \frac CD}\right\} \,.
\end{align*}
\begin{proof} 
The first claim follows directly from the assumptions and Lemma \ref{lemma:saga}. Moreover the tightest bound is obtained by differentiating $V$ with regard to $1/m$ and solving for $m$ (see Lemma~\ref{lemma:diffV} in appendix). 
\end{proof} 
\end{proposition}


The result implies that we will perform roughly $\log n + \log \frac CD$ epochs on the optimally sized sample. Also the value of the bound is (for simplicity, assuming $C=D$)
\begin{align}
V(m^*) = \frac {\log n }{n} +  \frac{1}{n}  \leq V(n) = \frac{1}{n} + \frac 1 e \,,
\end{align} 
showing that the single pass approximation error on the full sample is too large (constant), relative to the statistical accuracy.

\subsection{Dynamic Sample Growth} 

As we have seen, optimizing over a smaller sample can be beneficial (if we believe the significance of the bounds). But why chose a single sample size once and for all? A smaller sample set seems advantageous early on, but as an optimization algorithm approaches the empirical minimizer, it is hit by the statistical accuracy limit. This suggests that we should dynamically increment the size of the sample set. We illustrate this idea in Figure \ref{fig:path}. In order to analyze such a dynamic sampling scheme, we need to relate the suboptimality on a sub-sample $\T$ to a suboptimality bound on $\S$. We establish a basic result in the following theorem. 
\begin{theorem} 
\label{theorem:basic}
Let $\w$ be an $(\epsilon,\T)$-optimal solution, i.e.~$\risk_\T(\w) - \risk_\T^* \le \epsilon$, where $\T \subseteq \S$, $m:=|\T|$, $n := |\S|$. Then the suboptimality of $\w$ for $\risk_\S$ is bounded w.h.p.~in the choice of $\T$ as:
\begin{align}
\E_{\S} \left[ \risk_\S(\w) - \risk_\S^*\right] \le \epsilon + \frac{n-m}{n} \bound(m) \,.
\end{align}
\begin{proof}
 Consider the following equality
\begin{align*}
\risk_\S(\w) - \risk_\S^* = \risk_\S(\w) 
\stackrel{(1)}{\mp} \risk_\T(\w) 
\stackrel{(2)}{\mp} \risk^*_\T
\stackrel{(3)}{-} \risk^*_\S
\end{align*}
We bound the three involved differences (in expectation) as follows:
(2): $\risk_\T(\w)  - \risk^*_\T \le \epsilon$ by assumption.  (3): $\E_\S \left[ \risk_\T(\w_\T^*) - \risk_\S(\w_\S^*) \right] \le 0$ as $\T \subseteq \S$.  For (1) we apply the bound (see Lemma \ref{lemma:step-in-theorem} in the appendix)
\begin{align*}
\E_{\S|\T} \left[ \risk_\S(\w) - \risk_\T(\w) \right] \le \frac{n- m}{n} | \risk(\w) - \risk_\T(\w) |\,.
\end{align*}
Moreover 
\begin{align*}
\E_{\T} \left[ \risk(\w) \! - \! \risk_\T(\w) \right] & \le
\sup_{\w'} | \risk(\w') \! - \! \risk_\T(\w') | \le  \bound(m)
\end{align*}
by Eq.~\eqref{eq:bound}, which concludes the proof.
\end{proof} 
\end{theorem}

In plain English, this result suggests the following: If we have optimized $\w$ to $(\epsilon,\T)$ accuracy on a sub-sample $\T$ and we want to continue optimizing on a larger sample $\S \supseteq \T$, then we can bound the suboptimality on $\risk_\S$ by the same $\epsilon$ plus an additional "switching cost" of $(n-m)/n \cdot \bound(m)$.

\begin{figure}
    \begin{center}
    \end{center}
         \begin{tikzpicture}[thick]
        \foreach \y in {0,...,1}
         {
          \draw[thin] (-2.5,4-\y/2) -- (-2,4-\y/2);
         }
         \foreach \y in {0,...,1}
         {
          \draw[thin] (-2.5,3.5-\y/4) -- (-2,3.5-\y/4);
         }
         \draw (-3,3) -- (-1.5,3) ;
         \node at (-2.5,2.7) {$\bound(n/4)$};
         \foreach \y in {0,...,3}
         {
          \draw[thin] (-1,4-\y/3) -- (-0.5,4-\y/3);
         }
         \foreach \y in {1,...,6}
         {
          \draw[thin] (-1,3-\y/6) -- (-0.5,3-\y/6);
         }
         \draw (-1.5,2) -- (0,2) ;
         \node at (-1,1.7) {$\bound(n/3)$};
         \foreach \y in {0,...,4}
         {
          \draw[thin] (0.5,4-\y/4) -- (1,4-\y/4);
         }
         \foreach \y in {1,...,8}
         {
          \draw[thin] (0.5,3-\y/8) -- (1,3-\y/8);
         }
          \foreach \y in {1,...,12}
         {
          \draw[thin] (0.5,2-\y/12) -- (1,2-\y/12);
         }
         \draw (0,1) -- (1.5,1) ;
         \node at (0.5,0.7) {$\bound(n/2)$};
          \foreach \y in {0,...,5}
         {
          \draw[thin] (2,4-\y/5) -- (2.5,4-\y/5);
         }
           \foreach \y in {1,...,10}
         {
          \draw[thin] (2,3-\y/10) -- (2.5,3-\y/10);
         }
          \foreach \y in {1,...,15}
         {
          \draw[thin] (2,2-\y/15) -- (2.5,2-\y/15);
         }
           \foreach \y in {1,...,20}
         {
          \draw[thin] (2,1-\y/20) -- (2.5,1-\y/20);
         }
         \draw (1.5,0) -- (3,0) ;
         \node at (2,-.2) {$\bound(n)$};
         \draw[->,blue, ultra thick] (-2.25,4) -- (-2.25,3.5) -- (-0.75,3.5)
         -- (-0.75,2.66) -- (0.75,2.66) --(0.75,1.5) -- (2.25,1.5)
         -- (2.25,0);
         \draw[->] (-3.5,-0.5) -- (-3.5,4);
         \node at (-4,2) {\rotatebox{90}{$\risk(\w)$}};
         \draw[->] (-3.5,-0.5) -- (3.5,-0.5);
         \node at (0,-0.8) {sample size};
    \end{tikzpicture}
    \caption{Illustration of an optimal progress path via sample size adjustment. The vertical black lines show the progress made at each step, thus illustrating the faster convergence for smaller sample size.}
  \label{fig:path}
\end{figure}
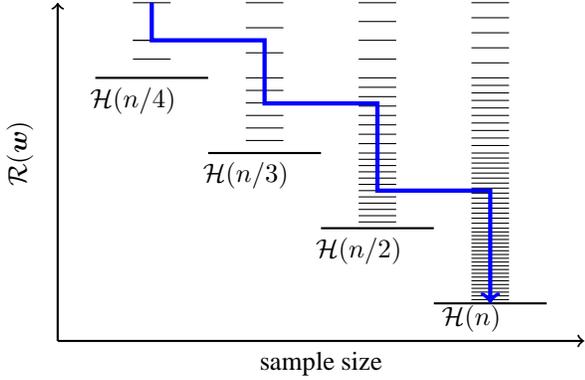


\section{Algorithms \& Analysis} 

\subsection{Computational Limited Learning}

The work of~\cite{bottou2010large} emphasized that for massive data sets the limiting factor of any learning algorithm will be its computational complexity $T$, rather than the number of samples $n$. For SGD this computational limit typically translates into the number of stochastic gradients evaluated by the algorithm, i.e.~$T$ becomes the number of update steps. One obvious strategy with abundant data is to sample a new data point in every iteration. There are asymptotic results establishing bounds for various SGD variants in \cite{bousquet2008tradeoffs}. However, SAGA and related algorithms rely on memorizing past stochastic gradients, cf.~\cite{hofmann2015variance}, which makes it beneficial to revisit data points, and which is at the root of results such as Lemma~\ref{lemma:saga}. This leads to a qualitatively different behavior and our findings indicate that indeed, the trade-offs for large scale learning need to be re-visited, cf.~Table~\ref{table:compare_with_baselines}.

\begin{table}[t]
\caption{Comparison of obtained bounds for different SAGA variants when performing $T \ge \kappa$ update steps.}
\label{table:compare_with_baselines}
\vskip 0.15in
\begin{center}
\begin{small}
\begin{tabular}{lcc} 
\hline
\abovespace
\textsc{Method} & \textsc{Optimization error} & \textsc{Samples}
\belowspace
\\ \hline
\abovespace
\textsc{SAGA} \text{(one pass)}&  const. & $T$\\ 
\textsc{SAGA} \text{(optimal size)}& $\bigO(\log T \cdot \bound(T))$ & $T/\log T$\\
\methodname& $\bigO(\bound(T))$  & $T/2$ 
\\ 
\hline
\end{tabular}
\end{small}
\end{center}
\vskip -0.1in
\end{table}

\subsection{SAGA with Dynamic Sample Sizes} We  suggest  to modify SAGA to work with a dynamic sample size schedule. Let us define a \textit{schedule} as a monotonic function $M: {\mathbb Z}_+ \to {\mathbb Z}_+$, where $t$ is the iteration number and $M(t)$ the effective sample size used at $t$. We assume that a sequence of data points ${\bf X}= (\x_1,\dots,\x_n)$ drawn from $P$ is given such that $M$ induces a nested sequence of samples $\T_t := \{ \x_i: 1 \leq i \leq M(t) \}$.

{\methodname} generalizes SAGA \cite{defazio2014saga} in that it samples data points non-uniformly at each iteration. Specifically, for a given schedule $M$ and iteration $t$, it samples uniformly from $\T_t$, but ignores $\X - \T_t$. The pseudocode for {\methodname} is shown in Algorithm \ref{alg:dynasaga}.
 \begin{algorithm}[tb]
   \caption{\methodname}
   \label{alg:dynasaga}
\begin{algorithmic}[1]
   \STATE {\bfseries Input:} \\
   $\quad$ training examples $\X = (\x_1,\x_2,\ldots,\x_n)$, $\x_i \sim P$\\
   $\quad$ total number of iterations $T$ (e.g.~$T= 2n$)\\
   $\quad$ starting point $\w_0 \in \Re^d$ (e.g~$\w_0 = {\bf 0}$) \\
   $\quad$ learning rate $\eta > 0$ (e.g.~$\eta = \frac{1}{4L}$) \\
   $\quad$ sample schedule $M: [1:T] \to [1:n]$ 
   \STATE $\w \leftarrow \w_0$
   \FOR{$i=1,\dots,n$}
    \STATE $\alpha_i  \leftarrow \nabla f_{\x_i}(\w_0)$    \COMMENT{can also be done on the fly}
   \ENDFOR
   \FOR{$t=1,\dots,T$}
   \STATE sample $\x_i \sim \text{Uniform}(\x_1,\dots,\x_{M(t)})$
   \STATE $g \leftarrow \nabla f_{\x_i}(\w_{t-1})$
   \STATE $A \leftarrow  \sum_{j=1}^{M(t)} \alpha_j / M(t)$ \COMMENT{can be done incrementally}
   \STATE $\w_t \leftarrow \w_{t-1} - \eta \left( g - \alpha_i+ A \right)$
   \STATE $\alpha_i \leftarrow g$
   \ENDFOR
\end{algorithmic}
\caption{{\methodname}}
\end{algorithm}
	
\subsection{Upper Bound Recurrence} 
Assume we are given a stochastic optimization method that
guarantees a geometrical decay at each iteration, i.e. 
\begin{equation} \label{eq:conv_assumed}
	\E_{\cal A} \left[ \risk_\S(\w^t) - \risk_\S^* \right] \leq \rho_n \left[
	\risk_\S(\w^{t-1}) - \risk_\S^*\right]
\end{equation}
where $|\S| = n$ and expectation is over randomness of optimization
process.~\footnote{Note that this assumption is slightly stronger than Lemma~\ref{lemma:saga} but it leads to a much simpler proof technique.}
For acceleration, we pursue the strategy
of using the basic inequalities obtained so far and to stitch them together in the form of a recurrence. At any iteration $t$ we allow ourselves the choice to augment the current sample of size $m$ by some increment $\triangle m \ge 0$. We define an upper bound function $\U$ as follows
\begin{align} 
\label{eq:recursion}
	\U(t,n) = \min
\left\{
	\begin{array}{ll}
		 \rho_{n} \U(t-1,n)\\
		\underset{m < n}{\min} \bigg[\U(t,m) + \frac{n - m}{n}
		\bound(m) \bigg],
	\end{array}
\right.
\end{align}
such that $\U(0,m) = \initerror$, where the initial error $\initerror$ is defined as: 
\begin{align}
\initerror := \frac{4 L}{\mu} \left[\risk(\w^0) - \risk(\w^*)\right].
\end{align}
We refer the reader to Lemma~\ref{lemma:initial_error} in the Appendix for further details on how to derive the expression for $\initerror$.

The construction of Eq.~\eqref{eq:recursion} is motivated by the following result: 
\begin{proposition}
\label{prop:ubound} 
W.h.p.~over the random $n$-sample $\X$, the iterate sequence $\w^t$ generated by {\methodname}  fulfils 
\begin{align*}
\E_{\X} \left[ \risk_{\T_n}(\w^t) - \risk_{\T_n}^* \right]  \le \U(t,n)\,. 
\end{align*}
\begin{proof}
By induction over $t$. The result for $t=0$ follows directly from
Lemma~\ref{lemma:initial_error}. The first case in Eq.~\eqref{eq:recursion} for
the induction step (fixed sample size) follows from Eq.~\eqref{eq:conv_assumed}.
The second case holds by virtue of Theorem \ref{theorem:basic} for any $m$, hence also for the minimum.
\end{proof}
\end{proposition} 


Although the $\U$-recursion can be solved for small $n$ using dynamic programming (assuming knowledge of all constants), we analyse a much simpler heuristics and its $n \to \infty$ behavior. This leads to interesting insights, while being very practical. In particular, our algorithm is an \textit{anytime} algorithm, which does not require knowledge of the total number of iterations $T$ ahead of time.   

\subsection{Sample Schedules}

In this section, we present and analyse two adaptive sample-size schemes for {\methodname}. 

\paragraph{\textsc{Linear}} We start with sample size $\kappa$ and perform $2\kappa$ steps. From then on, we add a new sample every other iteration. The effective sample size is thus
\begin{align}
M_{\textsc{Lin}}(t) = \max \left\{2\kappa, \left\lceil \tfrac t 2 \right\rceil \right\} 
\end{align}
Note that this strategy defines an upper bound on $\U(2t,t)$ and $\U(2t+1,t)$. 

\paragraph{\textsc{Alternating}} We have also implemented a variant where we perform updates in alternation: every other iteration we sample a new data point, which is added to the set. However, we also \textit{force} an update on this fresh sample. In alternation, we simply re-sample an existing data point uniformly at random. We do not provide a theoretical analysis for this scheme but show experimentally that it slightly outperforms the \textsc{Linear} strategy (see results in the appendix). We thus report results for the \textsc{Alternating} strategy in the experimental section.


\subsection{Analysis}

We now provide an analysis that establishes the convergence rate of the {\sc Linear} strategy.

\begin{lemma} \label{lemma:twopass}
For $\bound(n) = D n^{-\alpha}, 0<\alpha \leq 1$, the {\sc Linear}
strategy obtains the following suboptimality
\begin{align}
\U(2n,n) \leq \bound\left(n\right) +
\frac{\initerror}{2}\left(\frac{\kappa}{ n}\right)^2 
\label{eq:lemma:twopass_prop}
\end{align}
\begin{proof}
By induction over $n$. The base case follows from $C_m \leq \initerror$. Using Eq.~\eqref{eq:recursion} and ~\eqref{eq:lemma:twopass_prop} for the inductive case,  we get
\begin{align*}
\U \left(2(n+1),n+1\right) & \stackrel{~\eqref{eq:recursion}}{\leq} \rho_{n+1}^2 \left[ \U\left(2n,n\right) + \frac{1}{n+1} \bound(n) \right] \\
& \stackrel{~\eqref{eq:lemma:twopass_prop}}{\leq} \frac{\initerror}{2}
\left(\frac{\kappa}{n+1}\right)^2 +\frac{n^2\left(n+2\right)}{\left(n+1\right)^3} \bound(n) 
\end{align*}
Note that by definition of the logarithmic function, $\log \left[ n (n+2) \right] < 2\log(n+1)$, and moreover 
\begin{align*}
\frac{n}{n+1} \frac{\bound(n)}{\bound(n+1)} = \frac{n^{1-\alpha}}{(n+1)^{1-\alpha}} \leq 1 \,,
\end{align*}
which  completes the proof. 
\end{proof}
\end{lemma}
This means that for large enough $n$ the {\sc Linear} strategy is able to  approach the statistical accuracy with $2n$ iterations, i.e.~two "passes" over the data. Note the very significant improvement relative to the $\log n$ factor inherent to the optimal fixed sample size choice (see Table~\ref{table:compare_with_baselines} for a comparison of these two bounds).

What does that imply for the $T=n$ case that we have been emphasizing? It is simple to state an answer as a corollary.
\begin{corollary} \label{corollary:onepass}
Under the same assumptions as Lemma \ref{lemma:twopass}, it holds for even $n$
\begin{align*}
        \U(n,n) \leq \left(3 \cdot 2^{\alpha-1} \right)  \bound\left(n\right) + 2 \initerror
        \left(\frac{\kappa}{n}\right)^2
\end{align*}
\begin{proof}
Note that with Eq.~\eqref{eq:recursion} (a) and Lemma \ref{lemma:twopass} (b) we get 
\begin{align*}
\U(2n,2n) \stackrel{\text{(a)}}{\le} \U(2n,n) + \frac 12 \bound(n) \stackrel{\text{(b)}}\le \frac 32\bound(n) + 2 \initerror  \left( \frac \kappa {2n} \right)^2
\end{align*}
The fact that $\bound(n) = 2^{\alpha} \bound(2n)$ completes the proof. 
\end{proof} 
\end{corollary}    
The proof of the above corollary suggests to only use $n=T/2$ samples, when performing $T$ steps and to simply ignore the other half (that potentially could have been sampled). 
One might wonder if a better strategy than the {\sc Linear} one could be defined, e.g. by iterating more than twice on each newly added sample or by increasing the sample size by more than one. The next lemma answers this question and proves that the {\sc Linear} strategy is optimal for large-scale datasets as long as $\bound(n) \propto 1/n$. 
\begin{lemma}
	Assume that $\bound(n) \propto D/n$, then the {\sc Linear}
	strategy is optimal for all sample size $n > \kappa$.
	\begin{proof}
		Here, we briefly state a sketch of the proof . The details are presented in Appendix~\ref{App:optimality}. First, we reformulate the problem of the optimal sample size schedule in terms of number of iterations on each samples size. Given that this problem is convex, we can use the KKT conditions to prove the optimality of incrementing by one sample (see Lemma~\ref{lemma:plus_one_optimality}) and iterating twice on each sample size (see
		Lemma~\ref{lemma:two_iterations_optimality}).
	\end{proof}
\end{lemma}



\begin{figure}
\center
\includegraphics[width=0.4\textwidth]{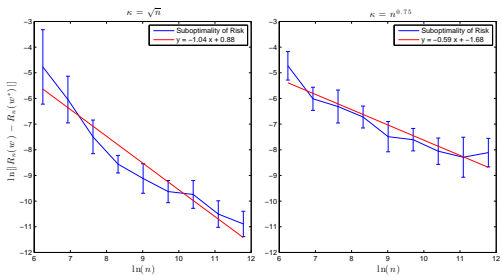} 
\caption{{\it Results on synthetic dataset.} (left) Since, the empirical suboptimality is $\propto 1/n$, we expect the slope measured on this plot to be close to one. (right) Since  $\kappa = n^{0.75}$ slows down the convergence rate, the slope of this plot is less than one.}
\label{fig:slopes}
\end{figure}

\begin{figure*}[t!]
	\begin{center}
          \begin{tabular}{@{}c@{\hspace{5mm}}c@{\hspace{5mm}}c@{}}
            \includegraphics[
            width=0.3\linewidth]{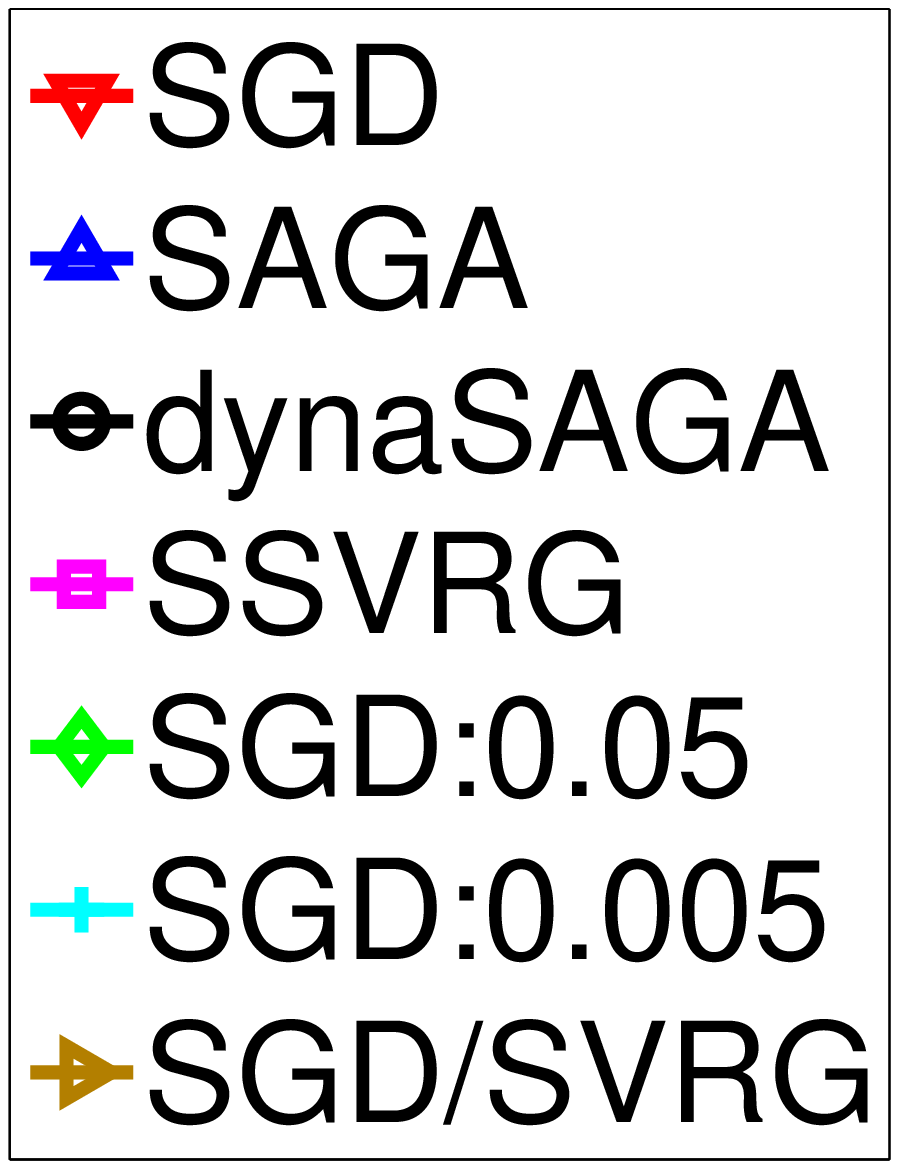} & 
            \includegraphics[
            width=0.3\linewidth]{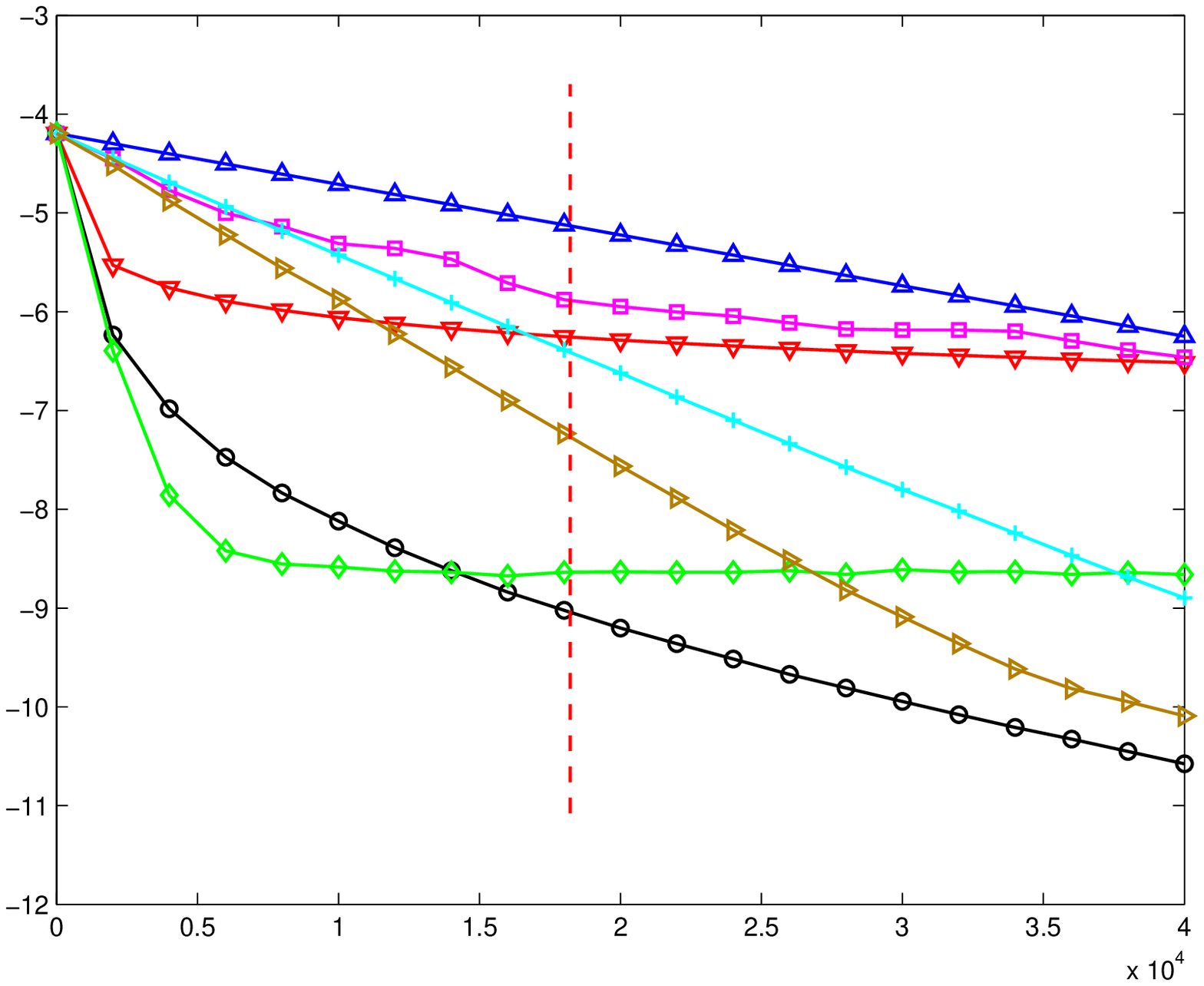} &
             \includegraphics[
             width=0.3\linewidth]{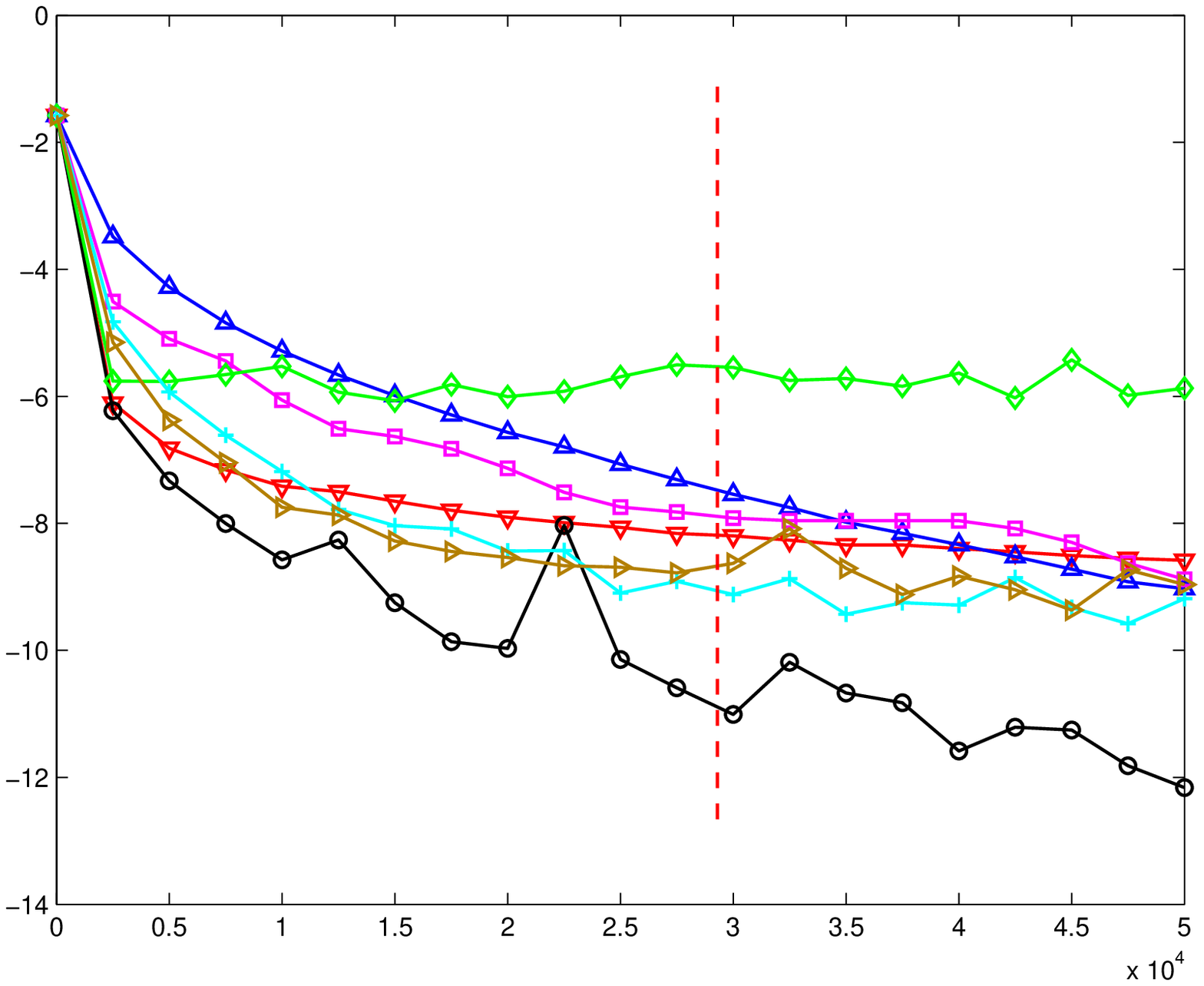} \\
                         &
            1. {\sc rcv} &
            2. {\sc a9a}
            \\
            \includegraphics[width=0.3\linewidth]{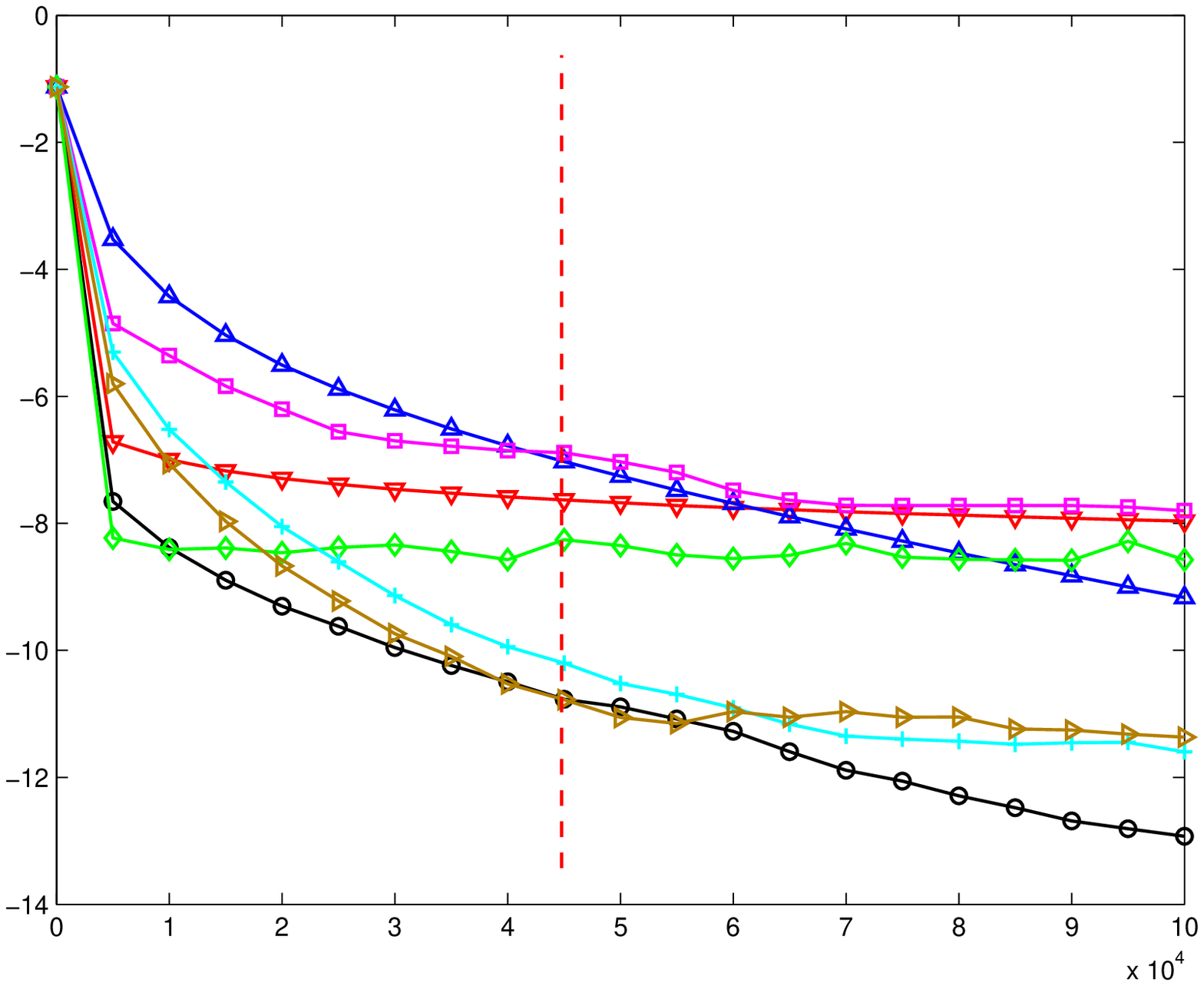} &
            \includegraphics[width=0.3\linewidth]{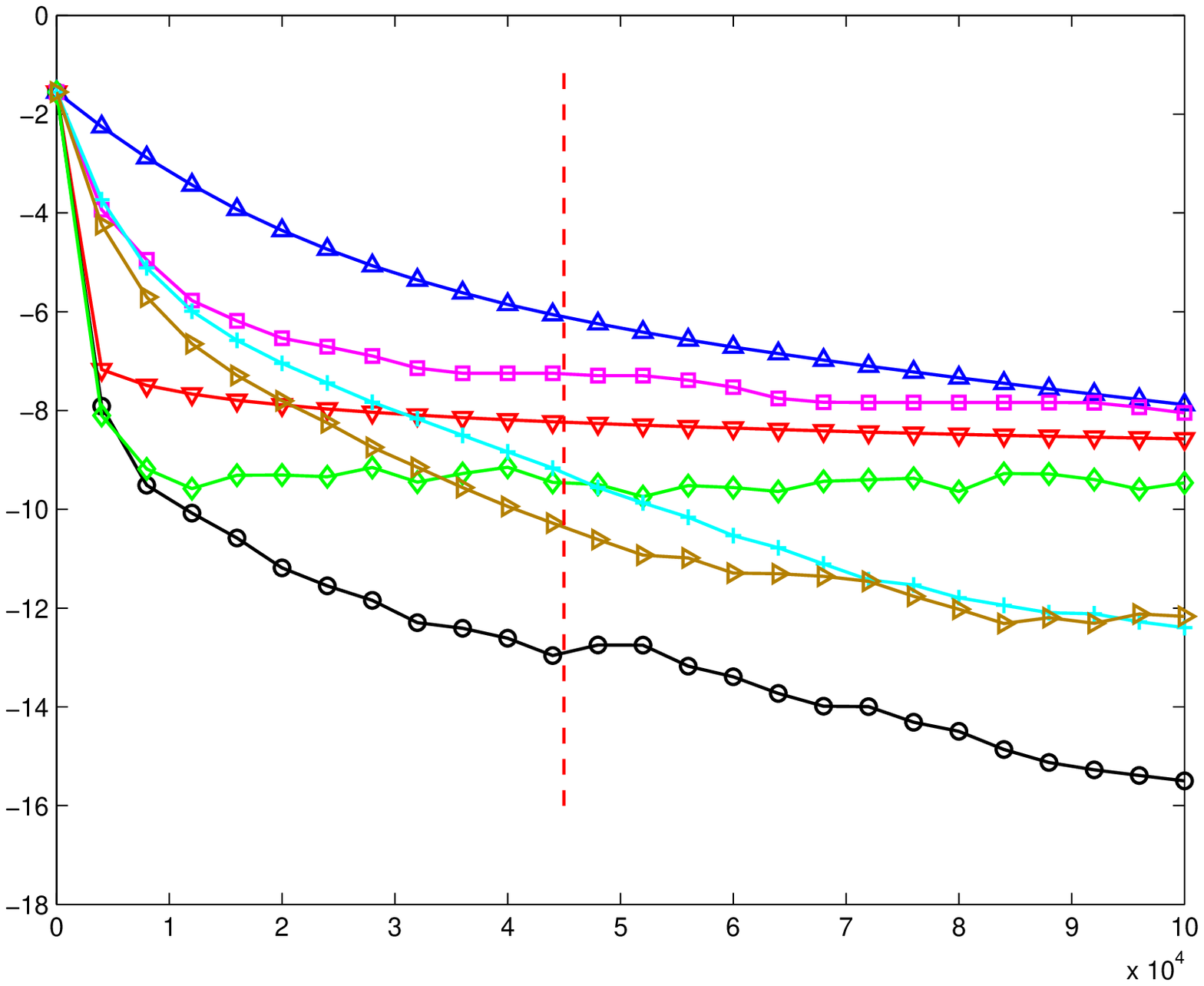} & 
             \includegraphics[width=0.3\linewidth]{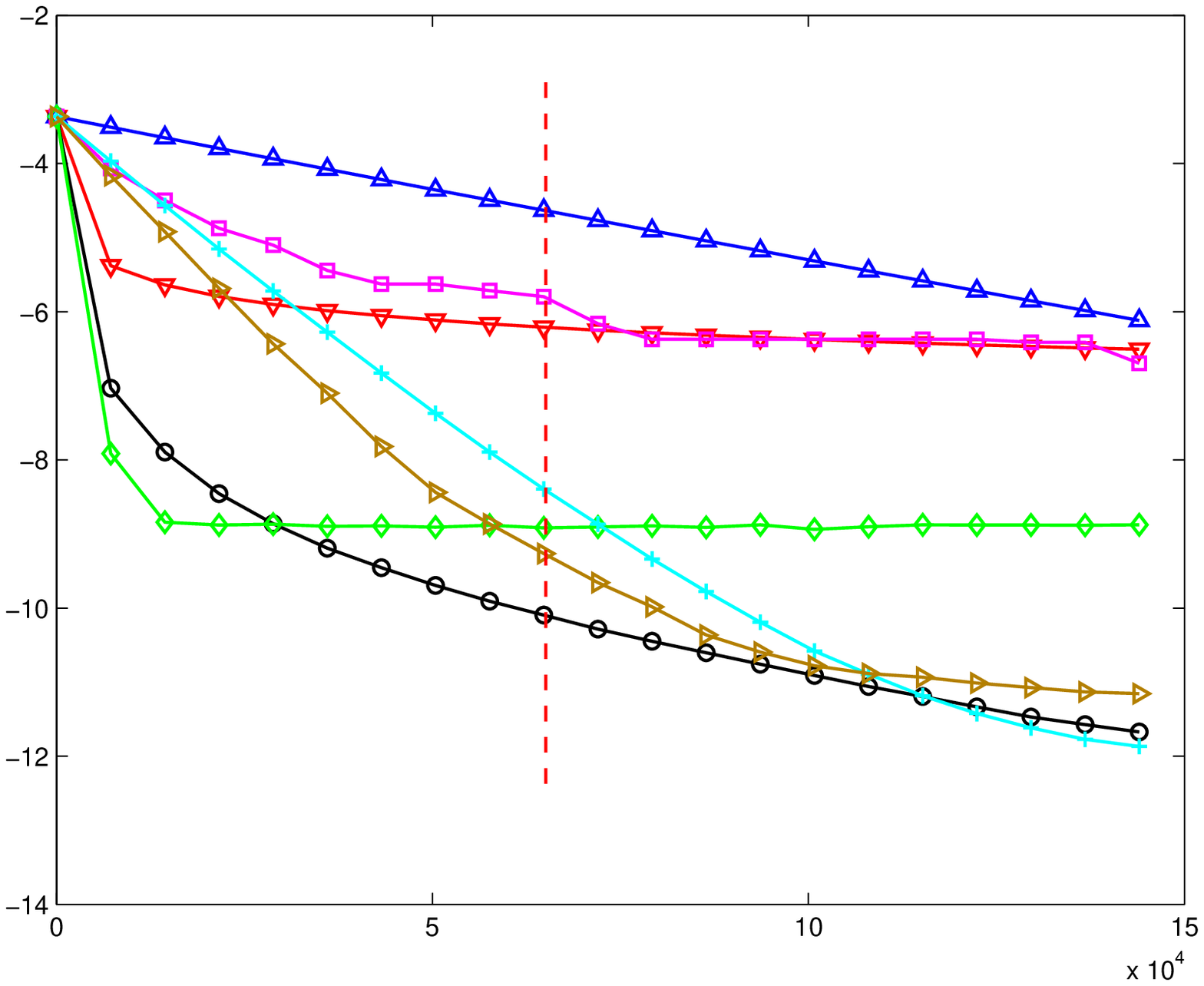}   
            \\ 
            3. {\sc w8a} &
            4. {\sc ijcnn1} & 
            5. {\sc real-sim} 
            \\
             \includegraphics[width=0.3\linewidth]{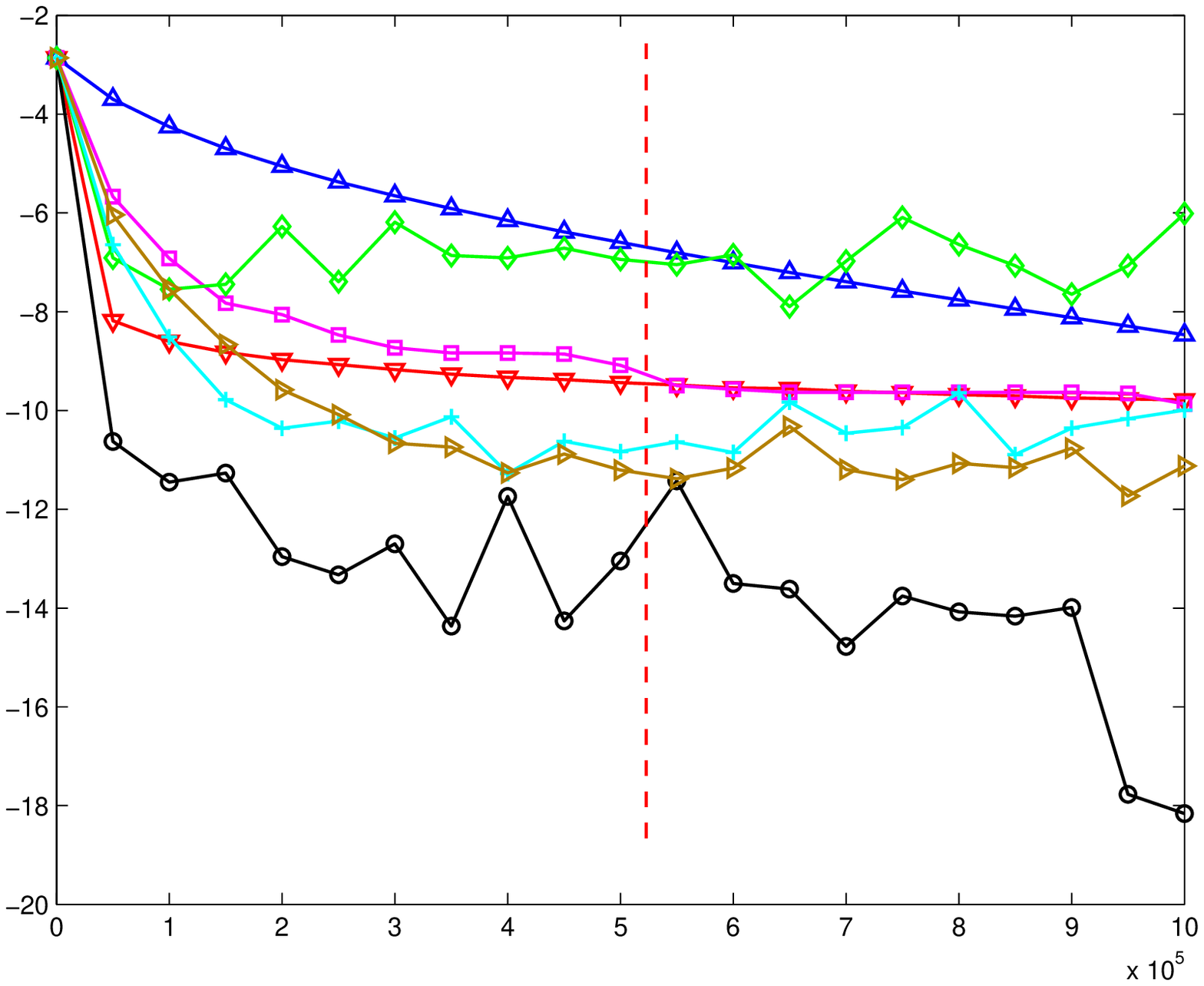} & 
             \includegraphics[width=0.3\linewidth]{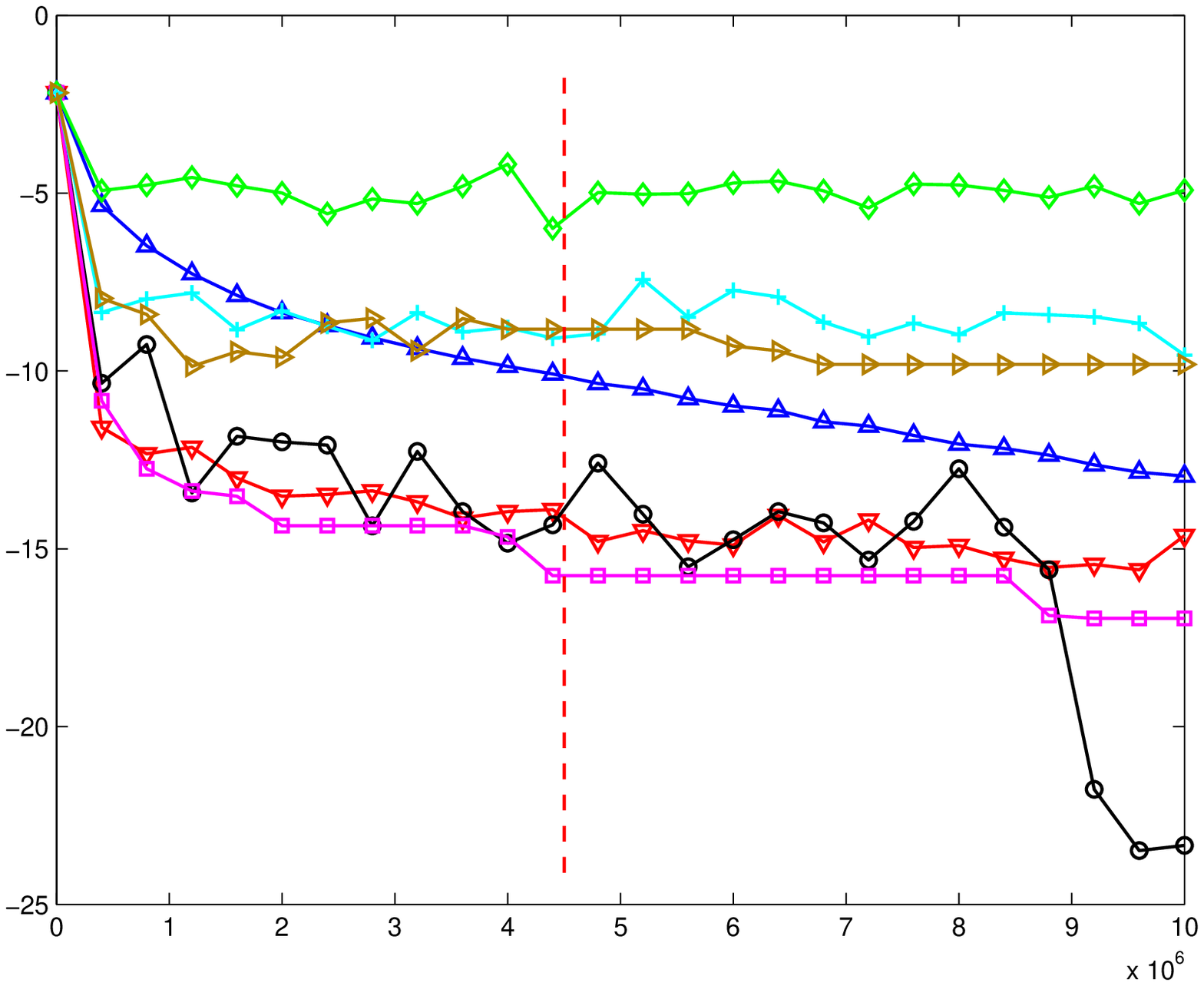} &  \\ 
            6. {\sc covtype} & 
            7. {\sc SUSY} & 
	  \end{tabular}
          \caption{ {\it Suboptimality on the empirical risk}.~The
          vertical axis shows the suboptimality of the empirical risk, i.e. $\log_{2}
          \E_{10} \left[ \risk_{\T}(\w^t) - \risk_{\T}^* \right]$ where the
          expectation is taken over 10 independent runs. The training set
          includes 90\% of the data. The vertical red dashed line is drawn after exactly one epoch over the data. }
          \label{fig:results}
	\end{center}
	
\end{figure*}

\section{Experimental Results}

We present experimental results on synthetic as well as real-world data, which largely confirms the above analysis.

\subsection{Baselines}
We compare \methodname{} (both the {\sc Linear} and {\sc
Alternating} strategy) to various optimization methods presented in Section~\ref{sec:related_work}. This includes SGD (with constant and decreasing step-size), SAGA, streaming SVRG (SSVRG) as well as the mixed  SGD/SVRG approach presented in~\cite{babanezhad2015stop}.

\begin{figure*}[t!]
	\begin{center}
          \begin{tabular}{@{}c@{\hspace{5mm}}c@{\hspace{5mm}}c@{}}
            \includegraphics[
            width=0.3\linewidth]{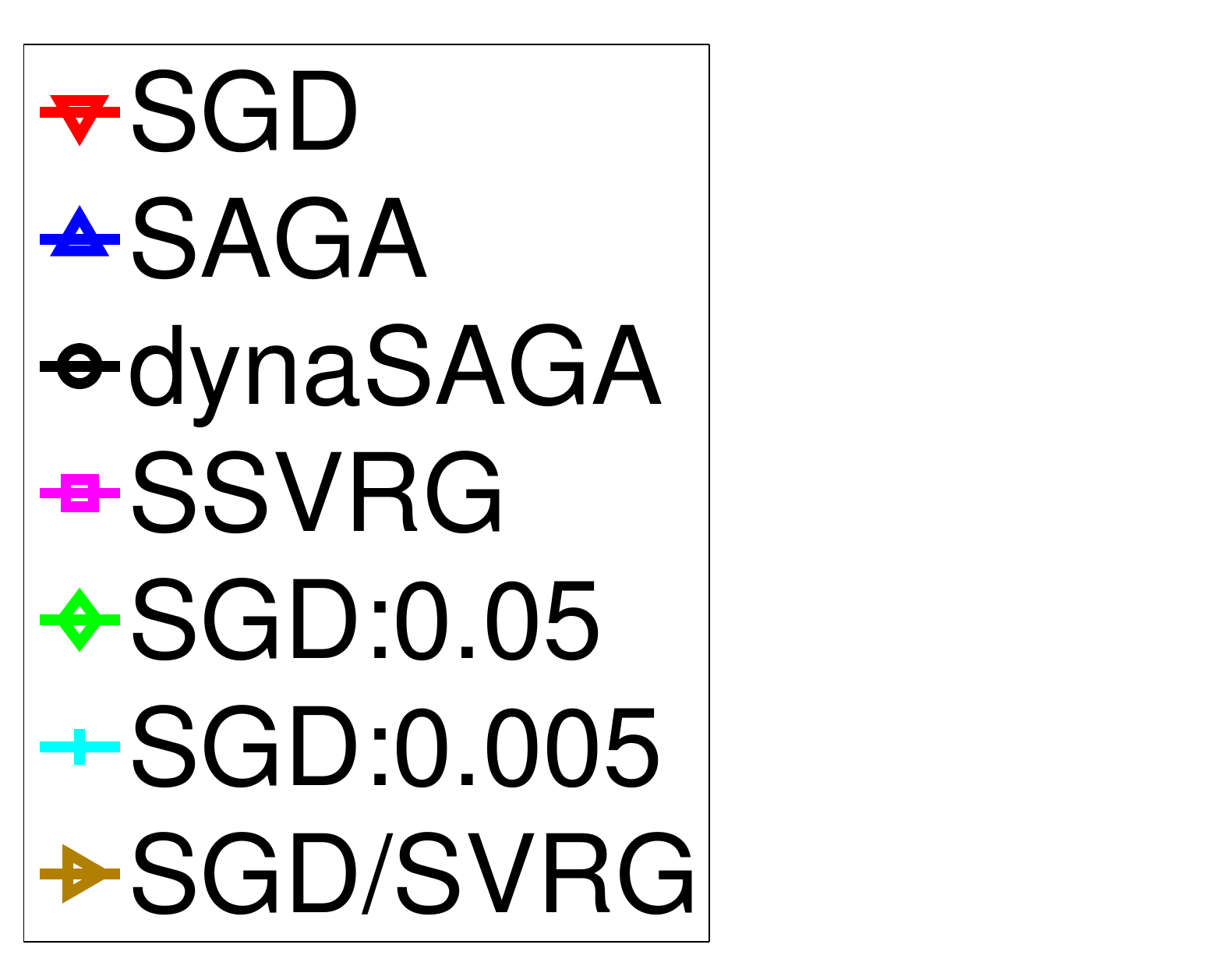} & 
            \includegraphics[
            width=0.3\linewidth]{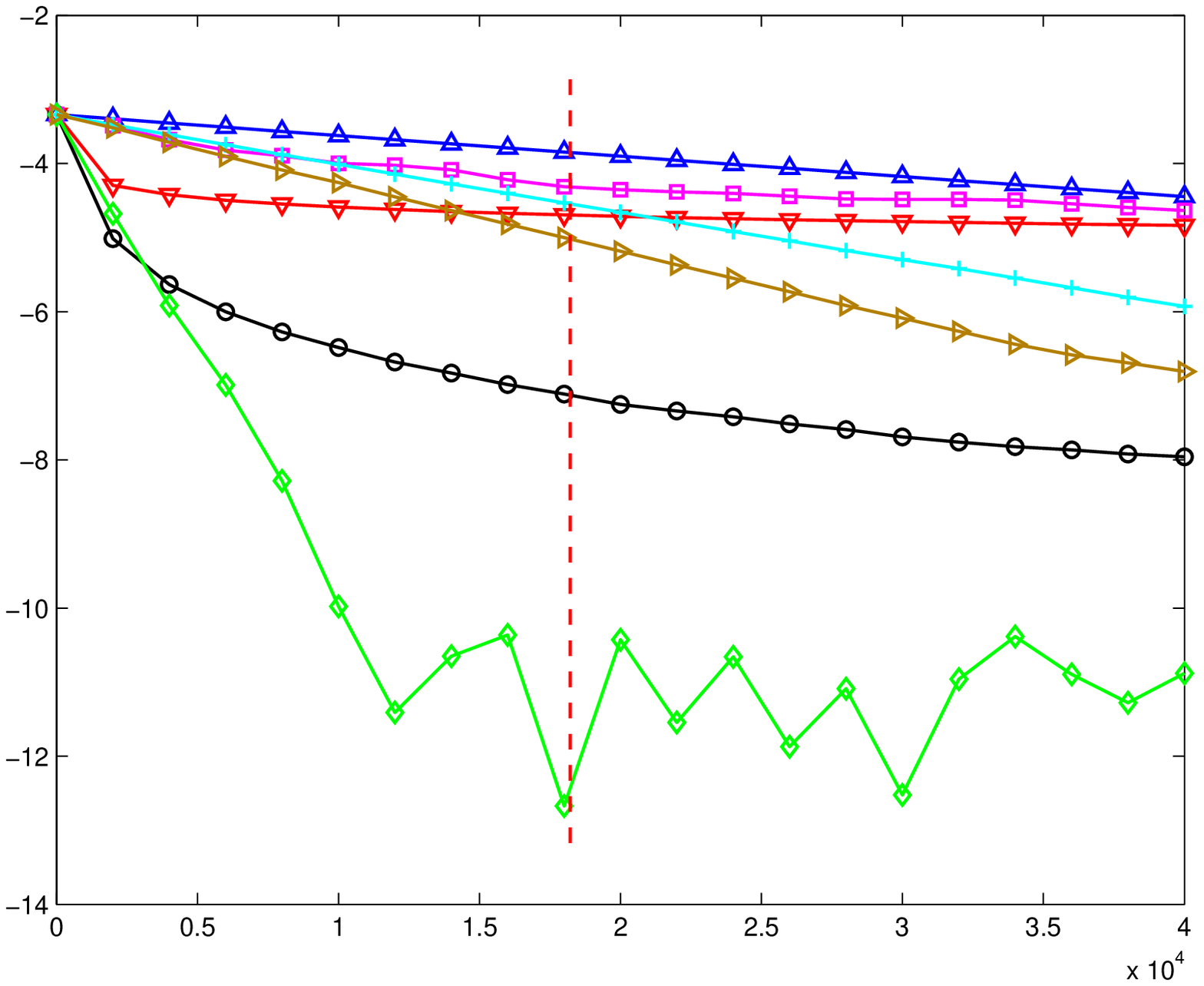} &
             \includegraphics[
             width=0.3\linewidth]{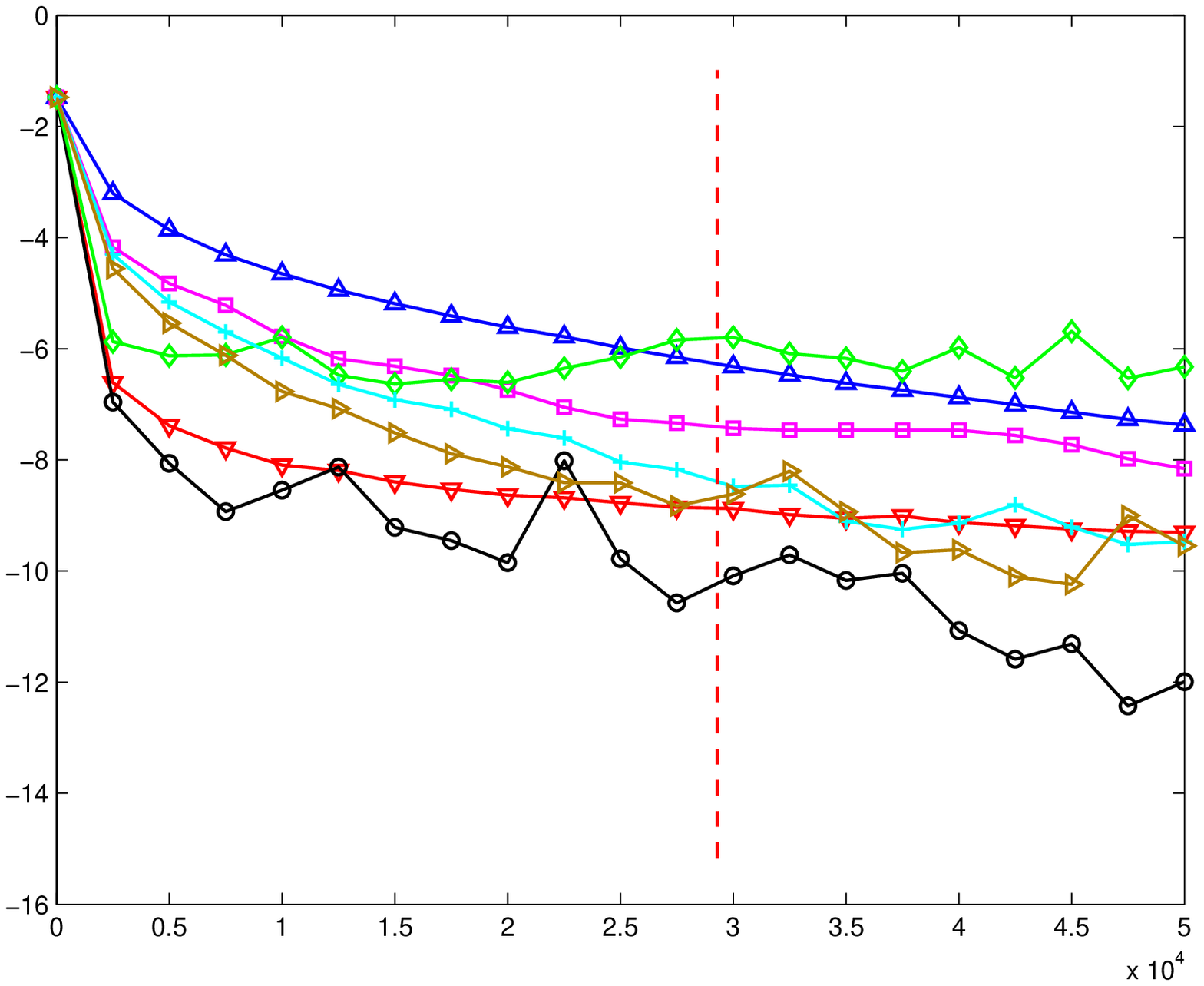} \\
                         &
            1. {\sc rcv} &
            2. {\sc a9a}
            \\
            \includegraphics[width=0.3\linewidth]{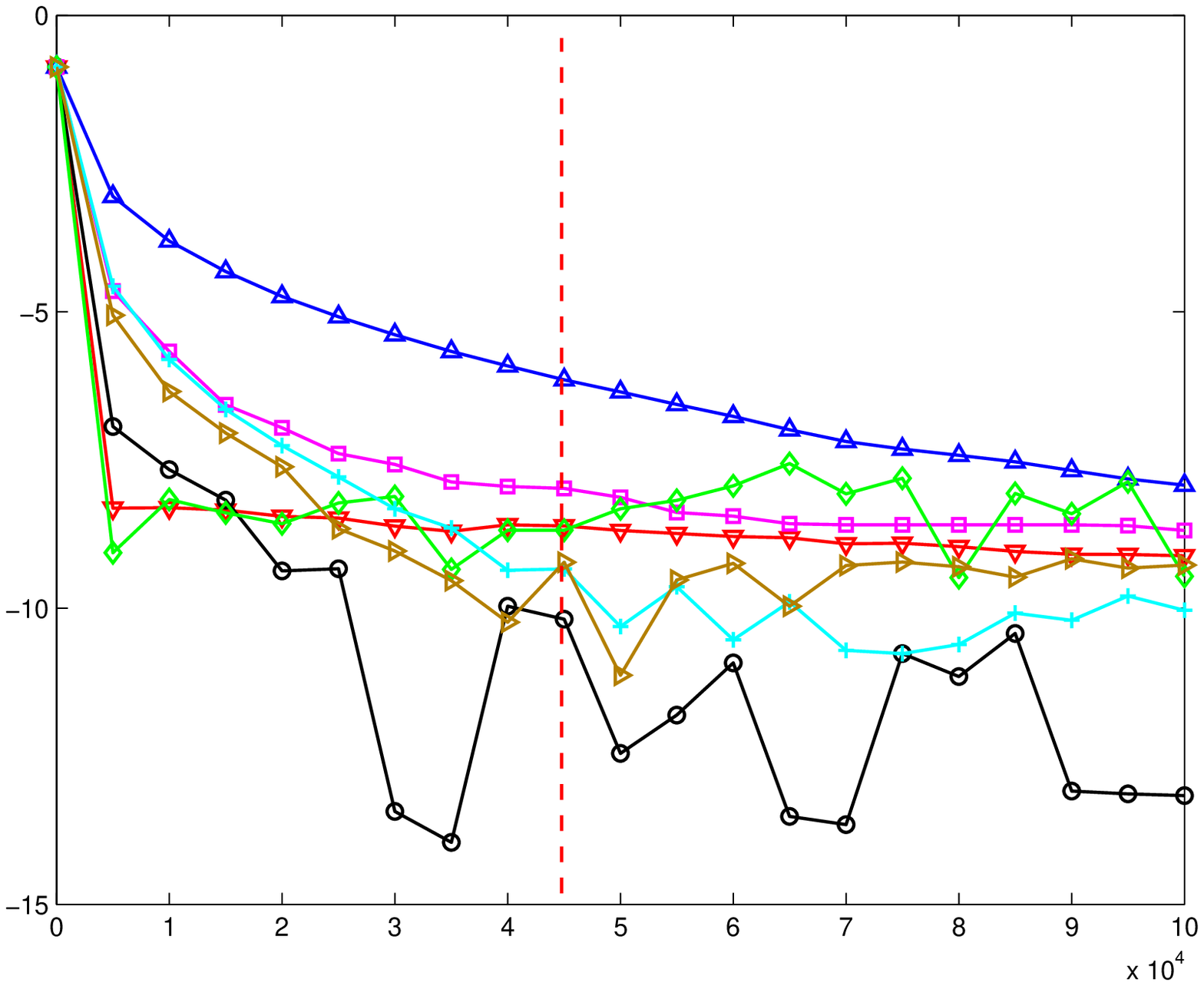} &
            \includegraphics[width=0.3\linewidth]{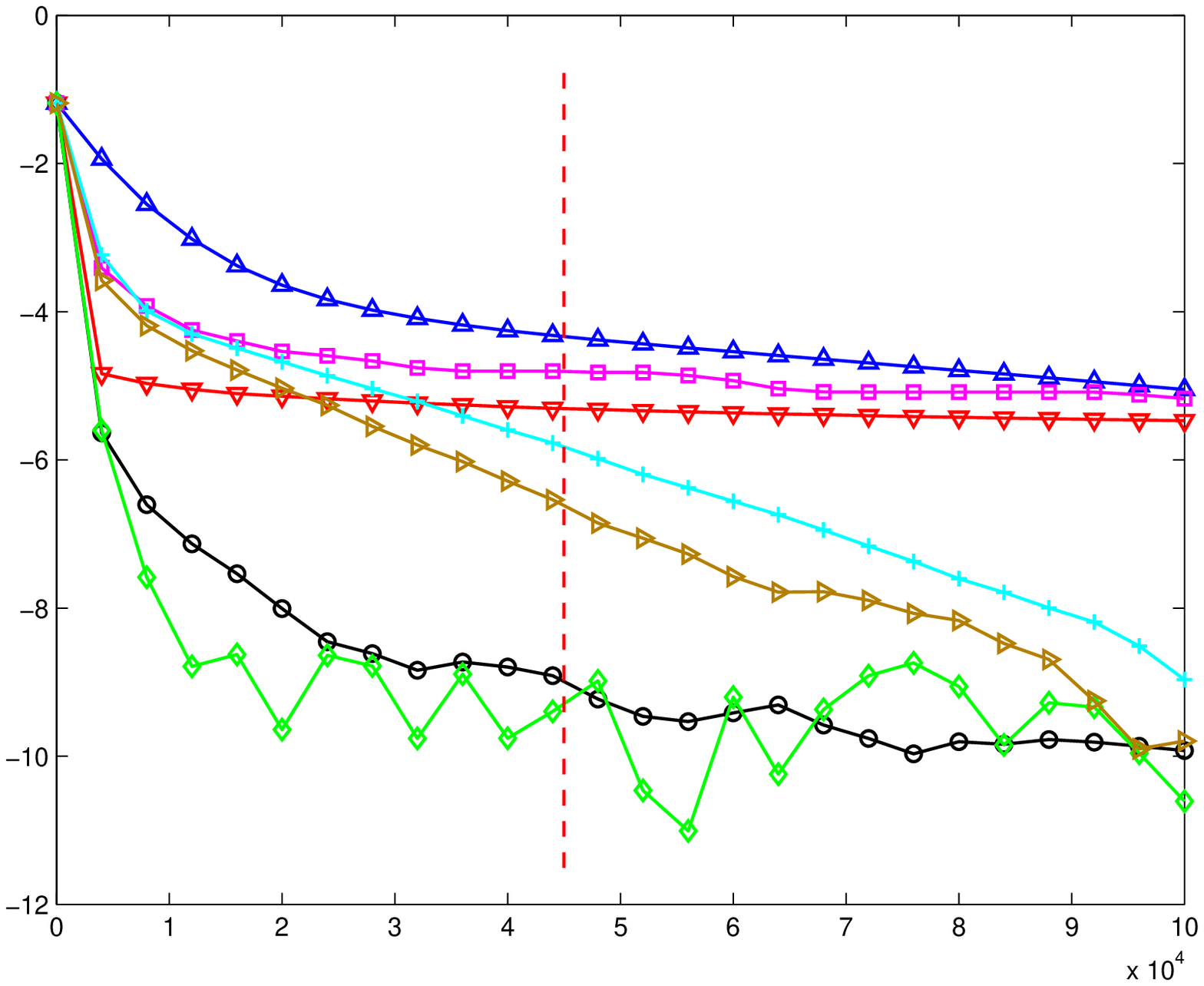} & 
             \includegraphics[width=0.3\linewidth]{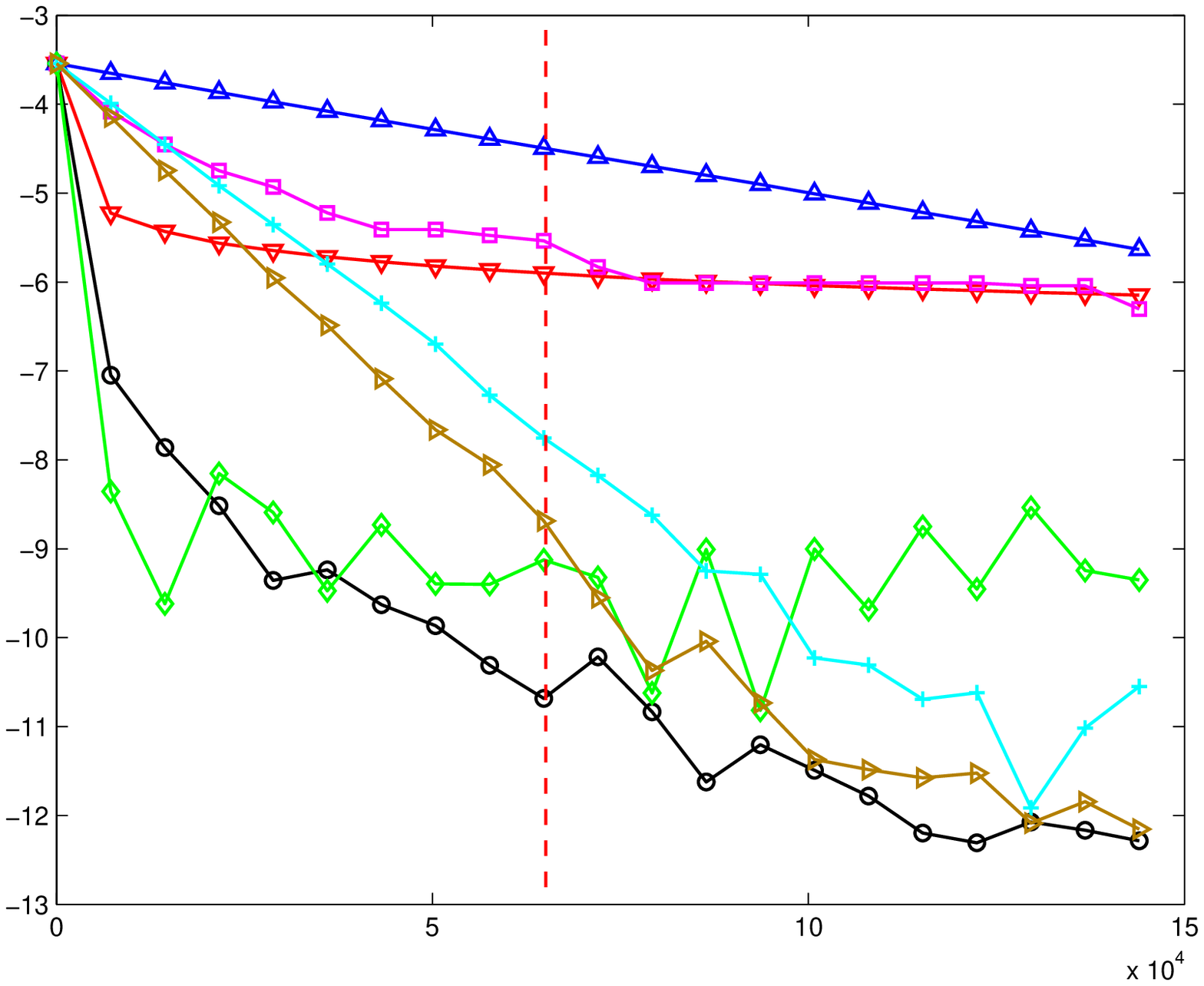}   
            \\ 
            3. {\sc w8a} &
            4. {\sc ijcnn1} & 
            5. {\sc real-sim} 
            \\
             \includegraphics[width=0.3\linewidth]{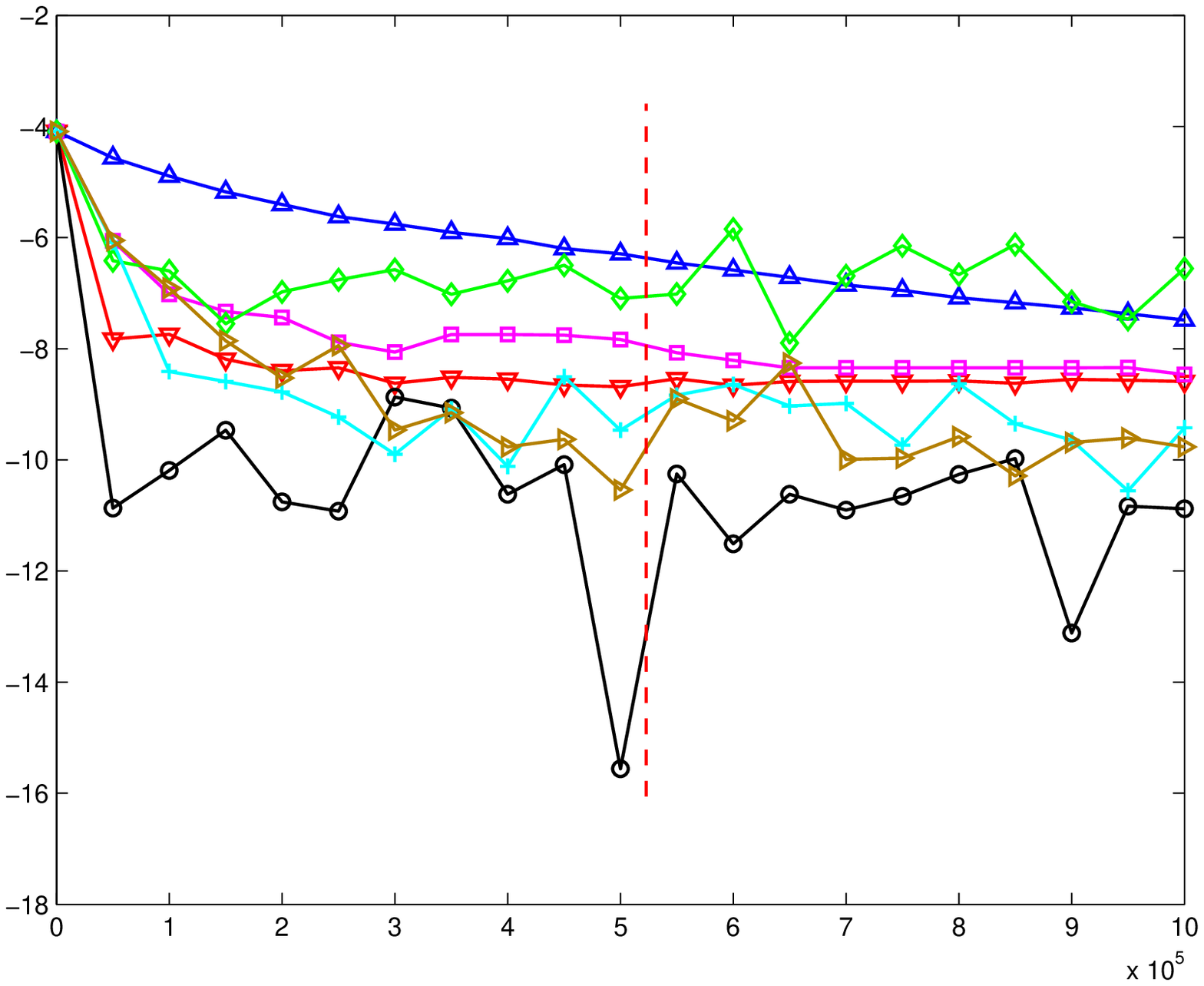} & 
             \includegraphics[width=0.3\linewidth]{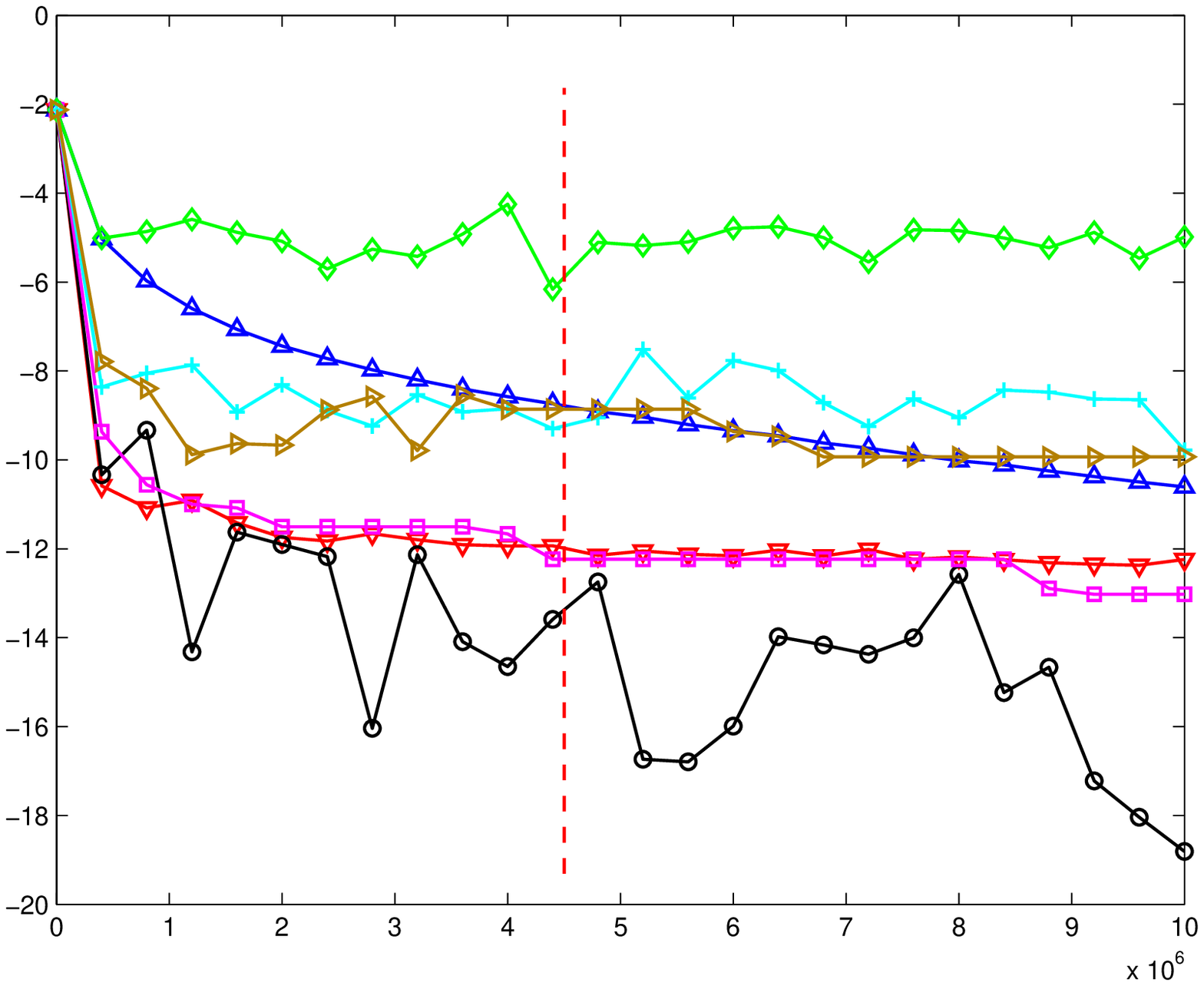} &  \\ 
            6. {\sc covtype} & 
            7. {\sc SUSY} & 
	  \end{tabular}
          \caption{{\it Suboptimality on the expected risk}.~The
          vertical axis shows the suboptimality of the expected risk, i.e. $\log_{2}
          \E_{10} \left[ \risk_{\S}(\w^t) - \risk_{\S}(\w^*_{\T}) \right]$,
          where $\S$ is a test set which includes 10\% of the data
          and $\w^*_{\T}$ is the optimum of the empirical risk on $\T$. The vertical red dashed line is drawn after exactly one epoch over the data.
          }
          \label{fig:results_test}
	\end{center}
	
\end{figure*}

\subsection{Experiment on synthetic data}


We consider linear regression, where inputs $\a \in \Re^d$ are drawn from a Gaussian distribution $\mathcal{N}(0,\Sigma_{d \times d})$ and outputs are corrupted by additive noise $y = \langle \x, \w^* \rangle + \epsilon$,  $\epsilon \sim \mathcal{N}\left(0,\sigma^2\right)$. We are given $n$ i.i.d observations of this model, $\S = \{ (\a_i,y_i)\}_{i=1}^n$, from which we compute the least squares risk $ \risk_\S(\w) = \frac{1}{n} \sum_{i=1}^n \left(\langle \a_i, \w\rangle -	y_i\right)^2$.

By considering the matrix $A_n$ to be a row-wise arrangement of the input vectors $\a_i$, we can write the Hessian matrix of $\risk_n(\w) $ as $\Sigma_n = \frac{1}{n} A_n^T A_n$. When $n \gg d$, the matrix $\Sigma_n$ converges to $\Sigma$ and we can therefore assume that $\risk_n(\w)$ is $\mu$-strongly convex and $L$-Lipschitz where the constants $\mu$ and $L$ are the smallest and largest eigenvalues of $\Sigma$. We experiment with two different values for the condition number $\kappa$.

\paragraph{Case $\kappa = \sqrt{n}$:} We use a diagonal  $\Sigma$ with elements decreasing from $1$ to $\frac{1}{\sqrt{n}}$, hence $\kappa = \sqrt{n}$. In this particular case the analysis derived in Lemma~\ref{lemma:twopass} predicts an upper bound $\U(n,n) < \bigO(\frac{1}{n})$ which is confirmed by the results shown in Figure~\ref{fig:slopes}.

\paragraph{Case $\kappa = n^{\frac{3}{4}}$:}

When $\kappa = n^{\frac{3}{4}}$, the term $\left(\frac{\kappa}{n}\right)^2$ is the dominating term in the proposed upper-bound. In this case, $\U(n,n)$ is thus upper-bounded by $\bigO\left(\frac{1}{\sqrt{n}}\right)$, which is once again verified experimentally in Figure~\ref{fig:slopes}.

\subsection{Experiments on Real Datasets}

\begin{table}[t]
\caption{{\it Details of the real datasets used in our experiments.} All datasets were selected from the LIBSVM dataset collection.}
\label{table:datasets}
\vskip 0.15in
\begin{center}
\begin{small}
\begin{sc}
\begin{tabular}{lcc} 
\hline
\abovespace\belowspace
Dataset & Size & Number of features 
\\
\hline
\abovespace
rcv1.binary &  20242 & 47236 \\ 
a9a & 32561 & 123 \\
w8a & 49749 & 300 \\
ijcnn1    & 49990  & 22 \\
real-sim &  72309 & 20958 \\
covtype.binary & 581012 & 54 \\
SUSY & 5000000 & 18
\belowspace
\\
\hline
\end{tabular}
\end{sc}
\end{small}
\end{center}
\vskip -0.1in
\end{table}

We also ran experiments on several real-world datasets in order to compare the performance of \methodname{} to state-of-the-art methods. The details of the datasets are shown in Table \ref{table:datasets}. Throughout all the experiments we used the logistic loss with a regularizer
$\lambda = \frac{1}{\sqrt{n}}$\footnote{We also present some additional results for various regularizers of the form $\lambda = \frac{1}{n^p}, p < 1$ in the appendix}. Figures~\ref{fig:results}, and~\ref{fig:results_test} show the suboptimality on the empirical risk and expected risk after a single pass over the datasets. The various parameters used for the baseline methods are described in Table~\ref{table:exp_set}. A critical factor in the performance of most baselines, especially SGD, is the selection of the step-size. We picked the best-performing step-size within the common range guided by existing theoretical analyses, specifically $\eta = 1/L$ and $\eta = \frac{C}{C + \mu t }$ for various values of $C$. Overall, we can see that \methodname{} performs very well, both as an optimization as well as a learning algorithm. SGD is also very competitive and typically achieves faster convergence than the other baselines, however, its behaviour is not stable throughout all the datasets. The SGD variant with decreasing step-size is typically very fast in the early stages but then slows down after a certain number of steps. The results on the RCV dataset are somehow surprising as SGD with constant step-size clearly outperforms all methods but we show in the appendix that its behaviour gets worse as we increase the condition number. As can be seen very clearly, \methodname{} yields excellent solutions in terms of expected risk after one pass (see suboptimality values that intersect with the vertical red dashed lines).


\section{Conclusion}

We have presented a new methodology to exploit the trade-off between computational and statistical complexity, in order to achieve fast convergence to a statistically efficient solution. Specifically, we have focussed on a modification of SAGA and suggested a simple dynamic sampling schedule that adds one new data point every other update step. Our analysis shows competitive convergence rates both in term of suboptimality on the empirical risk as well as (more importantly) the expected risk in a one pass or a two pass setting. These results have been validated experimentally. 

Our approach depends on the underlying optimization method only through its convergence rate for minimizing an empirical risk. We thus suspect that a similar sample size adaption is applicable to a much wider range of algorithms, including to non-convex optimization methods for deep learning. 

\newpage
\bibliography{bibliography}
\bibliographystyle{icml2016}

\newpage
\appendix
\onecolumn
\section{Appendix}
\label{App:table_details}

\subsection{Proofs}
\textbf{Proof of Lemma~\ref{lemma:saga}}.
\begin{proof} \label{lemma:proof_saga} 
	We start with the convergence rate of SAGA established in~\cite{defazio2014saga} as
	\begin{equation} \label{eq:saga_convergence_solution_space}
		\E_{\cal A} \left[\| \w^{t} - \w^*_\S \|^2\right] \leq \rho_{|\S|}^{t} 
		\left[\| \w^{0} - \w^*_\S \|^2  + \frac{|\S|}{\mu |\S| + L } \left(
		\risk_\S(\w^0) - \langle \nabla\risk_\S(\w^*_\S), \w^0 - \w^*_\S \rangle -
		\risk_\S^*\right) \right].
	\end{equation}
	We then use the $L$-smoothness assumption of $f_\x(\w)$ to relate the suboptimality on the function values to the bound in Eq.~\eqref{eq:saga_convergence_solution_space}.
	\begin{align*}
		 \E_{\cal A} \left[|\risk_\S(\w^{t}) - \risk_\S(\w^*_\S)|\right] & = \E_{\cal
		 A} \left[|\E_{\x \in \S}\left[f_{\x}(\w^{t})\right] - \E_{\x \in \S}
		 \left[f_{\x}(\w^*_\S) \right]|\right]
		 \\
		 & \stackrel{L-\text{smoothness}}{\leq} L \E_{\cal A} \left[\| \w^{t} -
		 \w^*_\S \|^2\right] \\ 
		 & \stackrel{Eq.~\ref{eq:saga_convergence_solution_space}}{\leq} \rho_{|\S|}^t
		 C_{\S},
	\end{align*} 
	where $C_{\S}$ is the initial suboptimality on the empirical risk defined as: 
	\begin{align*} \label{eq:initial_error}
		C_{\S} = L
		\left[\| \w^{0} - \w^*_\S \|^2  + \frac{|\S|}{\mu |\S| + L } \left(
		\risk_\S(\w^0) - \langle \nabla\risk_\S(\w^*_\S), \w^0 - \w^*_\S \rangle -
		\risk_\S^*\right) \right]
	\end{align*}
	Note that this initial error depends on the set $\S$ and its size $|\S|$. In the following Lemma, we propose an upper bound on this initial error that is independent of $\S$
\end{proof}

\begin{lemma} \label{lemma:initial_error}
	 W.h.p, the initial suboptimality error of sample $\S$ is bounded by:
	\[
		C_\S \leq\initerror := \frac{4 L}{\mu} \left[\risk(\w^0) - \risk(\w^*)\right]
	\]
\end{lemma}
\begin{proof}
We first use the fact that $\risk_\S(\w)$ is $\mu$-strongly~convex as well as
the optimality of $\w^*_\S$ to bound $C_\S$ as
	\begin{eqnarray*}
		& C_\S & := L \left(\| \w^0 - \w^*_\S \|^2 + \frac{|\S|}{\mu |\S| + L} \left[
		\risk_\S(\w^0) - \langle \nabla \risk_\S(\w^*_\S), \w^0 - \w^*_\S \rangle
		-\risk_\S(\w^*_\S) \right]\right) \\
		& & \leq \frac{L}{\mu
		}\left[ \risk_\S(\w^0) - \risk_\S(\w^*_\S) \right]+  \frac{|\S| L}{\mu |\S| +
		L} \left[ \risk_\S(\w^0) - \langle \nabla \risk_\S(\w^*_\S), \w^0 - \w^*_\S
		\rangle -\risk_\S(\w^*_\S) \right] \\ 
		& & \leq \frac{L}{\mu
		}\left[ \risk_\S(\w^0) - \risk_\S(\w^*_\S) \right]+  \frac{|\S| L}{\mu |\S| +
		L} \left[ \risk_\S(\w^0)-\risk_\S(\w^*_\S) \right] \\ 
		& & \stackrel{(L>0)}{\leq} \frac{2 L}{\mu} \left[ \risk_\S(\w^0) -
		\risk_\S(\w^*_\S) \right] \\ 
		& & \leq  \frac{2 L}{\mu} \left[ \risk_\S(\w^0) \stackrel{[1]}{\mp}
		\risk(\w^0) \stackrel{[2]}{\mp} \risk(\w^*) \stackrel{[3]}{\mp} \risk(\w^*_\S) - \risk_\S(\w^*_\S) \right]
	\end{eqnarray*}

We use the generalization bounds in~\cite{vapnik1998statistical} to upper bound [1] and [2]. For [3], we used the uniform convergence rate of the ERM that implies~\cite{vapnik1998statistical}:
	\[
		\risk(\w^*_\S) - \risk(\w^*) \leq c \sup_{\w} | \risk_\S(\w) - \risk(\w)
		|,
	\]
	where $c$ is a constant.
We then get
\begin{align}
		C_\S \stackrel{\text{w.h.p}}{\leq} \frac{2 L}{\mu} \left[ \bound(|\S|) +
		\risk(\w^0) - \risk(\w^*) + c \bound(|\S|) + \bound(|\S|) \right].
\end{align}

We also make the further assumption that with high probability the initial
suboptimality is greater than a constant factor of the statistical accuracy,
i.e. $\risk(\w^0) - \risk(\w^*)  > (2+c) \bound(|\S|)$. We can then further
upper bound $C_\S$ as
\begin{align}
		C_\S \leq \frac{4 L}{\mu}  \left[ \risk(\w^0) - \risk(\w^*) \right].
\end{align}
\end{proof}
\begin{lemma}[for Proposition \ref{proposition:optsample}]
\begin{align*}
V(m) := \frac{D}{m} + C e^{-\frac n m}, \;  \text{then}\; \argmin_{0 < m \le n} V(m) = \frac{n}{\log \frac{ nC}D}
\end{align*}
\label{lemma:diffV}
\begin{proof}
\begin{align*}
& \frac{dV}{d m^{-1}} = D -n C e^{-\frac{n}{m}} 
\stackrel != 0  \\
\iff &  e^{-\frac nm} = \frac D { nC} \\
\iff & \frac nm = \log \frac { nC} D \\
\end{align*}
Solving for $m$, this indeed corresponds to a minimum which can be verified by checking the boundary values $m=n$ and $m \to 0$. 
\end{proof}
\end{lemma}

\begin{lemma}[for Theorem \ref{theorem:basic}]
\label{lemma:step-in-theorem}
\begin{align*}
\E_{\S|\T} \left[ \risk_\S(\w) - \risk_\T(\w) \right] \le \frac{n- m}{n} | \risk(\w) - \risk_\T(\w) |\,.
\end{align*}
\begin{proof}
\begin{align*}
    \E_{\S|\T} \left[ \risk_\S(\w) - \risk_\T(\w) \right] & = \E_{\S-\T|\T}
    \left[ \risk_\S(\w) - \risk_\T(\w) \right]\\
    & = \E_{\S-\T}
  \left[
  \frac{1}{n} \left[  \sum_{\x \in \T} f_{\x}(\w)  + \sum_{\y \in \S - \T}
  f_{\y}(\w)\right] - \frac{1}{m}  \sum_{\x \in \T} f_{\x}(\w)  \right]\\
  & =  
  \frac{n-m}{n} \E_{\S-\T} \left[ \frac{1}{n-m} \sum_{\y \in \S - \T}
  f_{\y}(\w) -  \risk_\T(\w) \right] \\
  & = \frac{n-m}{n} \E_{\S-\T} \left[ \frac{1}{n-m} \sum_{\y \in \S - \T}
  f_{\y}(\w) -  \risk_\T(\w) \right]\\ 
  &  = \frac{n-m}{n}\left[ \E_{\S-\T} \left[ \risk_{\S-\T}(\w) \right]-
  \risk_\T(\w) \right] \\ 
  & = \frac{n-m}{n} \left[\risk(\w) - \risk_\T(\w)\right]
\end{align*}
\end{proof}
\end{lemma}

\newpage
\subsection{Optimality of the {\sc Linear} Strategy}
\label{App:optimality}
We here introduce a new notation and chose to represent a sample size schedule by a vector $\ts^n = \langle t_m \rangle,
m<n$ where $t_m$ denotes the number of iterations on sample size $m$. Note that the total number of iterations up to the sample size $n$ is $T =
\sum_{m<n} t_m$. We define $n^-$ as the sample size that we iterate on immediately before sample size $n$, i.e.
\begin{equation}
n^- = \max\{ k<n : t_k > 0\}.
\end{equation}

We now rewrite the suboptimality bound in terms of the sample size schedule $t_n$ as
\begin{align} \label{eq:A_recursive}
	A(\ts^n) &= \E_{\S} \left[ \risk_{\S}(\w(\ts^n)) - \risk_{\S}(\w^*) \right] \nonumber \\
	&= \rho_n^{t_n} \left( A(\ts^{n^-}) + \frac{n-n^-}{n} \bound(n^-)
	\right),
\end{align}
where the second equality is derived using Lemma~\ref{lemma:saga} and Theorem~\ref{theorem:basic}.
  
One can relate the upper bound $\U(n,n)$ to $A(\ts^n)$ using the following constrained program:
\begin{align}
	 & \U(n,n) = \min_{\ts^n} A(\ts^n) \label{eq:convex_formulation_for_u}\\ 
	  \text{Subject to } &\forall m\leq n: -t_m \leq 0   \nonumber\\ 
	  &  \sum_{m\leq n} t_m = n \nonumber
\end{align}

In the following we aim at showing that the {\sc Linear} Strategy is the optimal solution of Equation~\ref{eq:A_recursive}. We first prove a Lemma that will be used in the rest of our analysis.

 \begin{lemma}[Expansion of $A(\ts^n)$] \label{lemma:expanded_upperbound}
 if  $\bound(n) = D/n$, then
    \begin{align}
A(\ts^n) & := C(\ts^n) + \sum_{m=m_0+1}^n B_m(\ts^n), \quad \text{where}
\label{eq:A_expanded}\\
C(\ts^n) & := \initerror  \prod_{i=m_0}^n \left( \frac {i-1}i \right) ^{t_i}, \quad
B_m(\ts^n) := \frac{D}{(m-1)m} \prod_{i=m}^n \left( \frac {i-1}{i} \right)^{t_i}
\,. \label{eq:def_c_b}
\end{align}
\begin{proof}
	Although one could painstakingly unroll the recursivity in Equation~\ref{eq:A_recursive}, we here provide a simple induction proof. First, one can easily verify that the equation holds for $n = m_0$. For the inductive step, we assume it holds for $n^-$ and prove it holds for all $\{k: n^-<k\leq n\}$. According to the definition of $n^-$, we have $t_k = 0$ for all $n^-<k<n$, and therefore
\begin{equation} \label{eq:fake_products}
	\rho_k^{t_k} = \prod_{m=n^-+1}^{k} \rho_m^{t_m}.
\end{equation}

We will also make use of the following equality in our analysis:
\begin{align} \label{eq:series_expansion_for_bound}
	\frac{k-n^-}{k} \bound(n^-) & = \bound(n^-)	- \bound(k)
	\stackrel{(\bound(n) = D/n)}{=} \sum_{m=n^-+1}^k \bound(m -1) - \bound(m).
\end{align}
We are now ready to prove the inductive step.
\begin{align}
	A(\ts^k) & \stackrel{\text{EQ~\ref{eq:A_recursive}}}{=}
	\rho_k^{t_k} \left( A(\ts^{n^-}) + \frac{k-n^-}{k} \bound(n^-) \right) \\ 
	& = \rho_k^{t_k} \left( C(\ts^{n^-}) + \sum_{m=m_0+1}^{n^-}
	B_m(\ts^{n^-}) + \frac{k-n^-}{k} \bound(n^-) \right) \\ 
	&  \stackrel{\text{EQ~\ref{eq:def_c_b},~\ref{eq:fake_products}}}{=}
	C(\ts^k) + \sum_{m=m_0+1}^{n^-} B_m(\ts^{k}) + \rho_k^{t_k} \left(\frac{k-n^-}{k} \bound(n^-) \right) \\
	& \stackrel{\text{EQ~\ref{eq:series_expansion_for_bound}}}{=} C(\ts^k) + \sum_{m=m_0+1}^{n^-}
	B_m(\ts^{k}) + \rho_k^{t_k} \sum_{m=n^-+1}^{k}
	\frac{D}{(m-1)m}
	\\
	& \stackrel{\text{EQ~\ref{eq:fake_products}}}{=} 
	C(\ts^k) + \sum_{m=m_0+1}^{n^-}
	B_m(\ts^{k}) + \sum_{m = n^-+1}^{k} B_m(\ts^k) \\ 
	& =  C(\ts^k) + \sum_{m=m_0+1}^k B_m(\ts^k)
\end{align}
\end{proof}
 \end{lemma}

Using the definitions provided in Lemma~\ref{lemma:expanded_upperbound}, we investigate the optimality conditions of the optimal sample size strategy. 
In the following, we simplify our notations and write $B_m$ and $C$ instead of $B_m(\ts^n)$ and $C(\ts^n)$.

As a first step in our analysis, we introduce the following equations based on the definitions of $B_m$ and $C$. 
\begin{align}
B_{m} & = \frac{1}{m(m-1)} \prod_{i \geq {m}} \left( \frac{i-1}{i} \right)^{t_i} = \frac{m+1}{m-1} \left( \frac {m-1}{m}\right)^{t_m} B_{m+1} \,.
\label{eq:mtomp1}
\end{align}
\begin{align}
\prod_{i=m}^n \left( \frac {i-1}{i} \right)^{t_i} & = \prod_{i=m}^n
\exp\left(\log\left(\left( \frac {i-1}{i} \right)^{t_i}\right)\right)
= \exp \left[
\sum_{i=m}^n t_i \log \left( 1- \frac 1i \right) \right]
\label{eq:exponential_trick_for_products}
\,.
\end{align}
We now compute the derivative of $A(\ts^n_*)$ as
\begin{align} 
  \frac{\partial A(\ts^n_*)}{\partial t_m} & =  \log(1-\frac{1}{m}) \left(
   C(\ts^n) + \sum_{k=m_0+1}^m B_k(\ts^n)\right) \nonumber \\
  & \simeq -\frac{1}{m}  \left(  C + \sum_{k=m_0+1}^m
  B_k \right)\label{eq:partial_derivative} \,.
\end{align} 
$C(\ts^n)$
 		and $B_m(\ts^n)$ are \textit{log-convex} (hence \textit{convex})
 		functions with respect to $\ts^n$.
 		Since the sum operator preserves convexity \cite{boyd04}, $A(\ts_n)$ is
 		\textit{convex} as well. Let $\lambda_i$, $\nu$ denote the Lagrangian
 		coefficients associated with the inequality and equality
 		constraints respectively. According the KKT conditions~\cite{boyd04} for the the optimal solution, the following inequalities hold:
 		\begin{align}
 		     \lambda_m \geq 0 & \label{eq:KKT_positive_lambda}\\
 			 - \lambda_m t_m^* = 0 & \label{eq:KKT_slackness}\\ 
 			 \frac{\partial A(\ts^n_*)}{\partial t_m} - \lambda_m + \nu = 0 &
 			 \label{eq:KKT_derivatives}
 		\end{align}
 		According the above condition there are two possible cases for the partial
 		derivative $\frac{\partial A(\ts^n_*)}{\partial t_m}$: 
 		\begin{itemize}
 		  \item For the case of $t_m^* > 0$, the slackness
 		  condition~\ref{eq:KKT_slackness} implies that $\lambda_m = 0$. Then, according to the 
 		  condition~\ref{eq:KKT_derivatives}: 
 		  \begin{align}
 		  	& \frac{\partial A(\ts^n_*)}{\partial t_m} = - \nu \nonumber \\ 
 		  	\stackrel{\text{EQ.~\ref{eq:partial_derivative}}}{\Longrightarrow}
 		  	& \frac{1}{m}  \left(  C + \sum_{k=m_0+1}^m
  B_k\right) = \nu \label{eq:condition_for_nonzero_iteration}
 		  \end{align}
 		  \item For the case of $t_m^* = 0$, $\lambda_i > 0 (a.)$ holds
 		  based on the complementary slackness condition~\ref{eq:KKT_slackness}. 
 		  \begin{align}
 		  	&\frac{\partial A(\ts^n_*)}{\partial t_m} = \lambda_i - \nu
 		  	\stackrel{(a.)}{>} - \nu \nonumber \\ 
 		  	\stackrel{\text{EQ.~\ref{eq:partial_derivative}}}{\Longrightarrow}
 		  	& \frac{1}{m}  \left( C + \sum_{k=m_0 +1}^m
  B_k \right) < \nu \label{eq:condition_for_zero_iteration}
 		  \end{align}
 		\end{itemize}
 		
In the following two lemmas we use the conditions of optimality derived in Equations~\ref{eq:condition_for_nonzero_iteration} and~\ref{eq:condition_for_zero_iteration} to prove optimality of the {\sc Linear} Strategy. Specifically, we first prove that for the optimal strategy, $t_m>0$ for $m_0<m\leq n^-$ and $t_m = 0$ for $m>n^-$. We also prove the optimality of incrementing the sample size by one. In the second lemma, we show that $t_m^* \simeq 2$.
 
 \begin{lemma}[Optimality of sample size increment]
 \label{lemma:plus_one_optimality}
 For large enough $m$, a schedule with $t_{m}=0$ and $t_{m+1}>0$ cannot be optimal. 
\begin{proof}
 Note that by repeated application of Equation~\eqref{eq:mtomp1} we obtain 
 \begin{align} \label{eq:chain_of_inequalities}
B_{m+1} < B_{m} < \dots < B_{m^-+1}
\stackrel{\text{EQ.~\ref{eq:condition_for_nonzero_iteration}
\&~\ref{eq:condition_for_zero_iteration}}}{<}
\nu
\end{align}
where optimality conditions \textit{a.} $t_{m^-}>0$
(EQ.\ref{eq:condition_for_nonzero_iteration}) and \textit{b.} $t_{m^-+1} = 0$
(EQ.\ref{eq:condition_for_zero_iteration}) yeild the last inequality:
\begin{align}
		B_{m^-+1} & = \sum_{k = m_0 +1}^{m^-+1} B_k - \sum_{k = m_0 +1}^{m^-} B_k \mp
		 C \\ 
		& \stackrel{a.}{=}  \sum_{k = m_0 +1}^{m^-+1} B_k +  C - m \nu \\ 
		& \stackrel{b.}{<} (m+1) \nu - m \nu = \nu
\end{align}
On the other hand, optimality of \textit{a.}  $t_{m+1} > 0$
(EQ.\ref{eq:condition_for_nonzero_iteration}) and \textit{b.} $t_{m}=0$
(EQ.\ref{eq:condition_for_zero_iteration}) also imply $B_{m+1} > \nu$ which
is in contradiction with the previously established $B_{m+1} < \nu$. Indeed, we have
\begin{align}
		B_{m+1} & = \sum_{k = m_0 +1}^{m+1} B_k - \sum_{k = m_0 +1}^{m} B_k \mp
		 C \\ 
		& \stackrel{a.}{=} (m+1) \nu - \sum_{k = m_0 +1}^{m} B_k - C \\ 
		& \stackrel{b.}{>} (m+1) \nu - m \nu = \nu \;\;
\end{align}
\end{proof}
 \end{lemma}
 \begin{lemma}[Optimality of two iterations]
 \label{lemma:two_iterations_optimality}
 	Consider $\ts^n_*$ as the minimizer of the optimization
 	problem~\ref{eq:convex_formulation_for_u}.
 	For sufficiently large $m:m_0<m\leq n^-$,
 	$t_m^* \simeq 2$.
 	\begin{proof}
 		Using Lemma~\ref{lemma:plus_one_optimality}, $t_m^*>0$ holds for $m_0<m\leq
 		n^-$. We proceed with optimality conditions \textit{a.} $t_m^*>0$ and
 		\textit{b.} $t_{m-1}^*>0$ in
 		equation~\ref{eq:condition_for_nonzero_iteration}.
 		\begin{align}
 			B_{m} & = \sum_{k = m_0 +1}^m B_k - \sum_{k = m_0 +1}^{m-1} B_k \mp
 			 C \\
 			& \stackrel{a.}{=} m \nu - \sum_{k = m_0 +1}^{m-1} B_k - C \\ 
 			& \stackrel{b.}{=} m \nu - (m-1) \nu = \nu
 		\end{align}
 		Consequently, $B_m = B_{m+1} = \nu$. Using Equation~\ref{eq:mtomp1}, one
 		conclude that $t_m^* \simeq 2 $: 
 		\begin{align}
 			\frac{m-1}{m+1} = \left( \frac{m-1}{m}  \right)^{t_m^*} \iff t_m^* =
 			\frac{\log \left( 1- \frac 2 {m+1} \right) }{\log \left( 1- \frac 1m \right)} \simeq
 			\frac{2m}{m+1} \simeq 2\,.
 		\end{align}
 	\end{proof}
 \end{lemma}

\newpage
\subsection{Additional Experimental results}

\subsubsection{Comparison of the two adaptive sample size schemes for \methodname}

We here compare the {\sc Linear} and {\sc Alternating} schemes on the collection of real datasets presented in Table~\ref{table:datasets} for a regularizer $\lambda = n^{-\frac{1}{2}}$. The results for the empirical and expected risk shown in Figure~\ref{fig:results_iid} and Figure~\ref{fig:results_test_iid} show that the {\sc Alternating} scheme slightly outperforms the {\sc Linear} strategy.

\begin{figure}[H]
	\begin{center}
          \begin{tabular}{@{}c@{\hspace{5mm}}c@{\hspace{5mm}}c@{}}
            \includegraphics[width=0.3\linewidth]{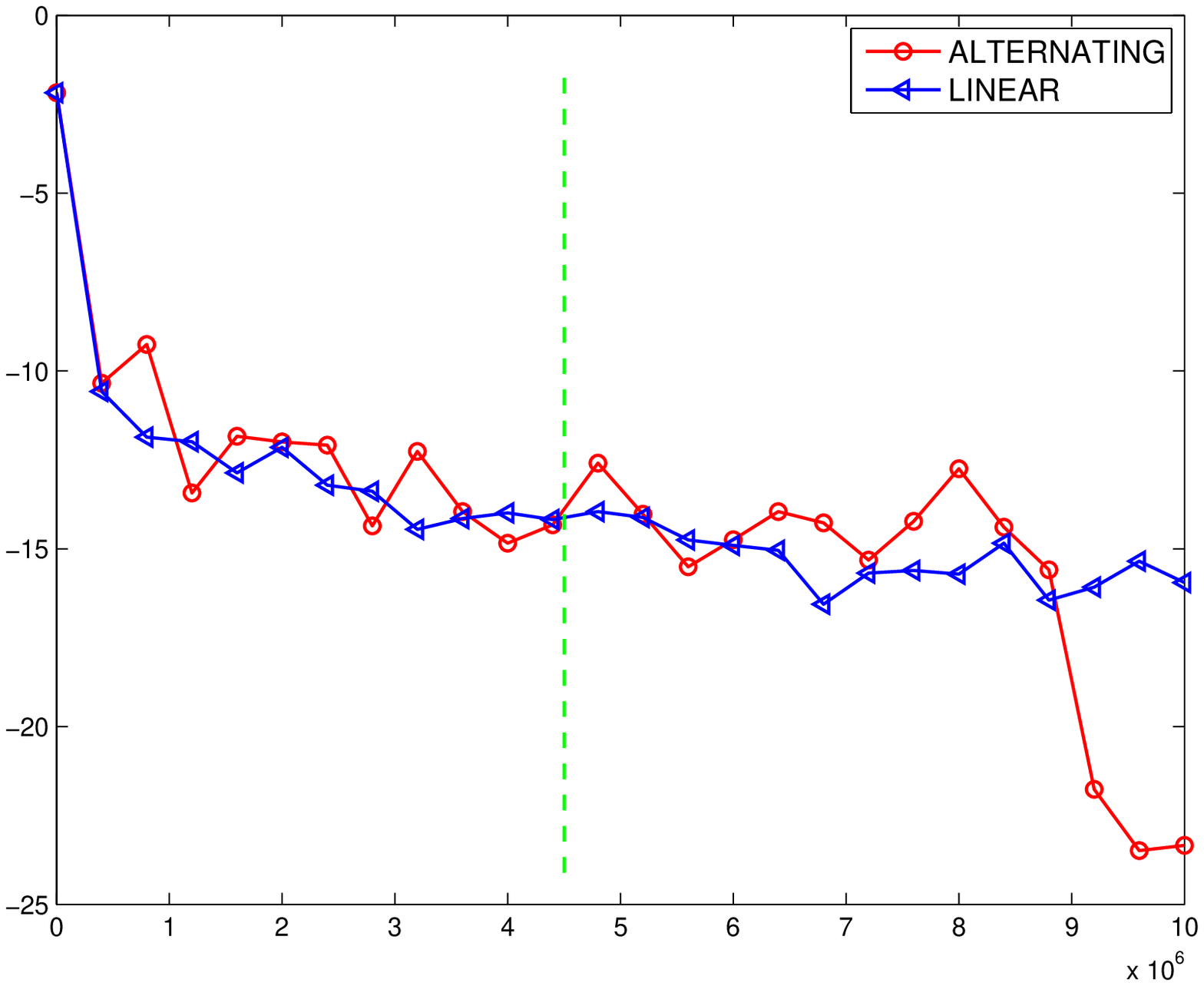} & 
            \includegraphics[width=0.3\linewidth]{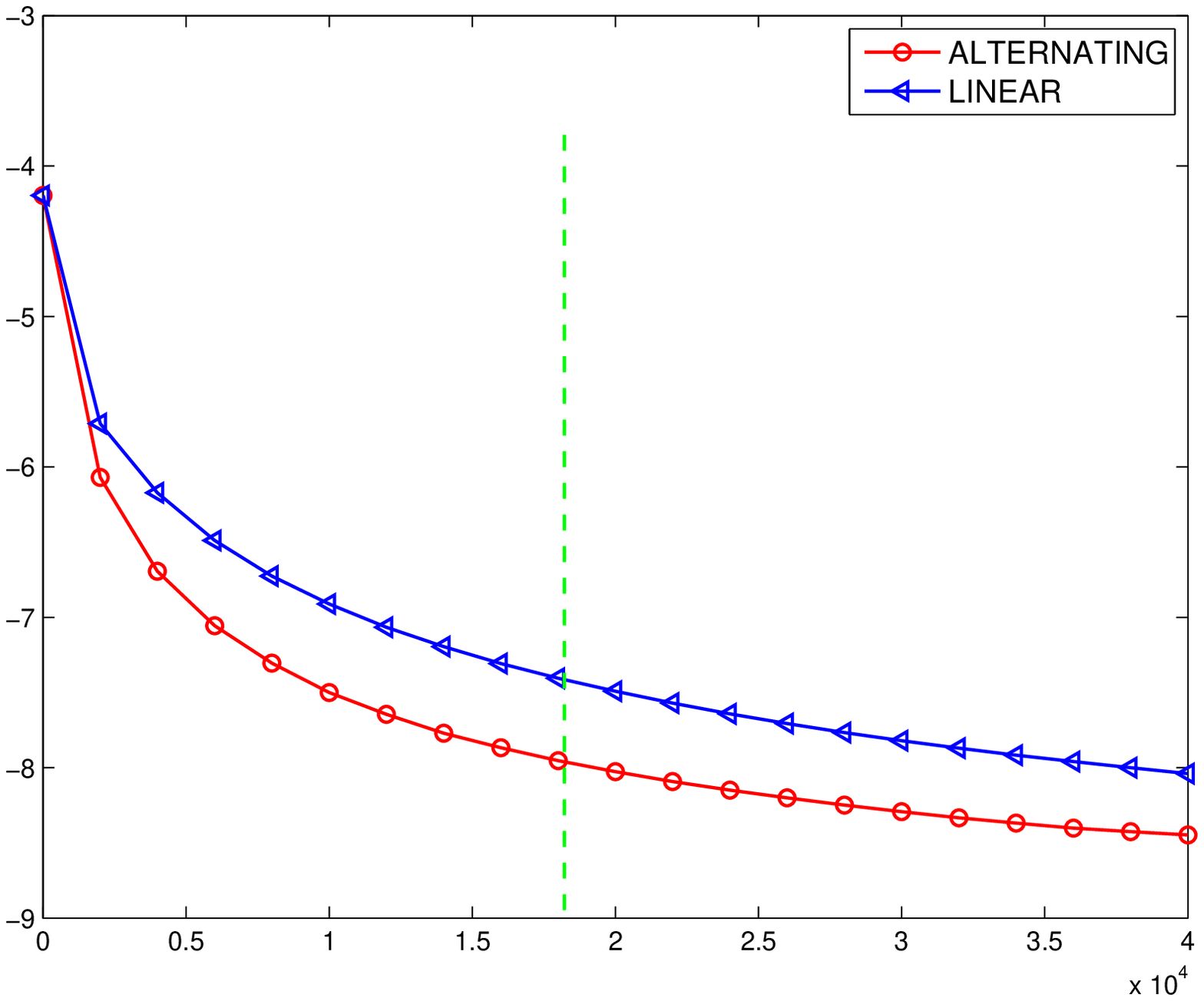} &
             \includegraphics[width=0.3\linewidth]{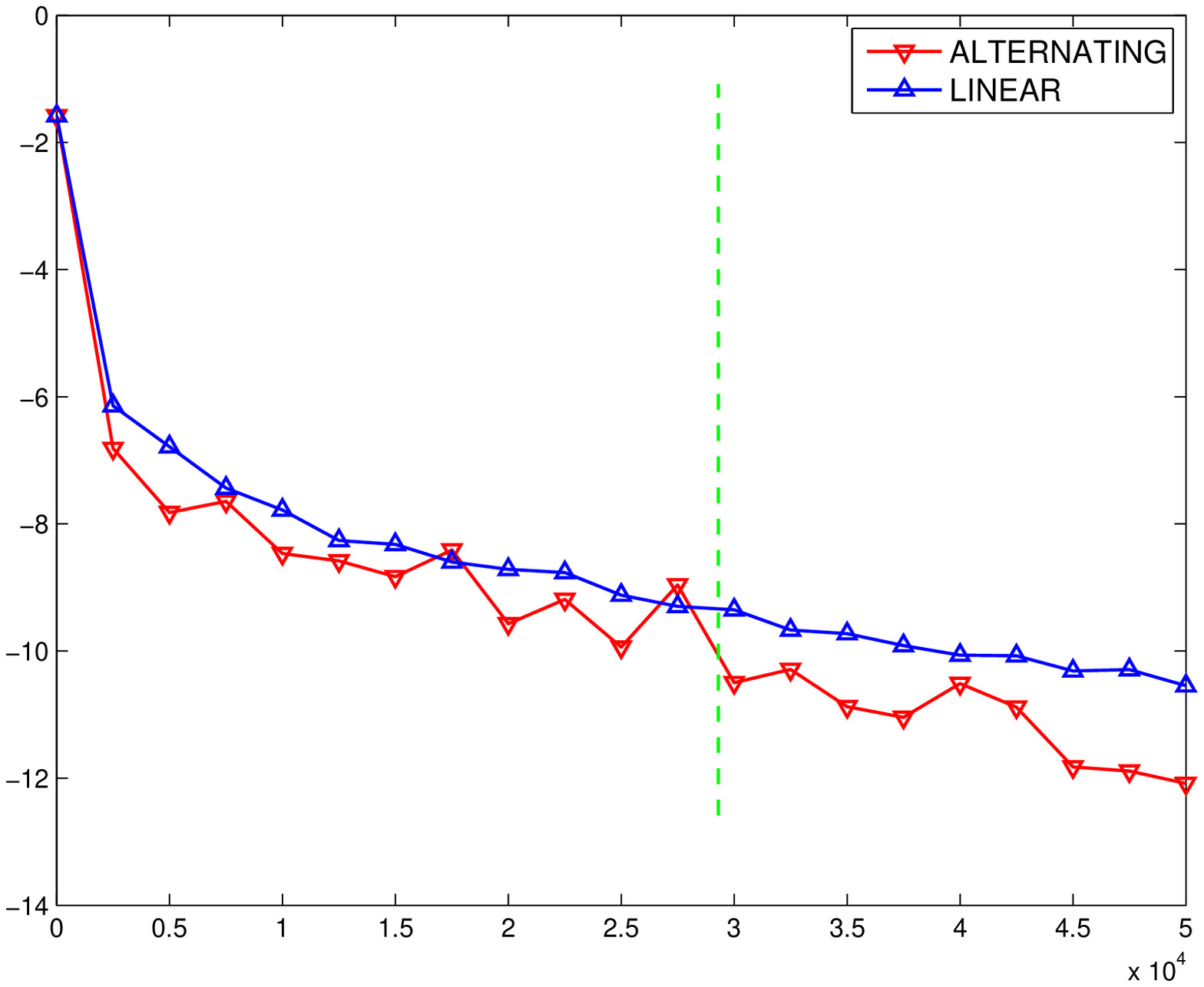} \\
            1. {\sc SUSY} &
            2. {\sc rcv} &
            3. {\sc a9a}
            \\
            \includegraphics[width=0.3\linewidth]{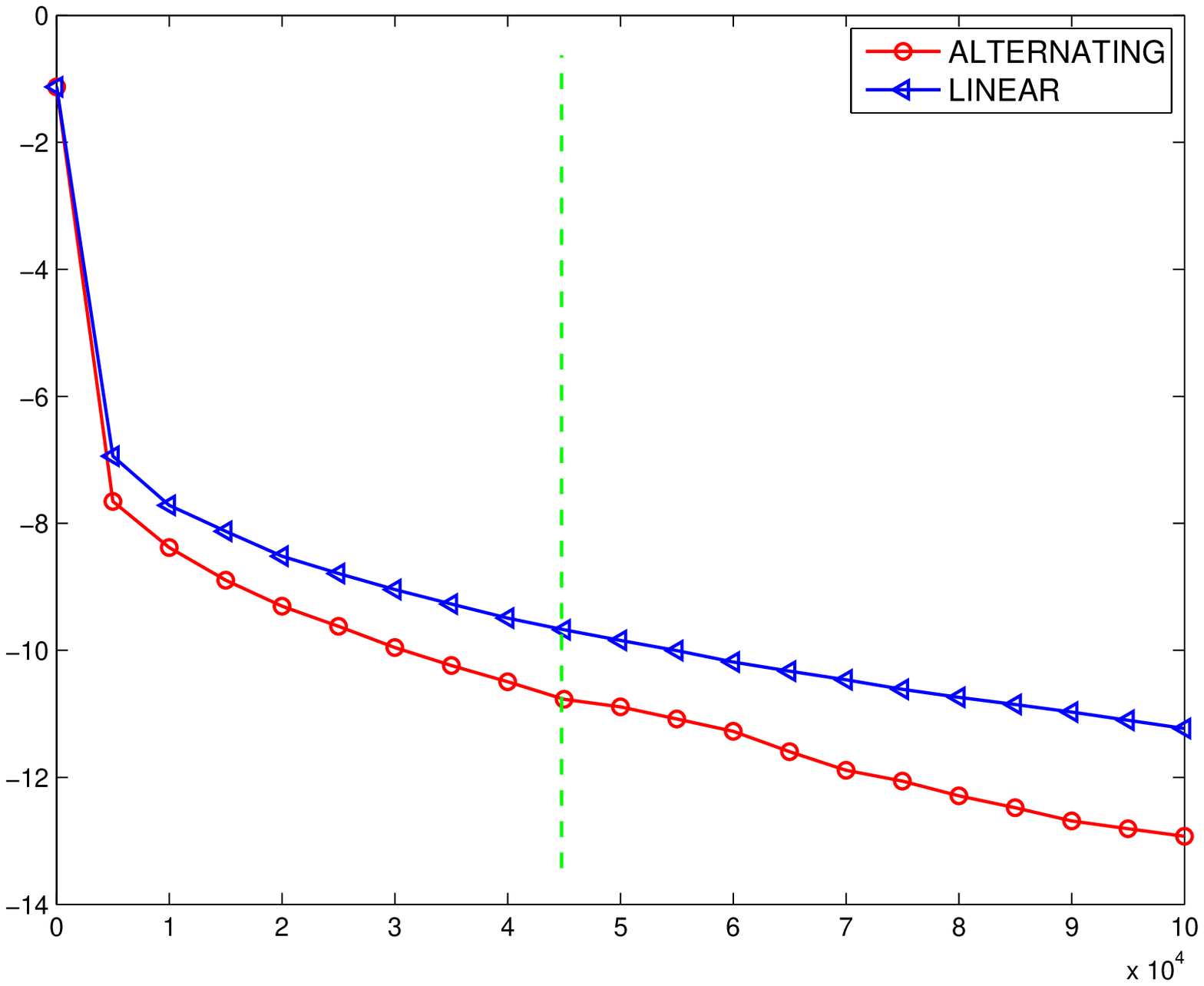} &
            \includegraphics[width=0.3\linewidth]{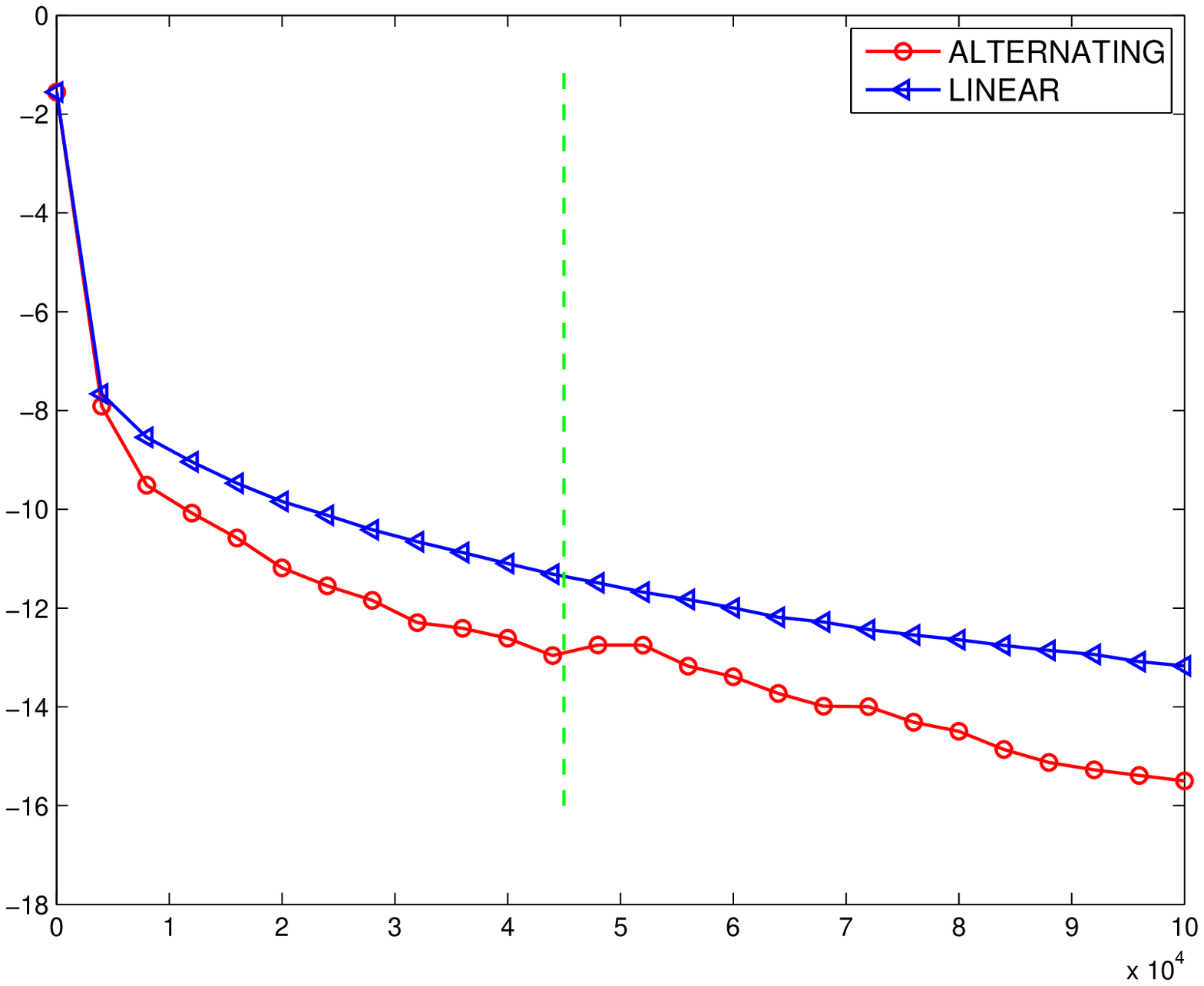} & 
             \includegraphics[width=0.3\linewidth]{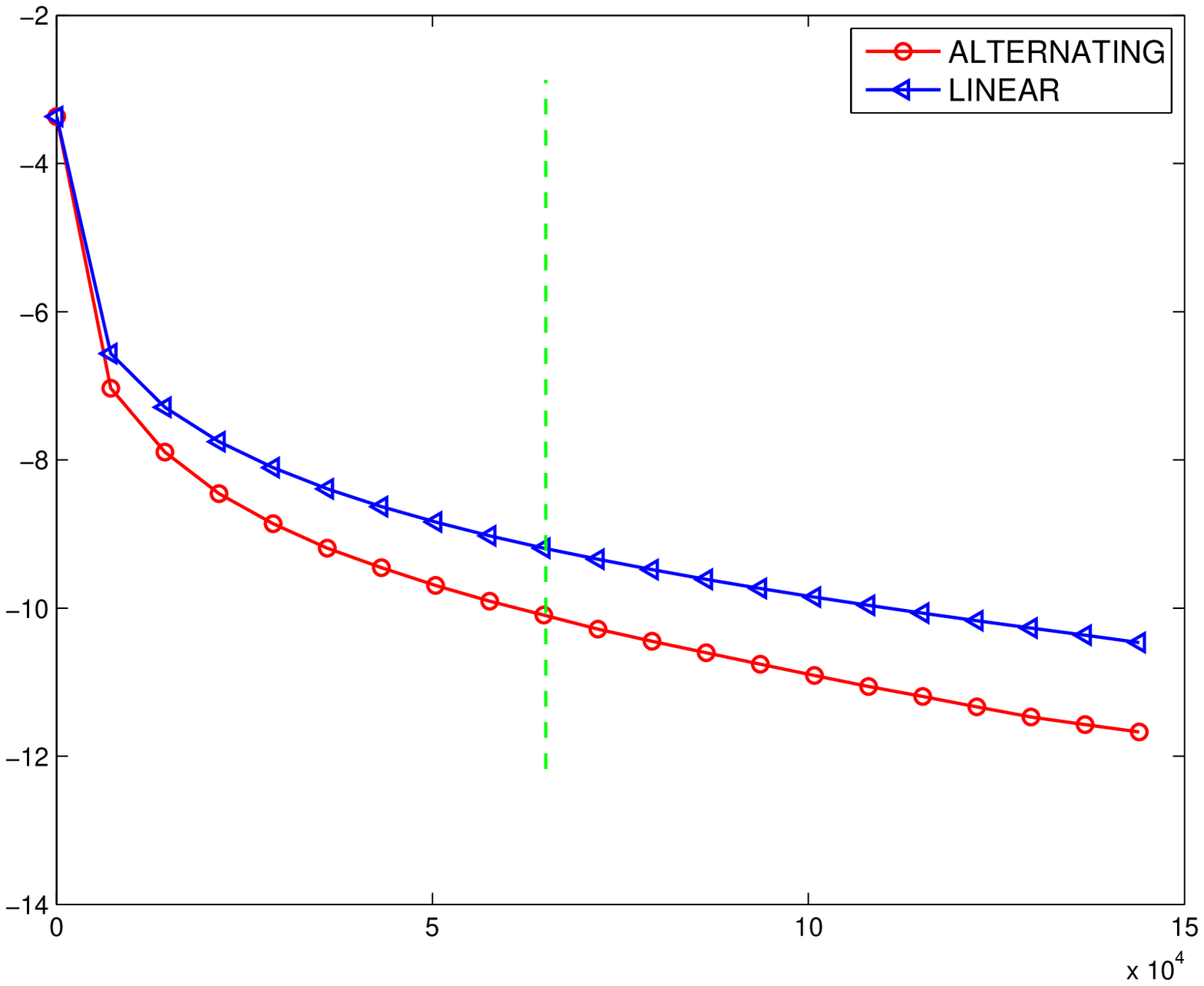}   
            \\ 
            4. {\sc w8a} &
            5. {\sc ijcnn1} & 
            6. {\sc real-sim} 
            \\
             \includegraphics[width=0.3\linewidth]{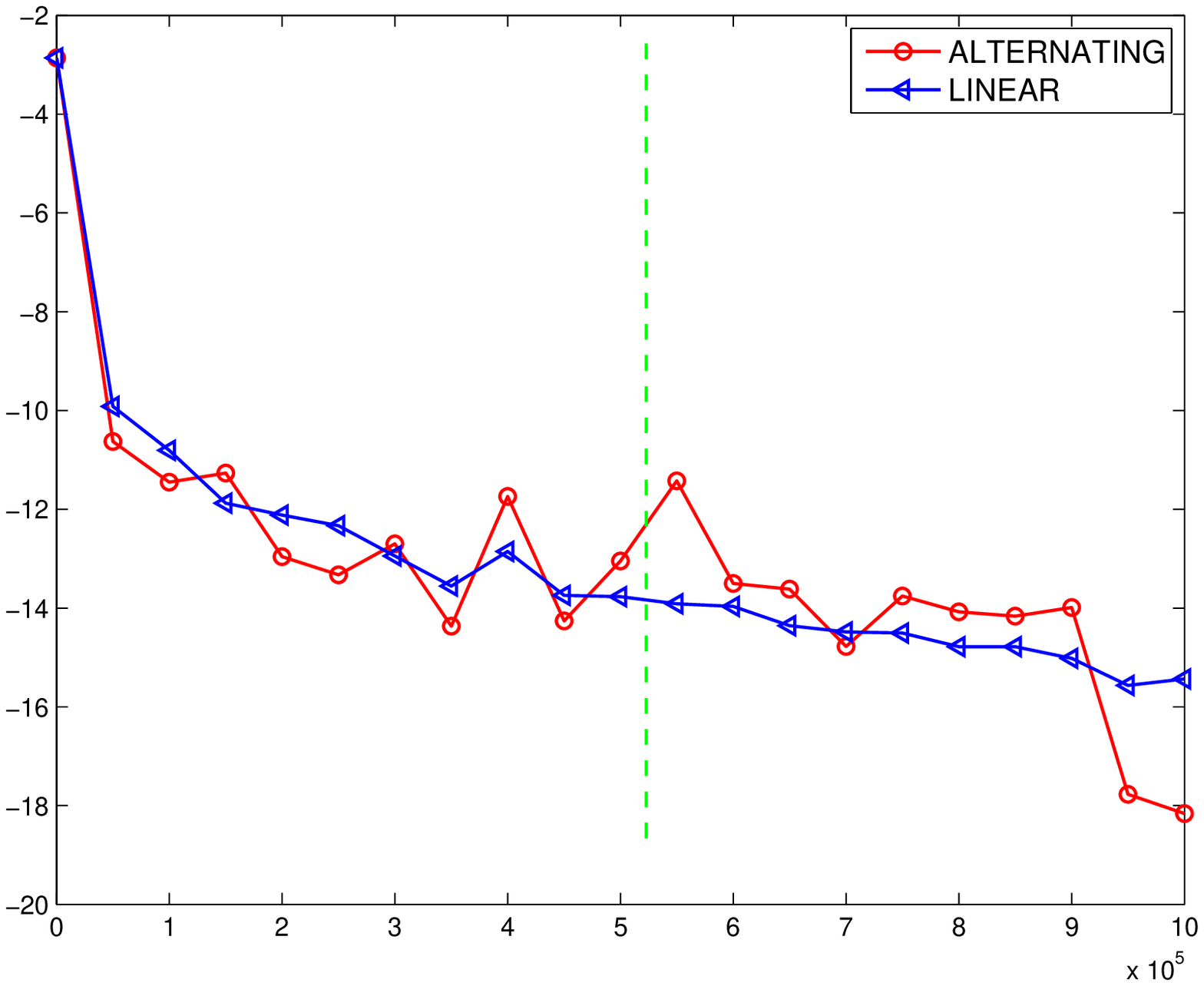} \\ 
            7. {\sc covtype} & 
	  \end{tabular}
          \caption{{\it Suboptimality on the empirical risk.}~The
          vertical axis shows the suboptimality of the empirical risk, i.e. $\log_{2}
          \E_{10} \left[ \risk_{\T}(\w^t) - \risk_{\T}^* \right]$ where the
          expectation is taken over 10 independent runs. The training set
          includes 90\% of the data. The vertical green dashed line is drawn after exactly one epoch over the data. }
          \label{fig:results_iid}
	\end{center}
	
\end{figure}

\begin{figure}[H]
	\begin{center}
          \begin{tabular}{@{}c@{\hspace{5mm}}c@{\hspace{5mm}}c@{}}
            \includegraphics[width=0.3\linewidth]{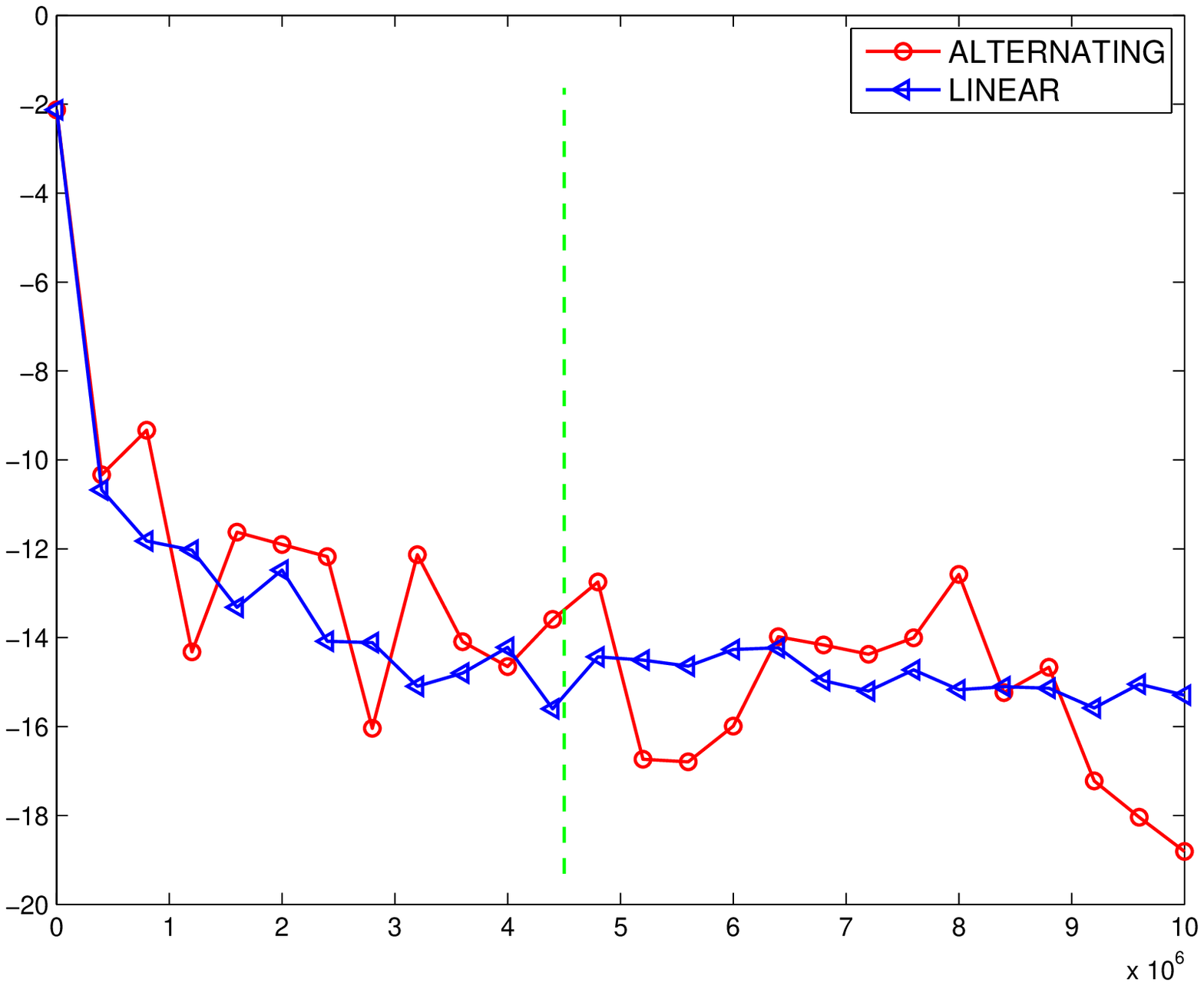} &
            \includegraphics[width=0.3\linewidth]{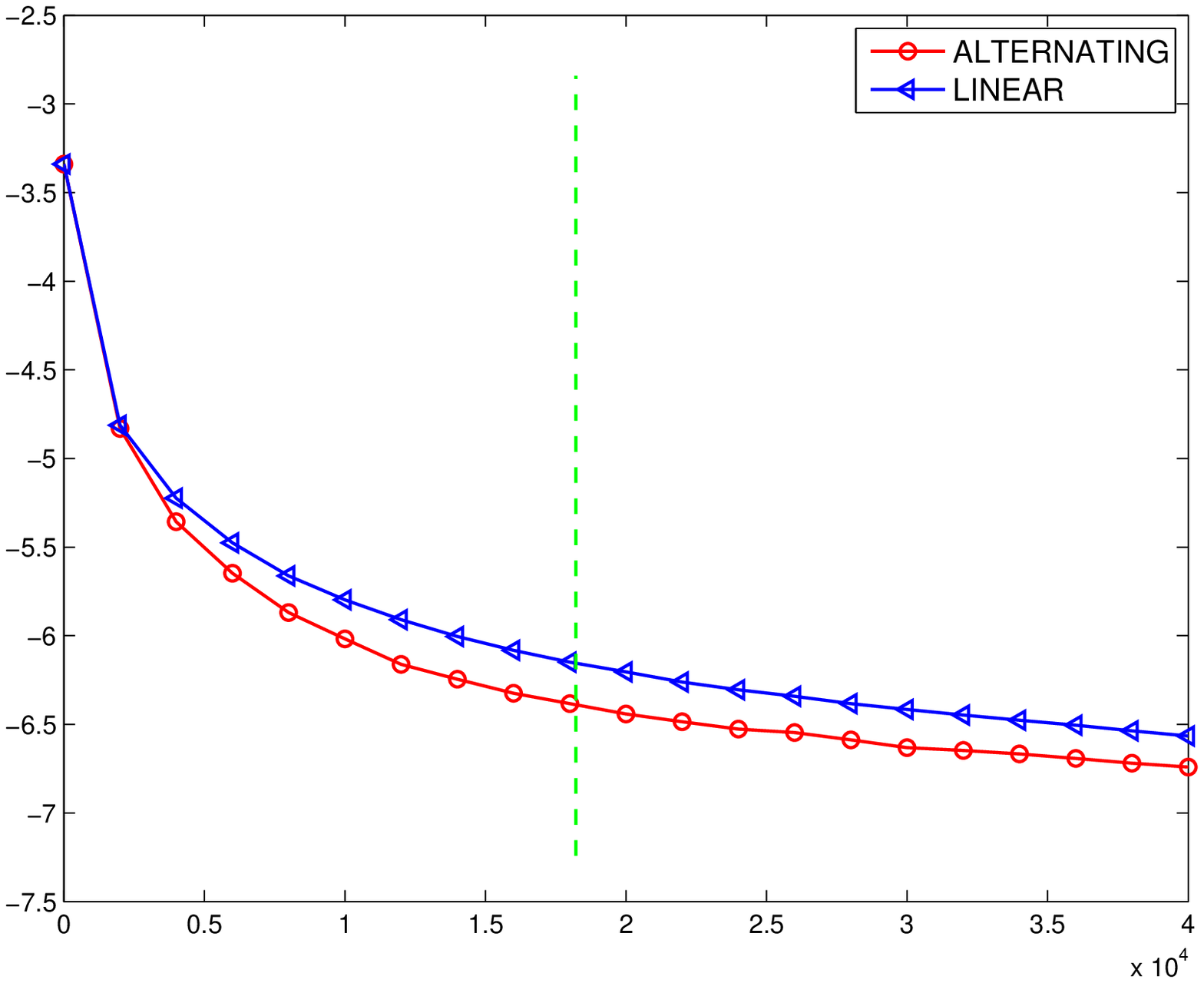} &
            \includegraphics[width=0.3\linewidth]{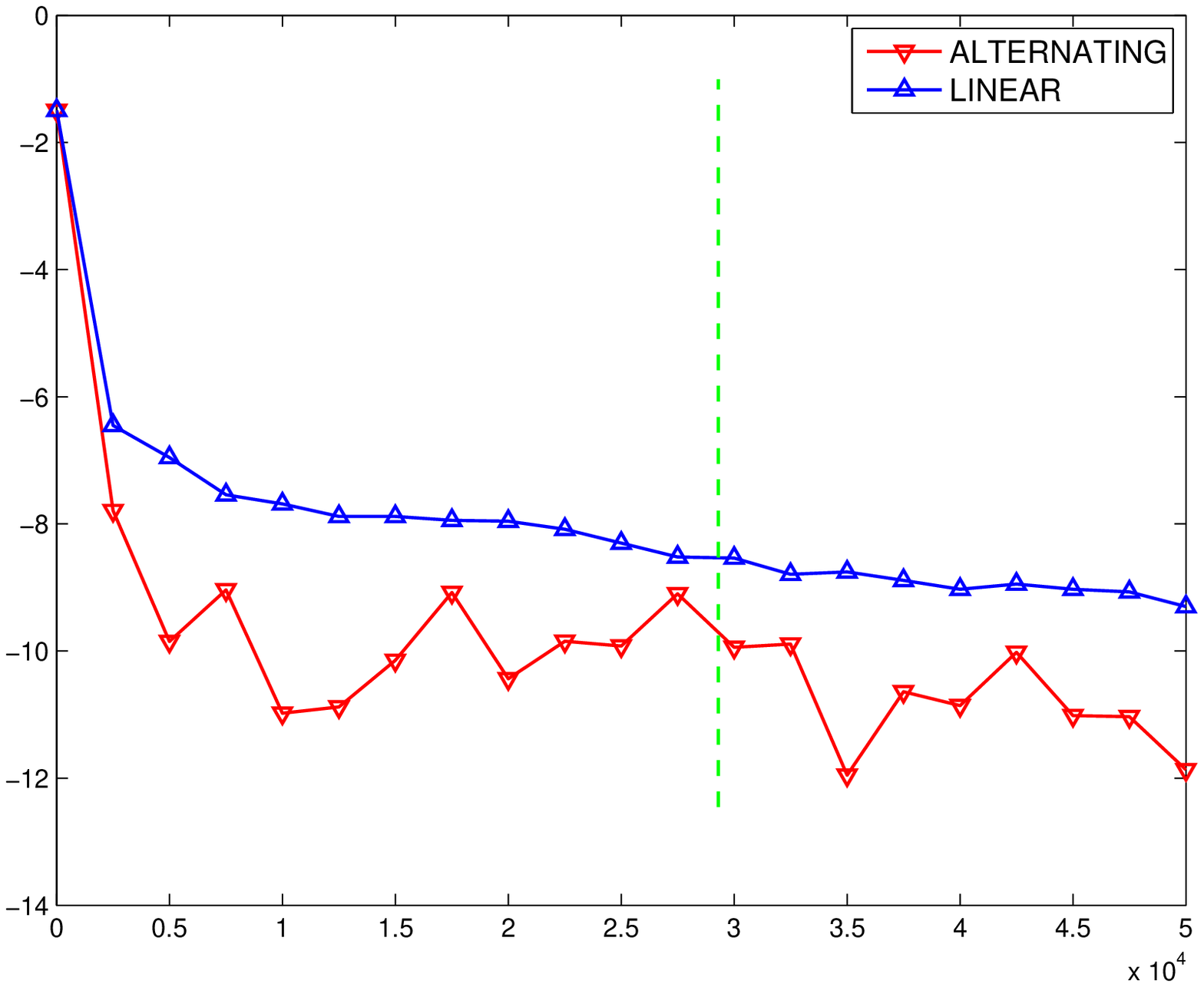} \\
            1. {\sc SUSY} &
            2. {\sc rcv} &
            3. {\sc a9a}
            \\
            \includegraphics[width=0.3\linewidth]{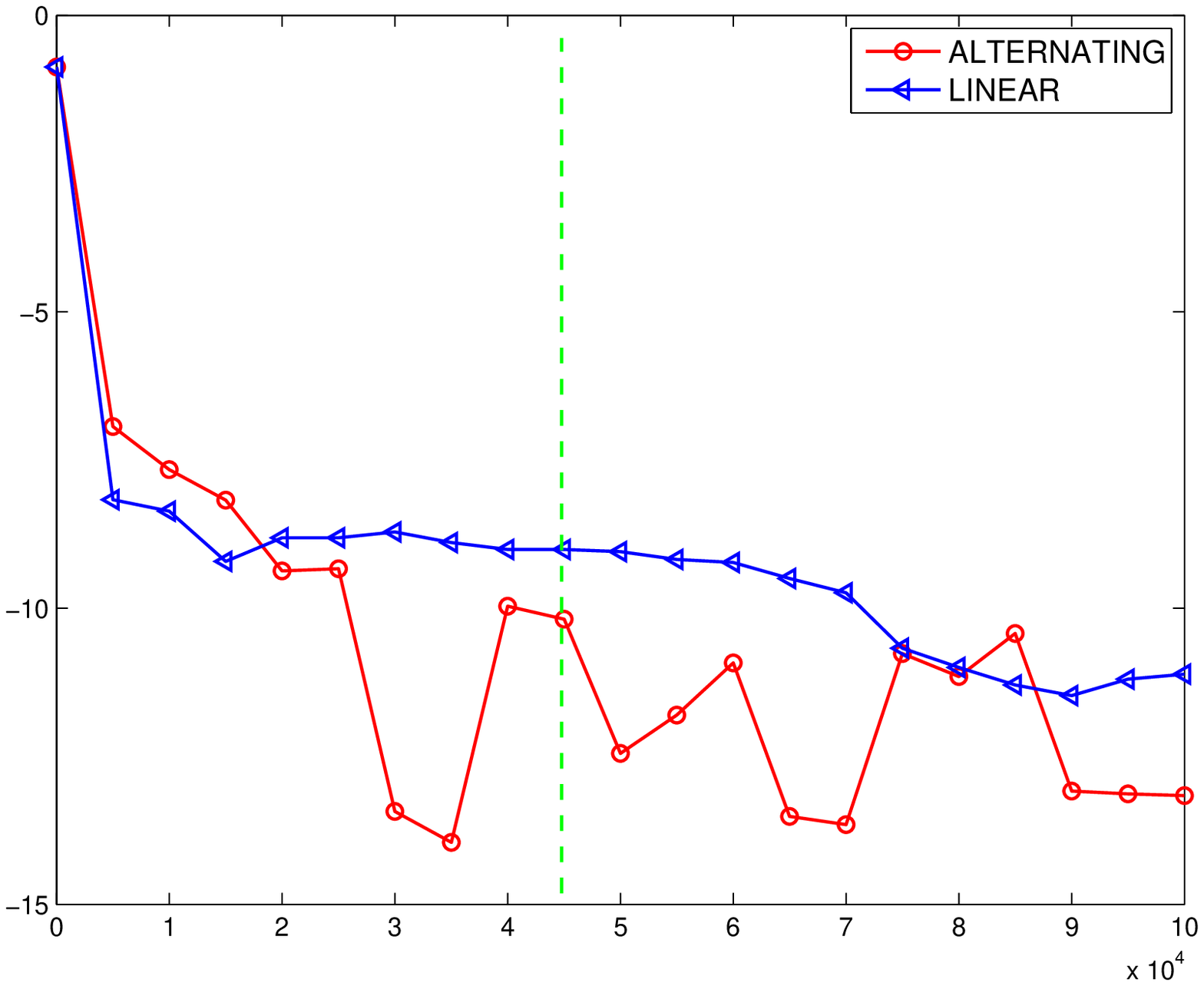} &
            \includegraphics[width=0.3\linewidth]{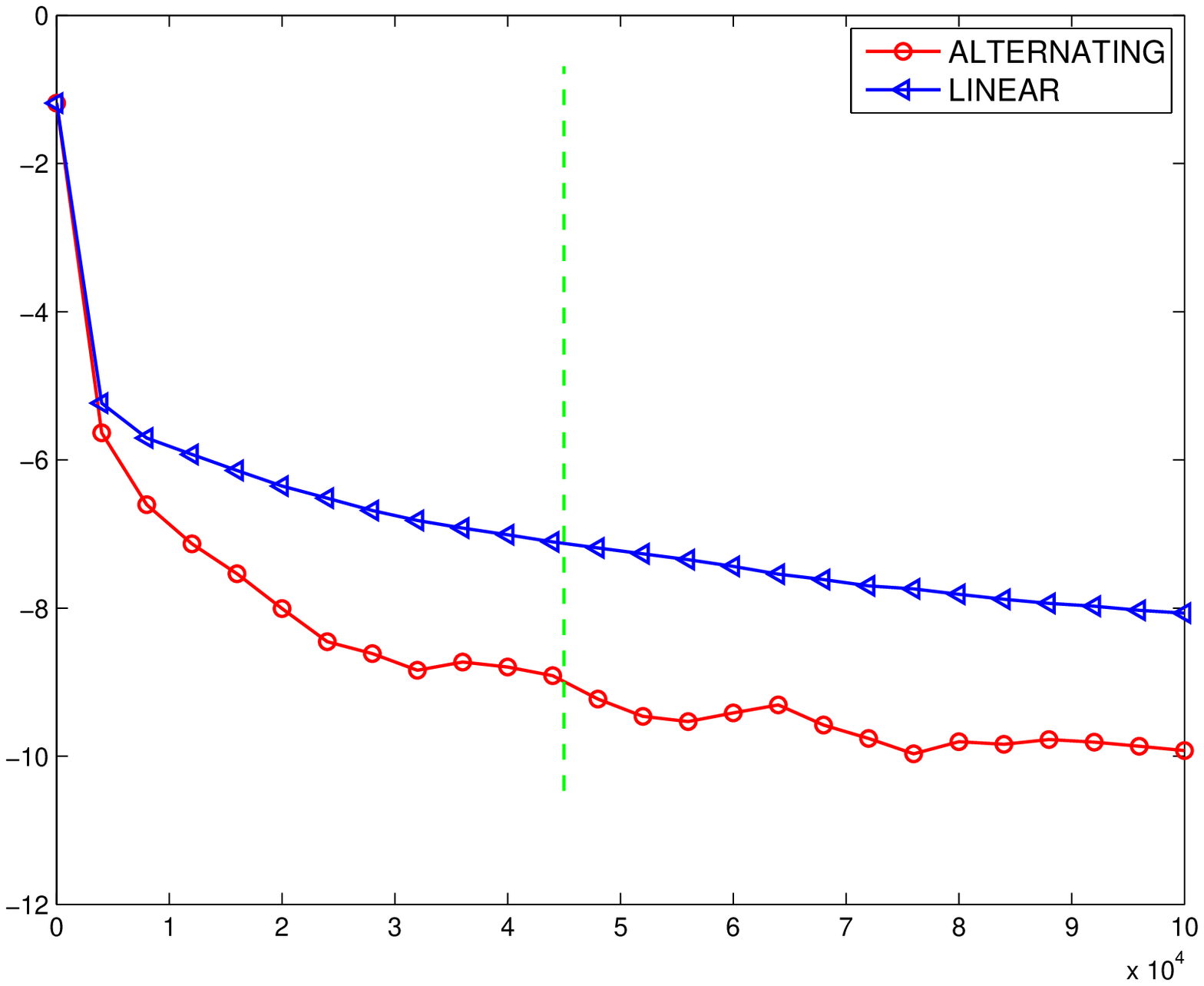} & 
             \includegraphics[width=0.3\linewidth]{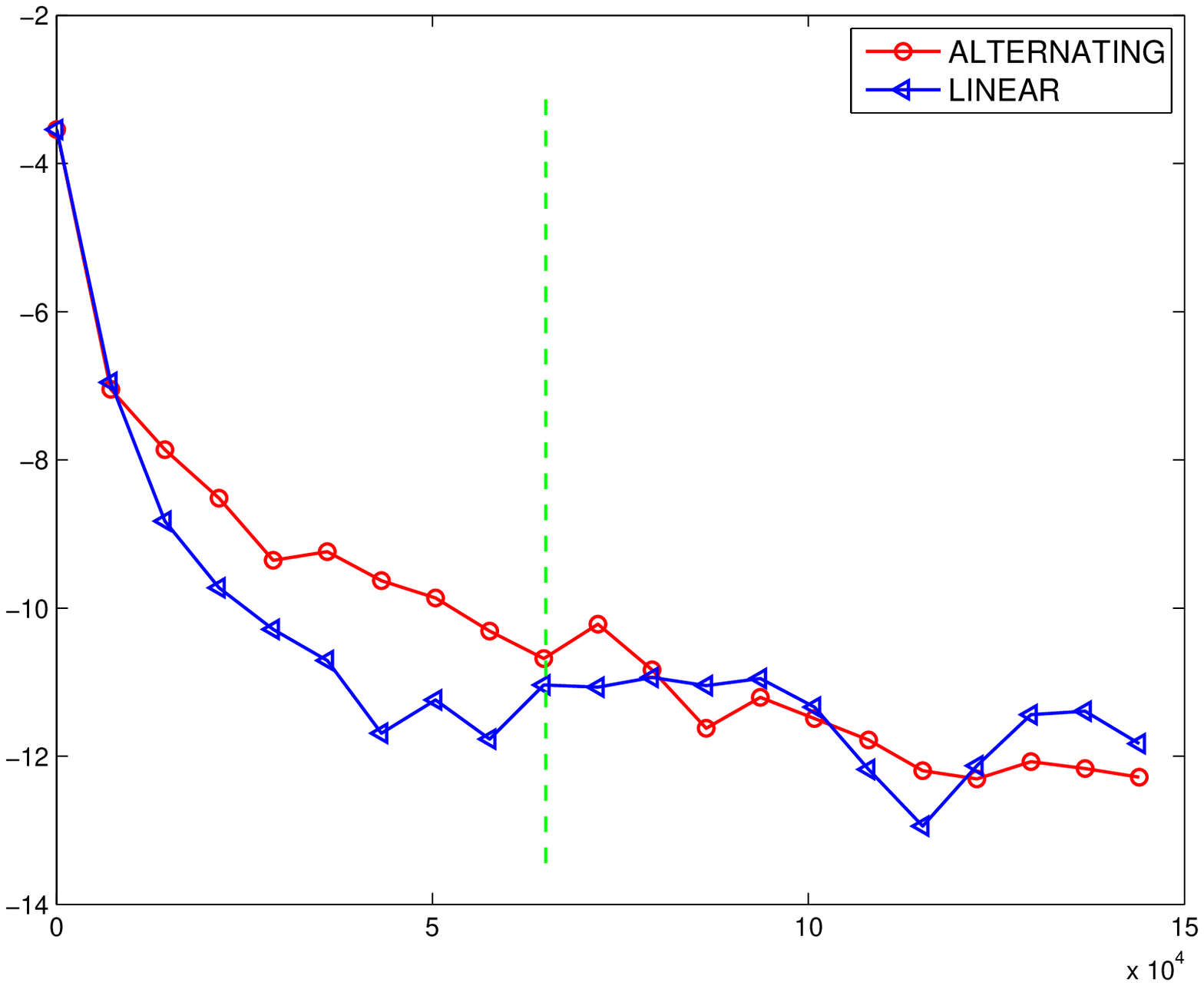}   
            \\ 
            4. {\sc w8a} &
            5. {\sc ijcnn1} & 
            6. {\sc real-sim} 
            \\
             \includegraphics[width=0.3\linewidth]{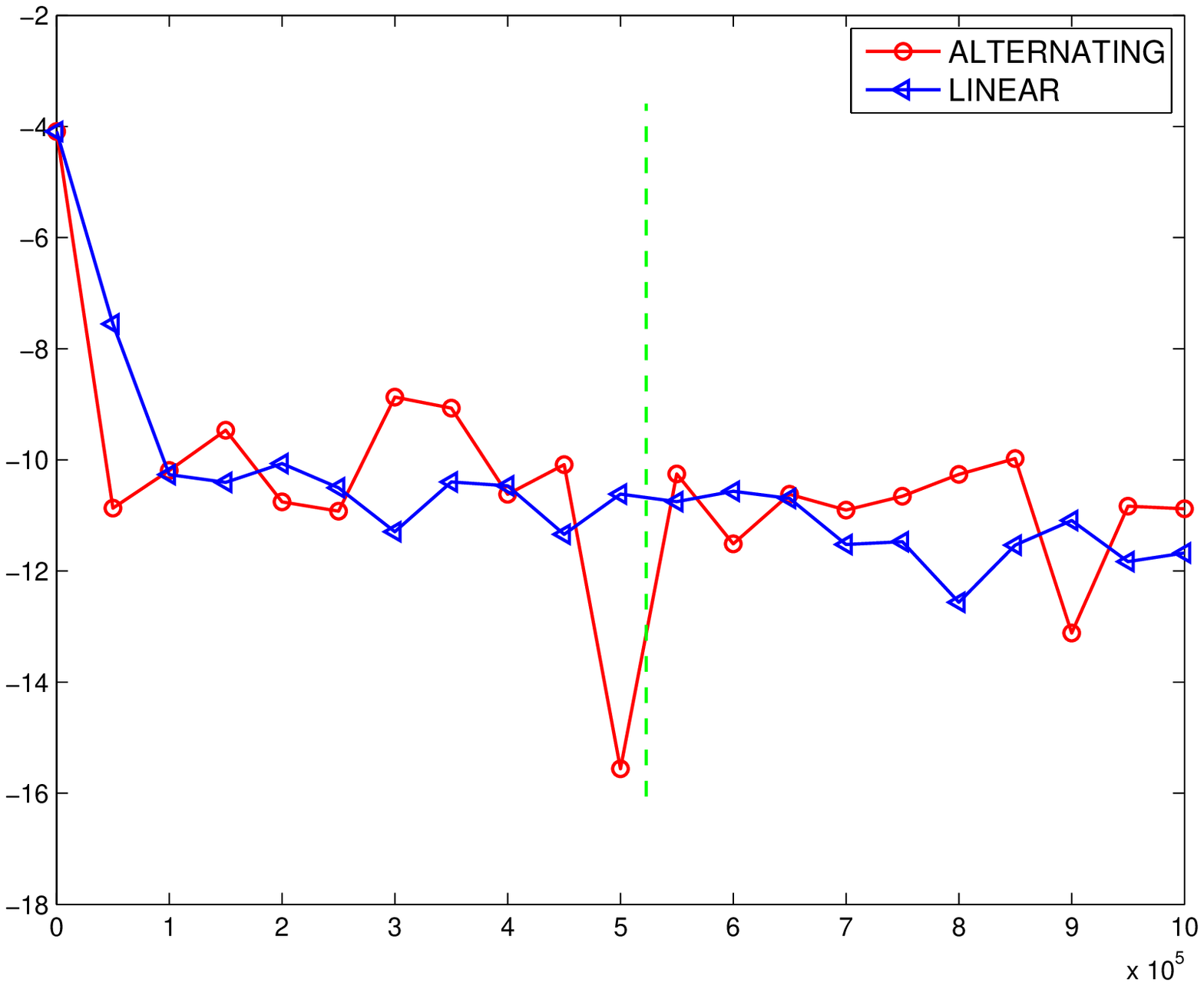} \\
            7. {\sc covtype} & 
	  \end{tabular}
          \caption{{\it Suboptimality on the expected risk.}~The
          vertical axis shows the suboptimality of the expected risk, i.e. $\log_{2}
          \E_{10} \left[ \risk_{\S}(\w^t) - \risk_{\S}(\w^*_{\T}) \right]$,
          where $\S$ is a test set which includes 10\% of the data
          and $\w^*_{\T}$ is the optimum of the empirical risk on $\T$. The vertical green dashed line is drawn after exactly one epoch over the data.
          }
          \label{fig:results_test_iid}
	\end{center}
	
\end{figure}

\subsubsection{Effect of the regularizer}

We here present additional results for various regularizers of the form $\lambda = \frac{1}{n^p}, p < 1$. In the interest of clarity we only show results on four datasets. We can see a similar trend to the main results presented in the paper for $\lambda = \frac{1}{\sqrt{n}}$ where \methodname shows very fast convergence in terms of both empirical and expected risk. SGD is also very competitive and typically achieves faster convergence than the other baselines, however, its behaviour is not stable throughout all the datasets.

\begin{figure}[H]
	\begin{center}
          \begin{tabular}{@{}c@{\hspace{5mm}}c@{\hspace{5mm}}c@{}}
            \includegraphics[width=0.27\linewidth]{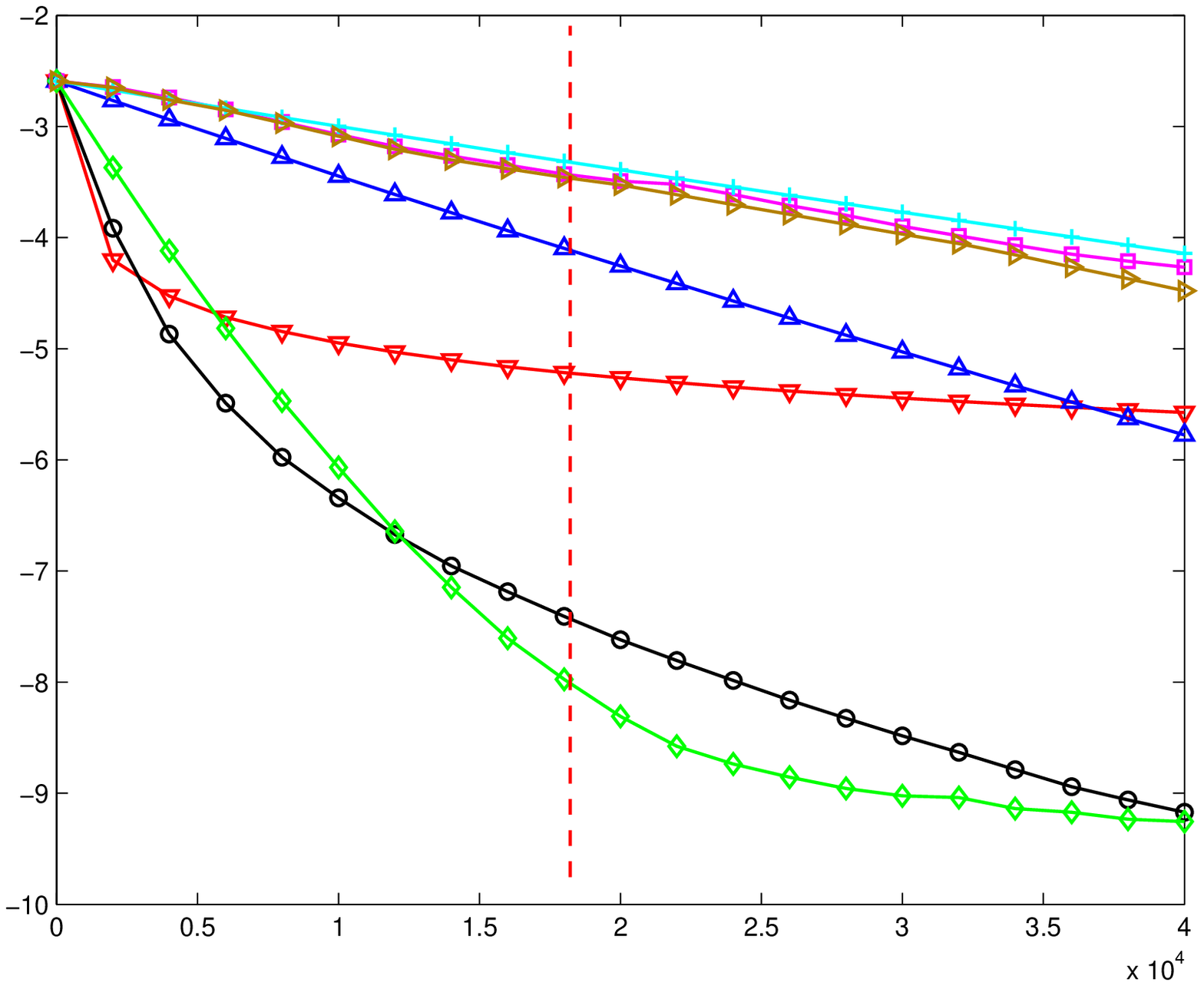} &
            \includegraphics[width=0.27\linewidth]{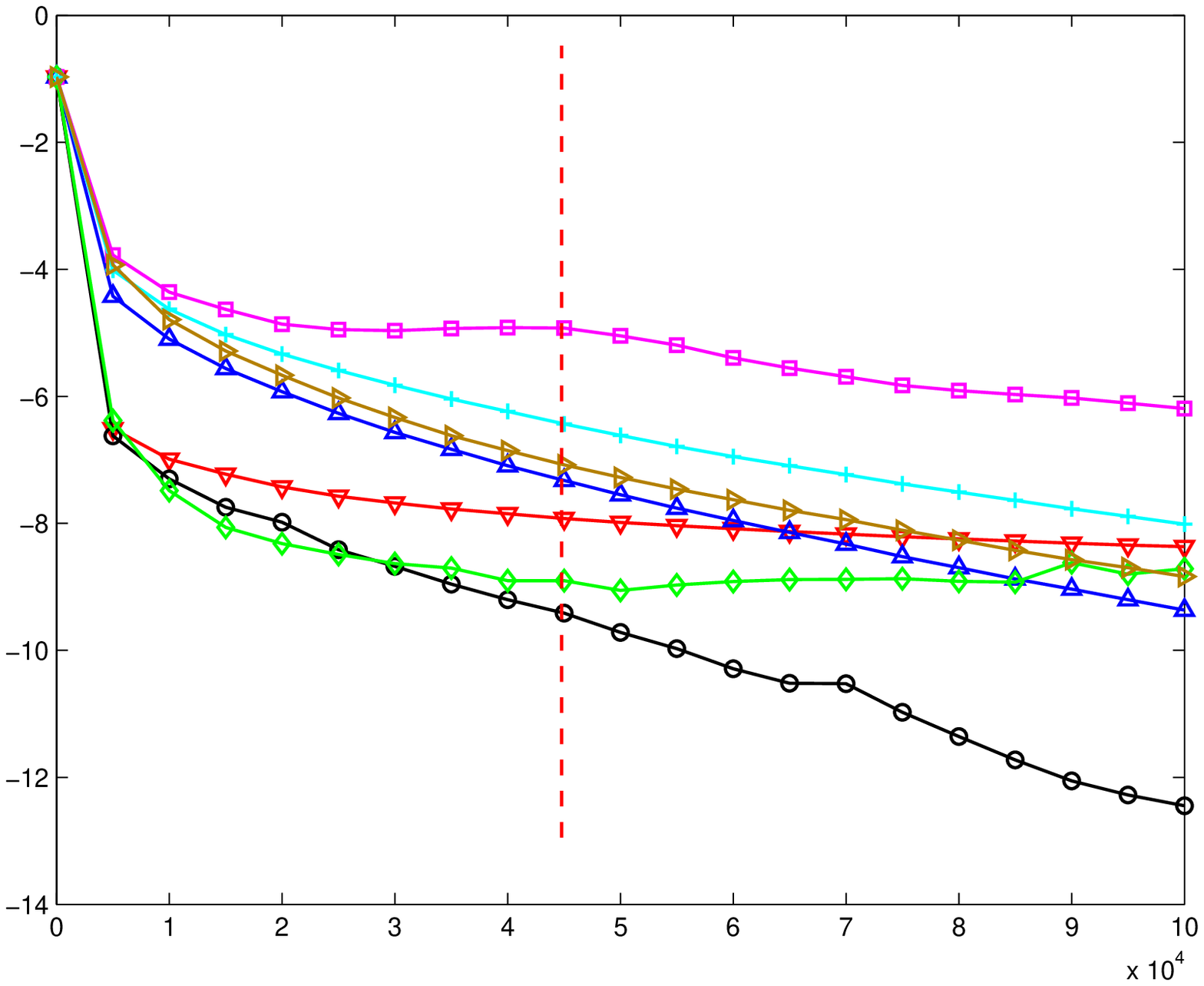} &
            \includegraphics[width=0.2\linewidth]{legend} \\
            1. {\sc rcv} &
            2. {\sc w8a} &
            \\
             \includegraphics[width=0.27\linewidth]{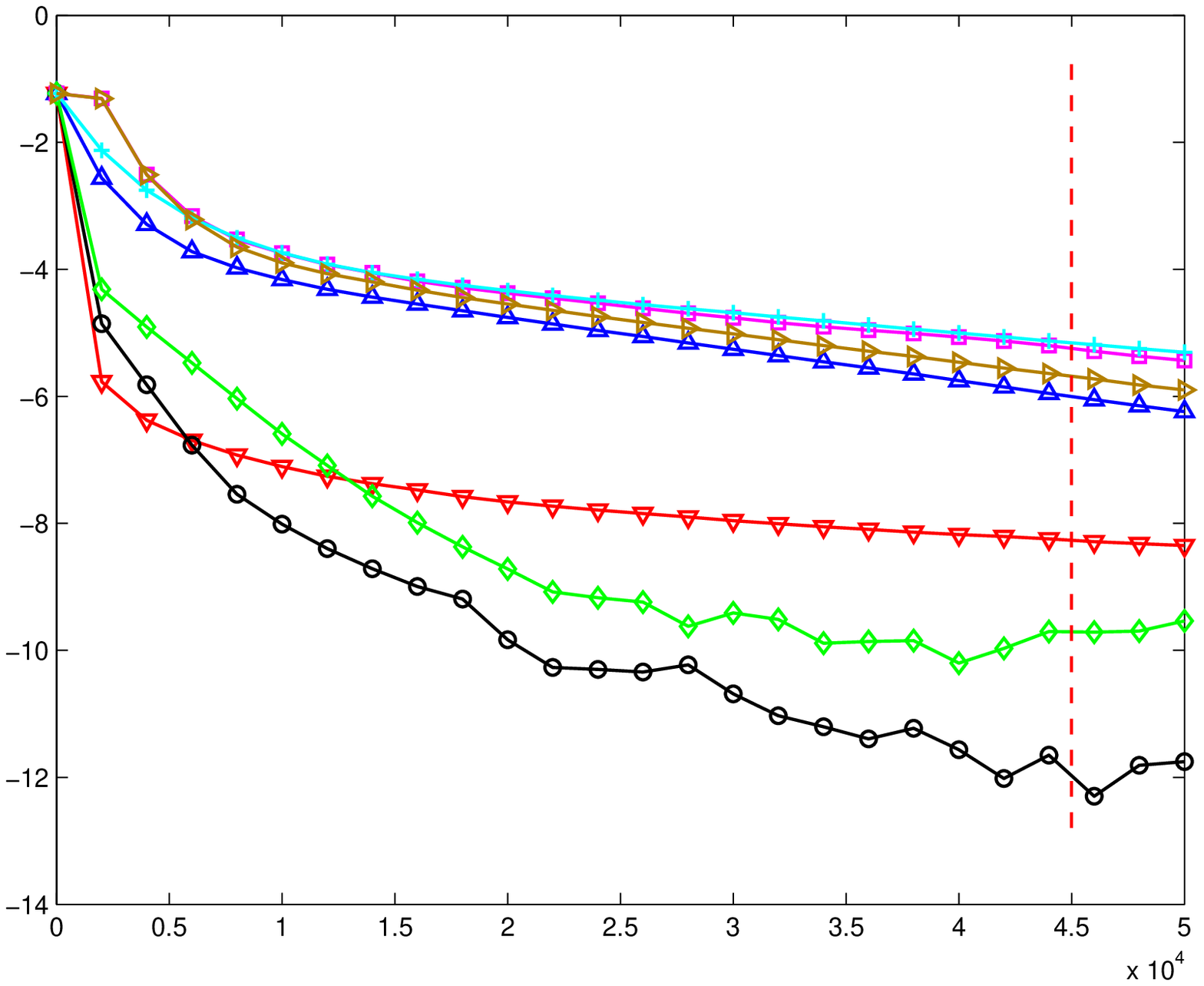}&
            \includegraphics[width=0.27\linewidth]{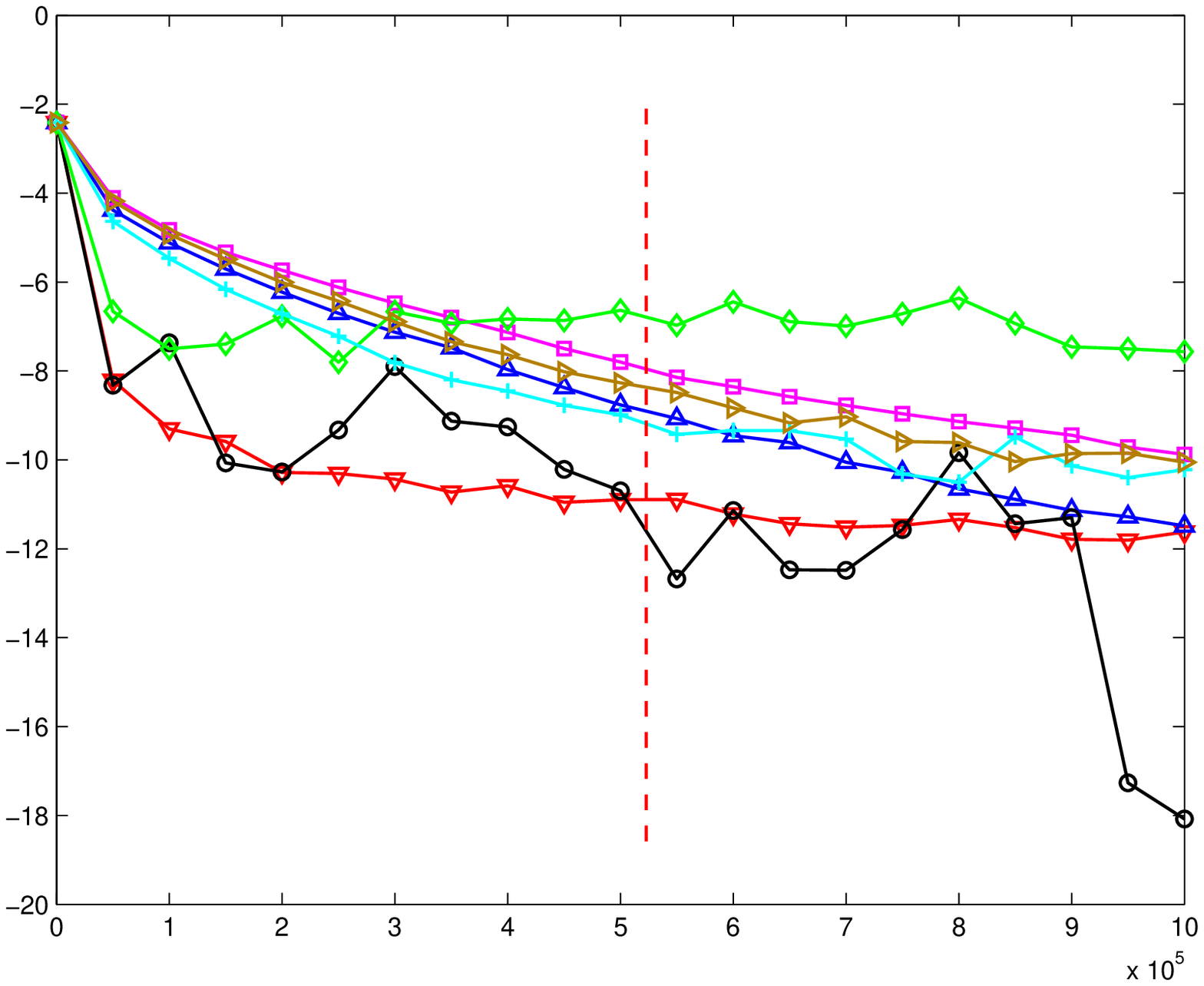} & 
               
            \\ 
            3. {\sc ijcnn1} &
            4. {\sc covtype} & 
              \\
	  \end{tabular}
          \caption{Suboptimality on the empirical risk with regularizer
          $\lambda = n^{-\frac{2}{3}}$}
          \vspace{-5mm}
          \label{fig:results_medium_regularizer}
	\end{center}
\end{figure}

\begin{figure}[H]
	\begin{center}
          \begin{tabular}{@{}c@{\hspace{5mm}}c@{\hspace{5mm}}c@{}}
            \includegraphics[width=0.27\linewidth]{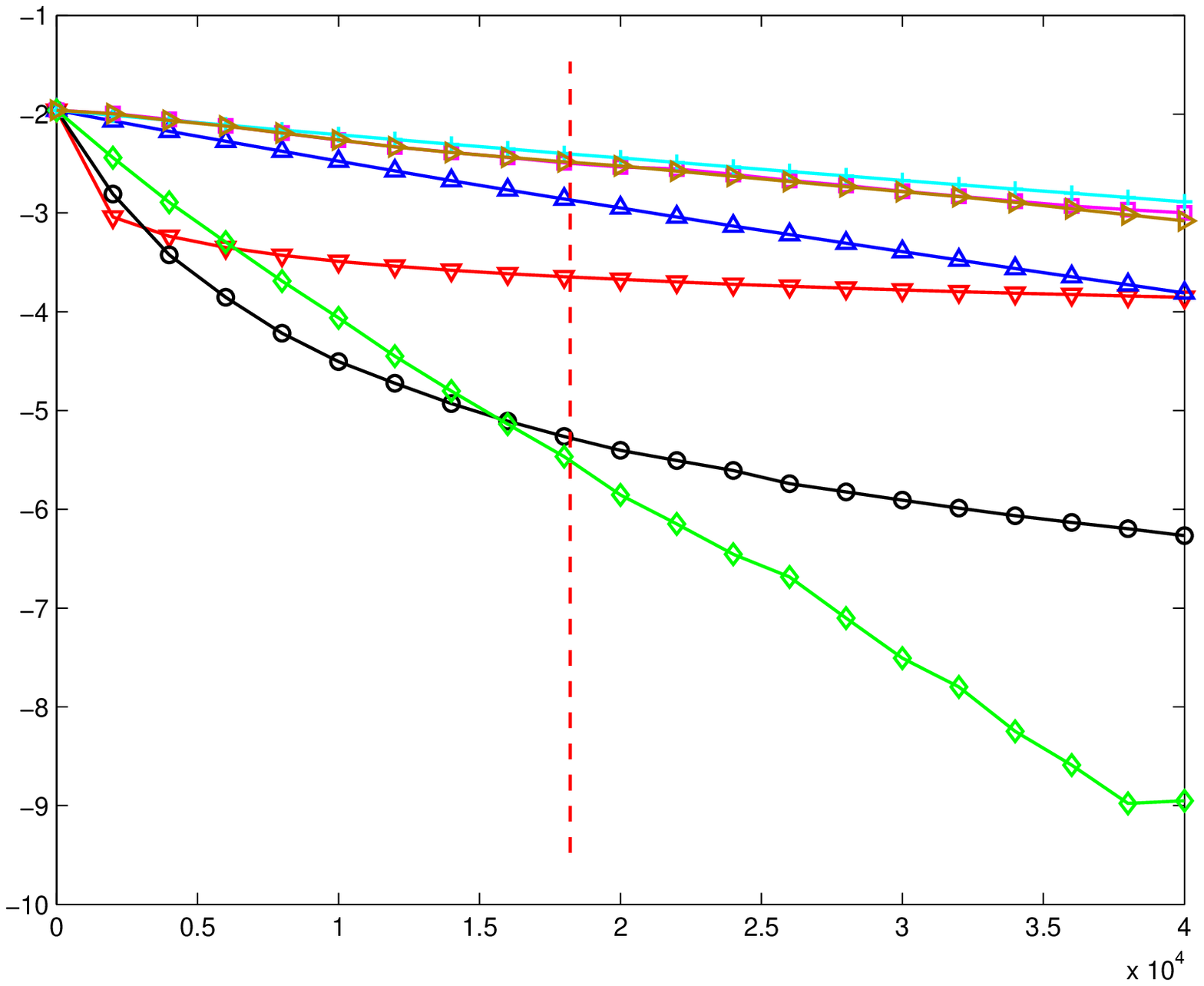} &
            \includegraphics[width=0.27\linewidth]{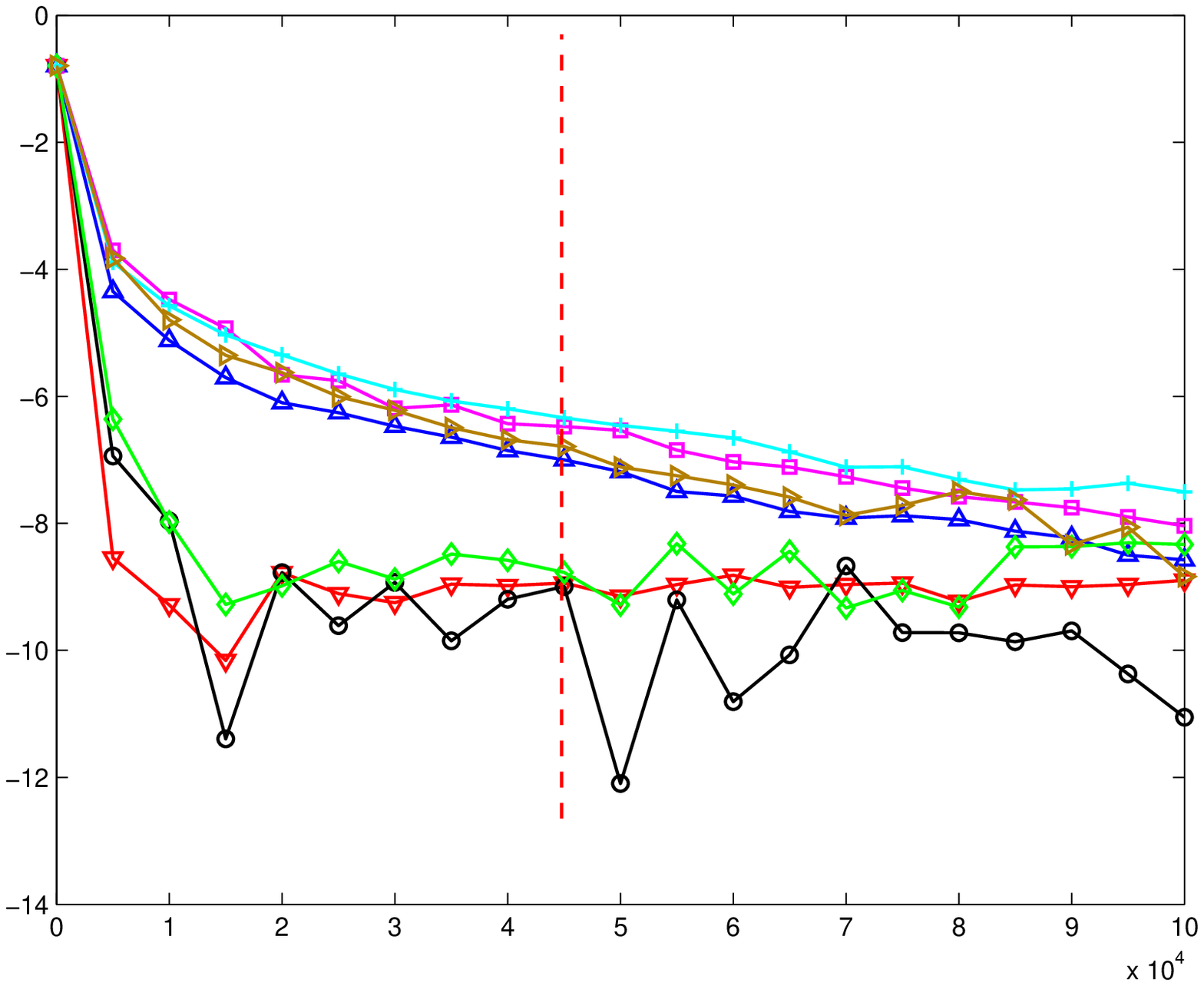} &
            \includegraphics[width=0.2\linewidth]{legend} \\
            1. {\sc rcv} &
            2. {\sc w8a} &
            \\
             \includegraphics[width=0.27\linewidth]{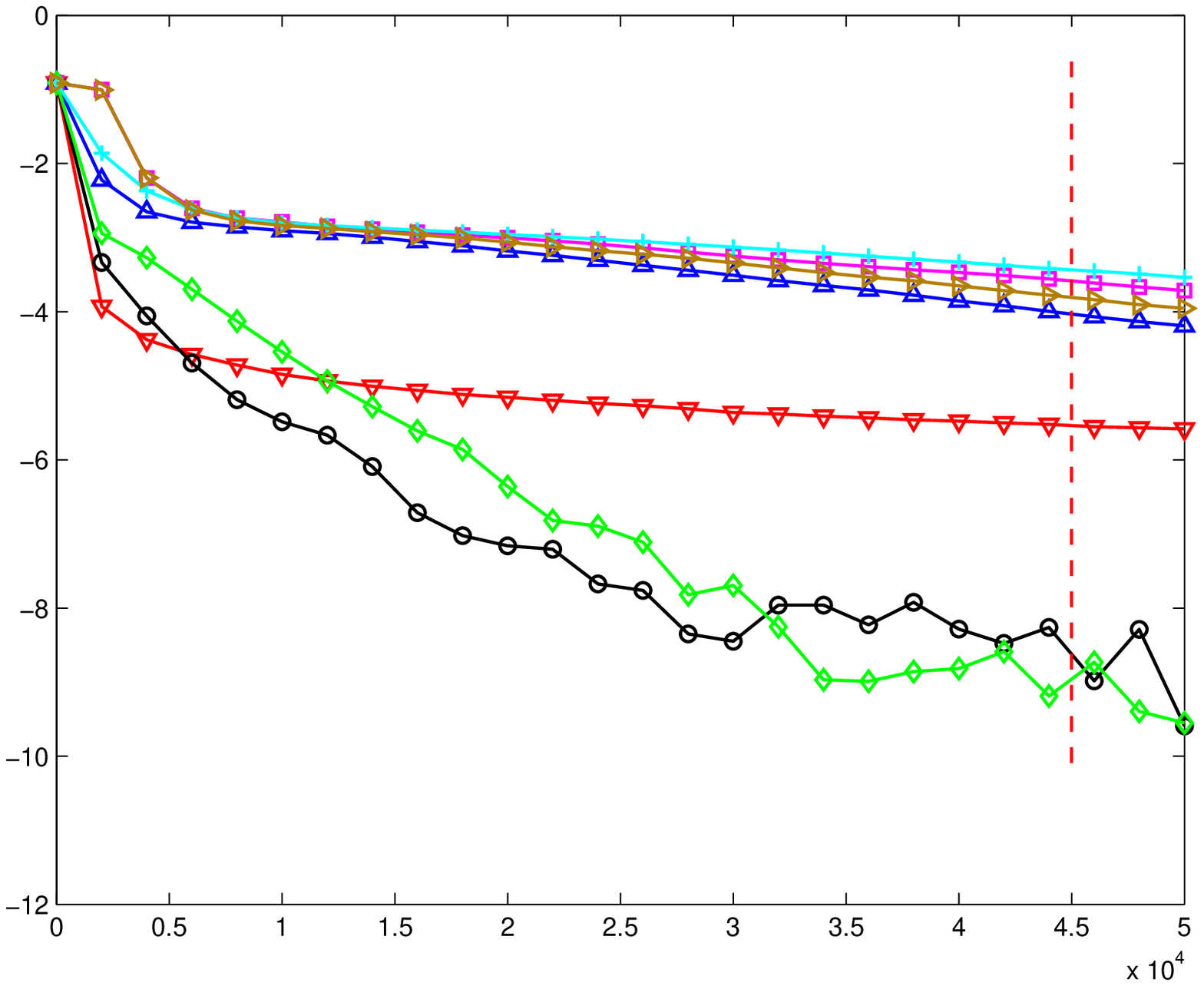}&
            \includegraphics[width=0.27\linewidth]{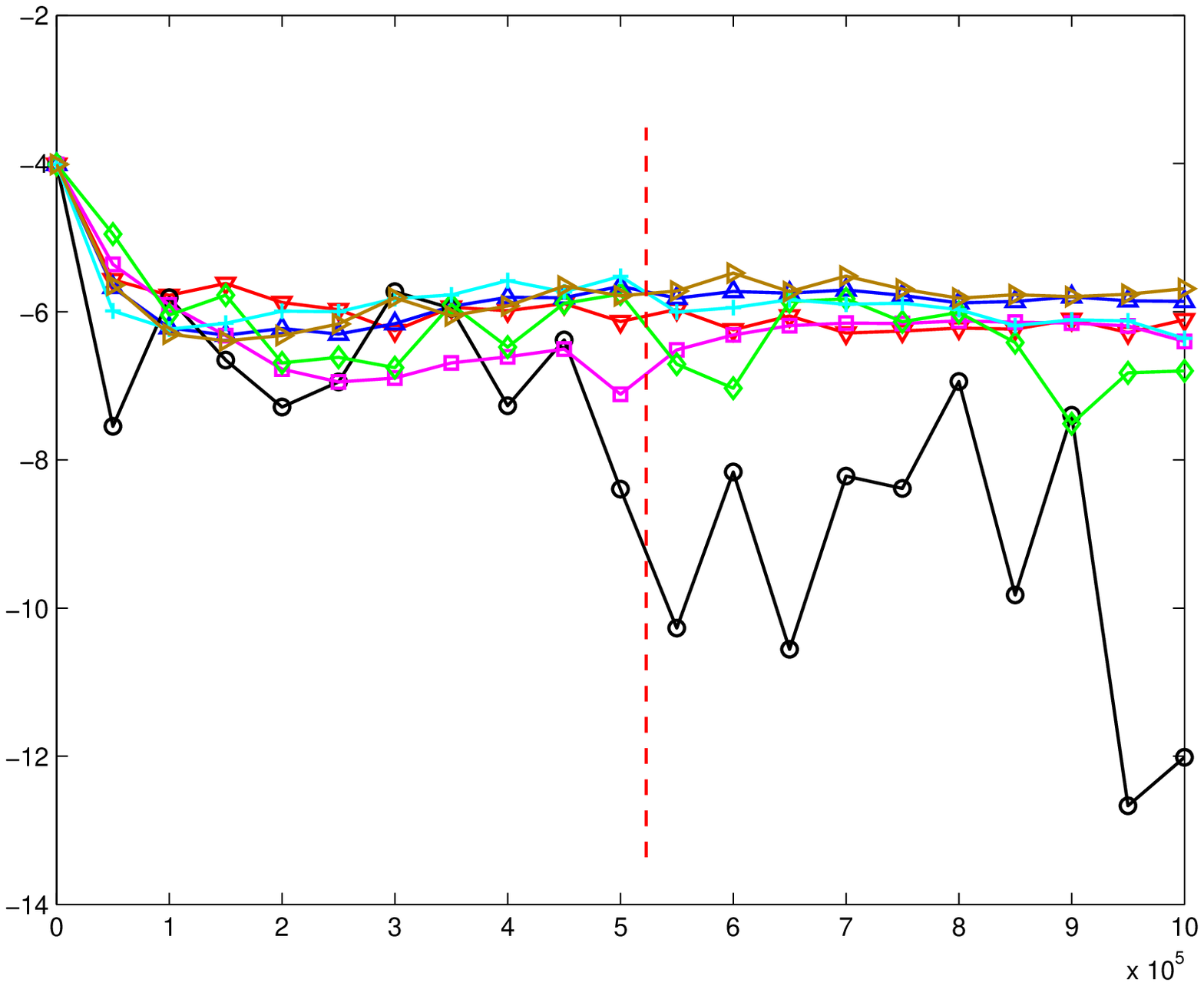} & 
            \\ 
            3. {\sc ijcnn1} &
            4. {\sc covtype} & 
              \\
	  \end{tabular}
          \caption{Suboptimality on the expected risk with regularizer
          $\lambda = n^{-\frac{2}{3}}$}
          \label{fig:results_medium_regualrizer_test}
	\end{center}
	
\end{figure}

\begin{figure}
	\begin{center}
          \begin{tabular}{@{}c@{\hspace{5mm}}c@{\hspace{5mm}}c@{}}
            \includegraphics[width=0.28\linewidth]{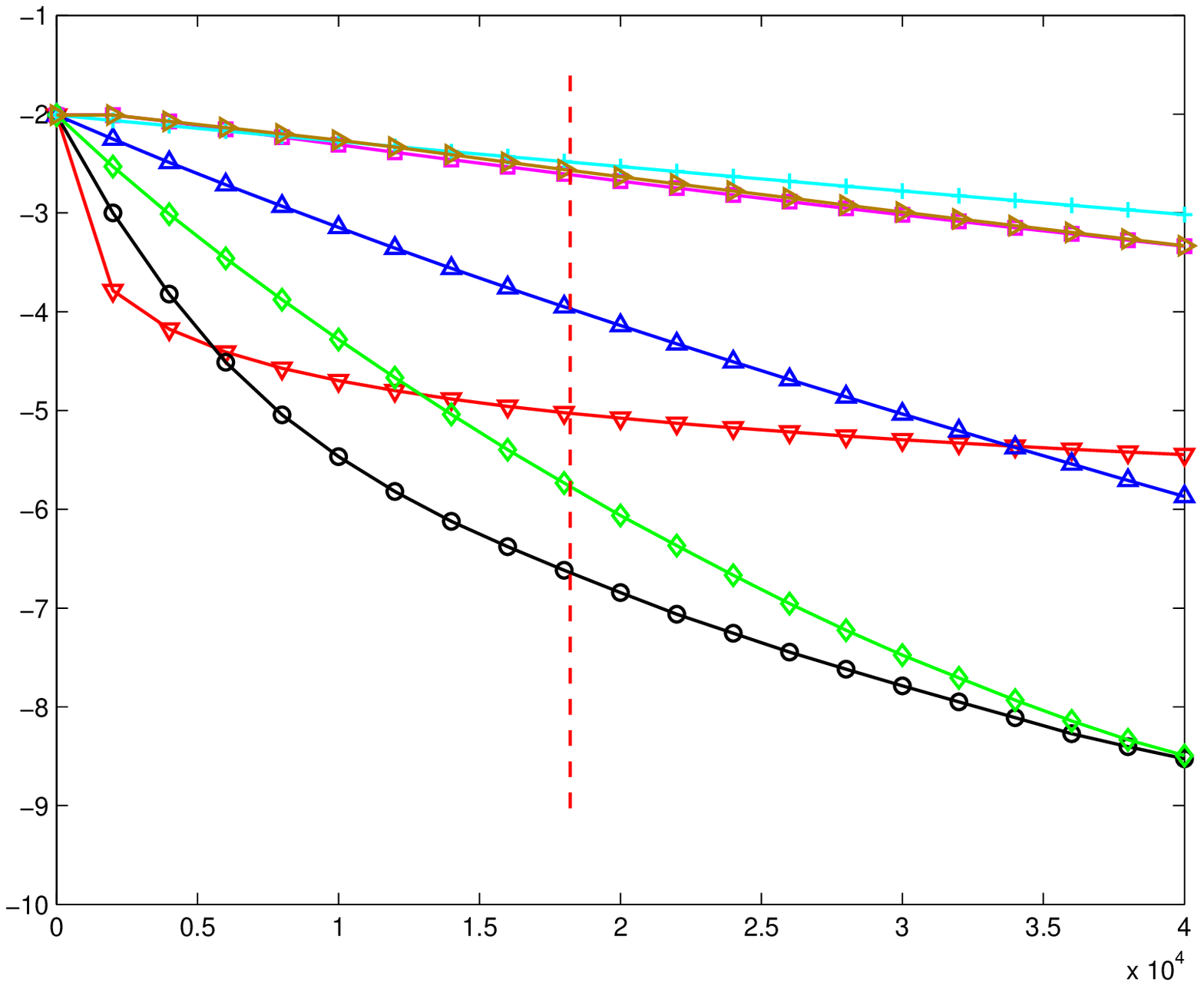} &
            \includegraphics[width=0.28\linewidth]{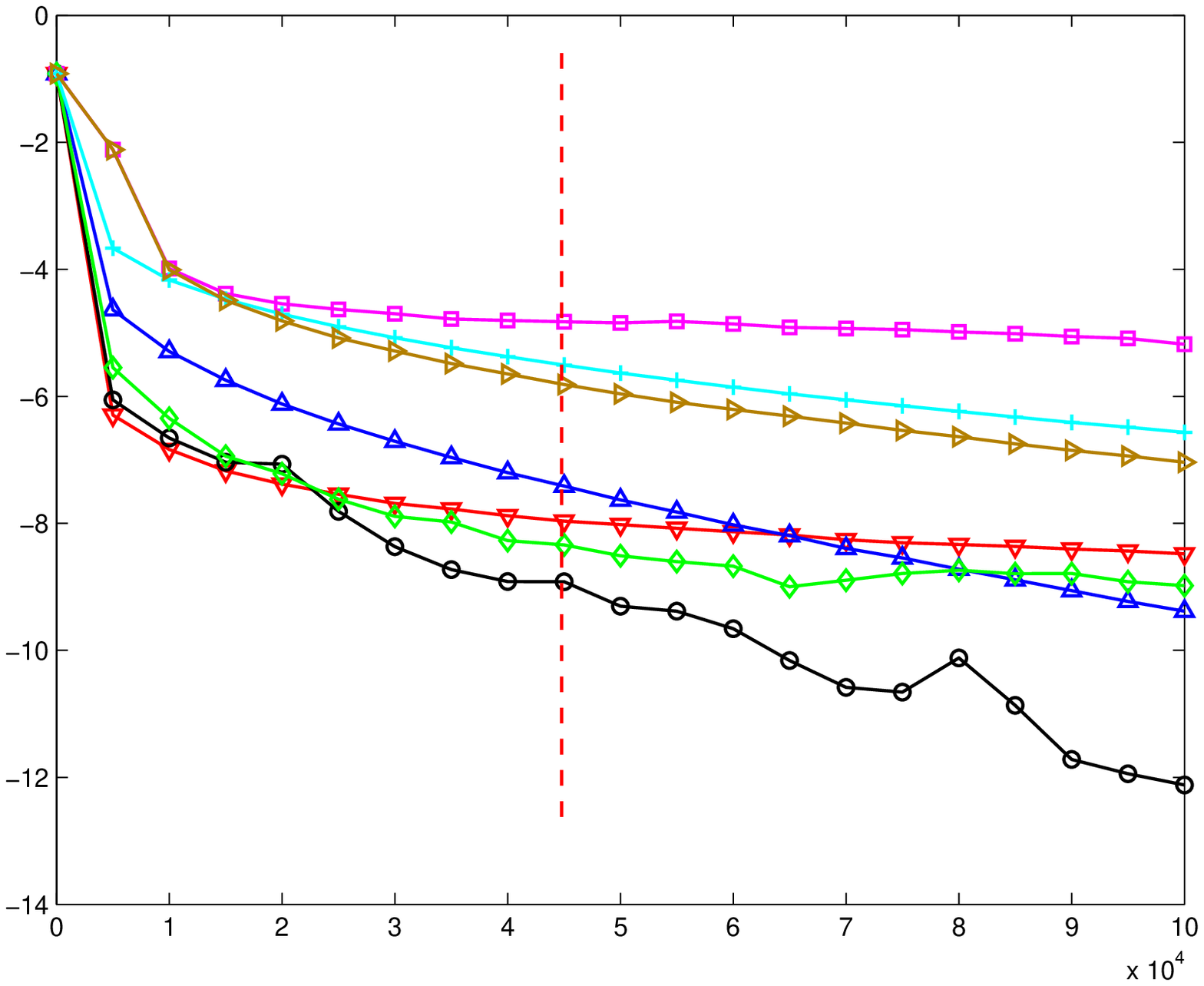} &
            \includegraphics[width=0.2\linewidth]{legend} \\
            1. {\sc rcv} &
            2. {\sc w8a} &
            \\
            \includegraphics[width=0.28\linewidth]{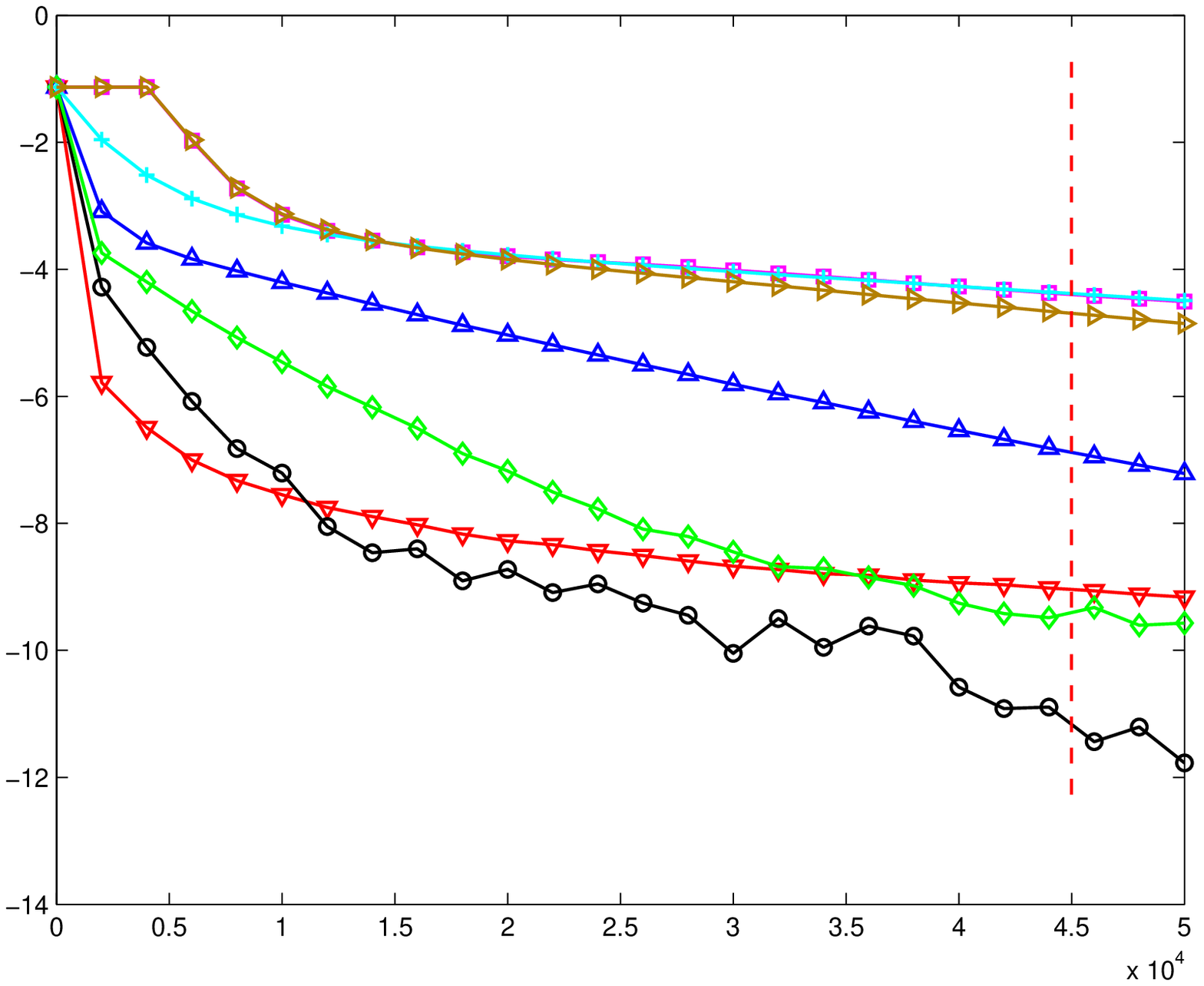}&
            \includegraphics[width=0.28\linewidth]{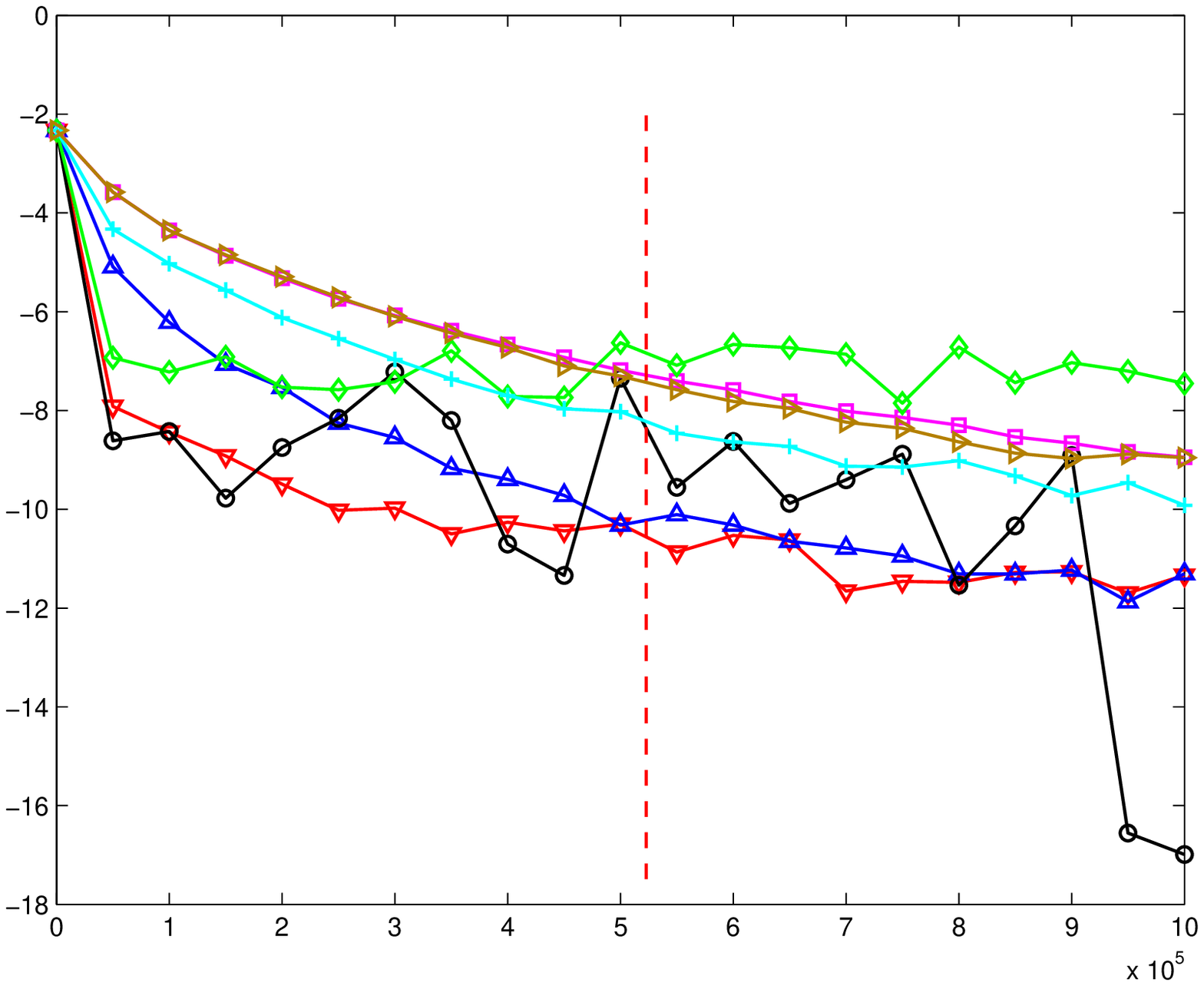} &
            \\ 
            3. {\sc ijcnn1} &
            4. {\sc covtype} & 
            \\
	  \end{tabular}
       \caption{Suboptimality on the empirical risk with regularizer
       $\lambda = n^{-\frac{3}{4}}$}
       \label{fig:results_small_regularizer}
	\end{center}
\end{figure}

\begin{figure}
	\begin{center}
          \begin{tabular}{@{}c@{\hspace{5mm}}c@{\hspace{5mm}}c@{}}
            \includegraphics[width=0.28\linewidth]{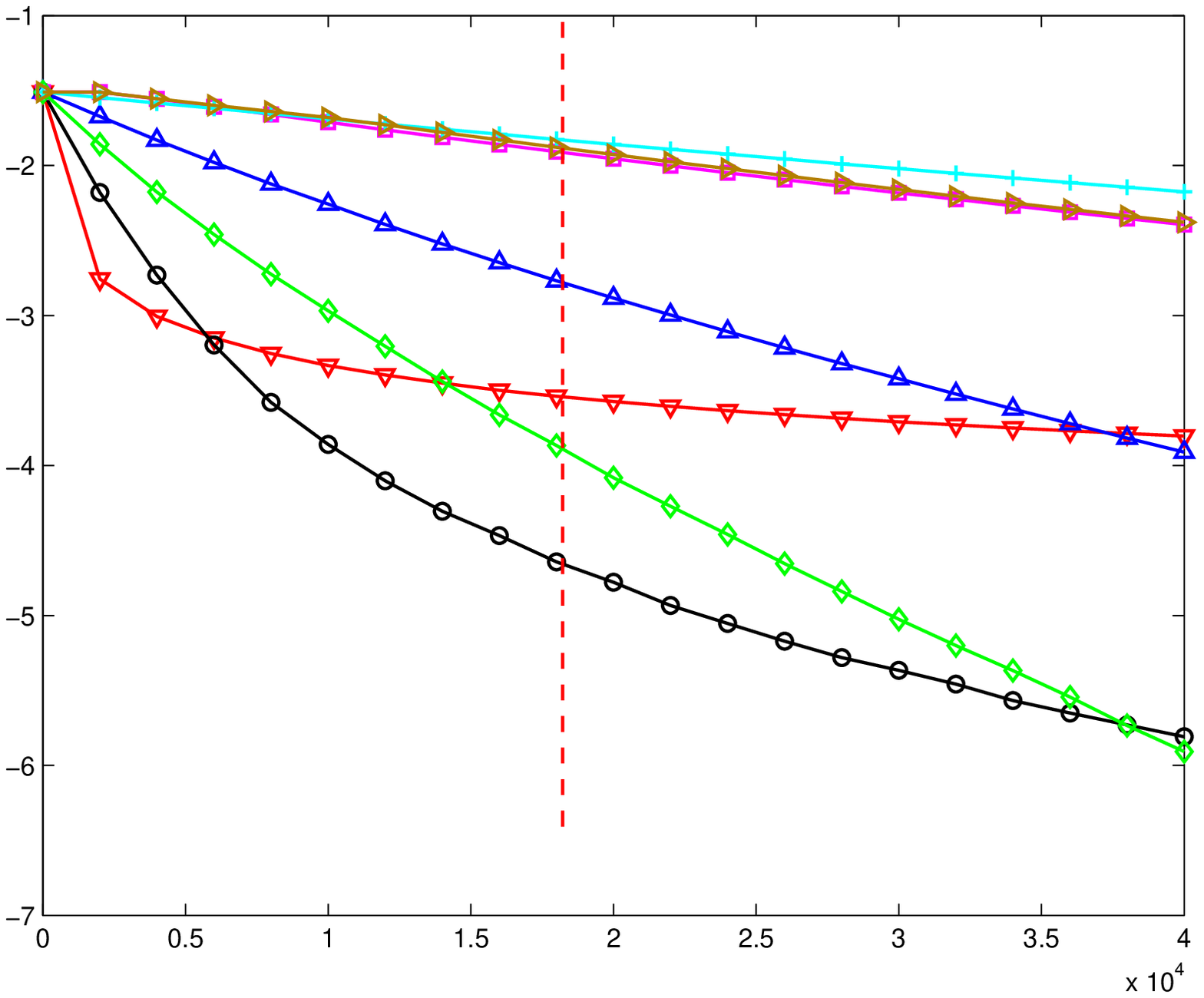} &
            \includegraphics[width=0.28\linewidth]{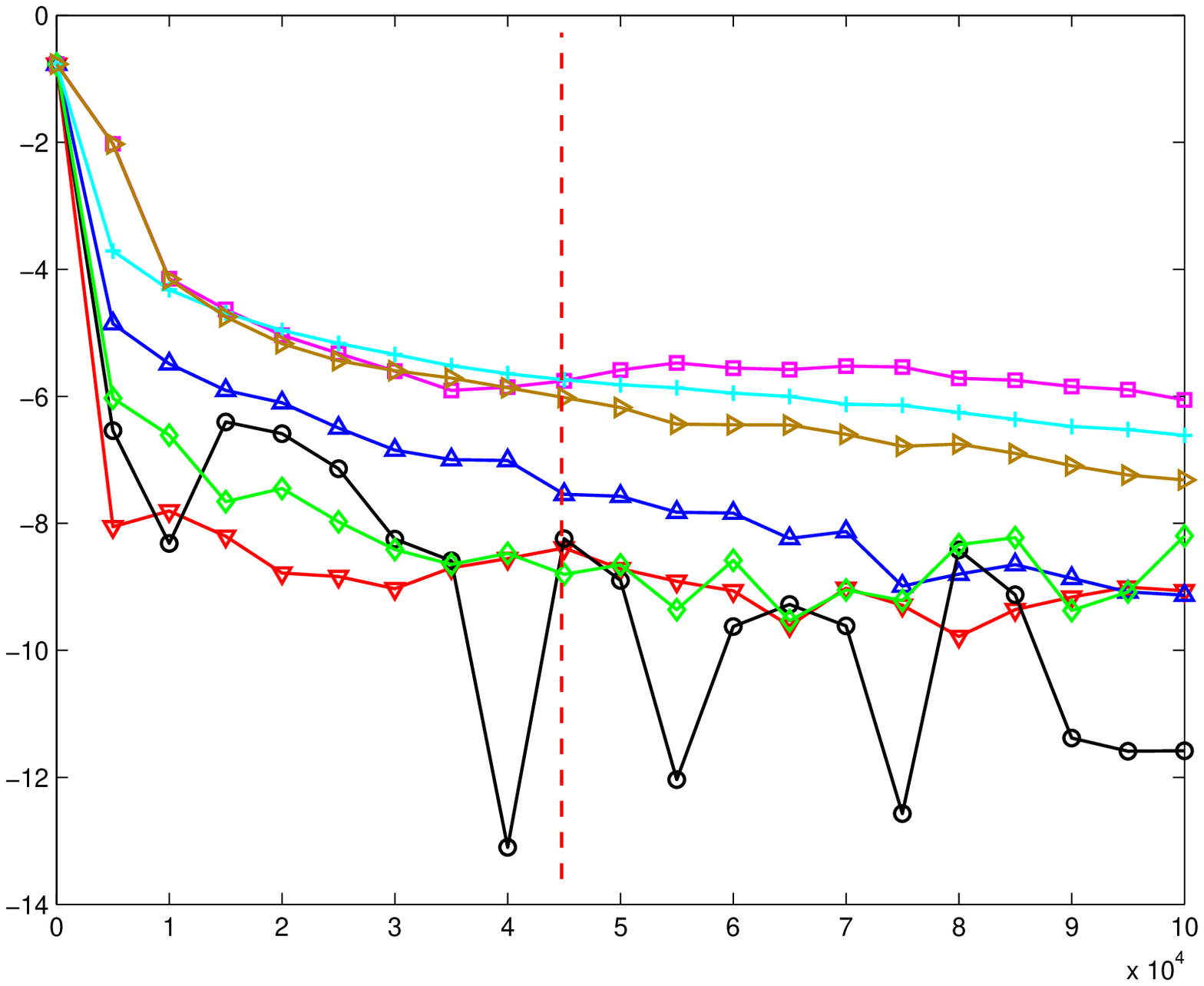} &
            \includegraphics[width=0.2\linewidth]{legend} \\
            1. {\sc rcv} &
            2. {\sc w8a} &
            \\
            \includegraphics[width=0.28\linewidth]{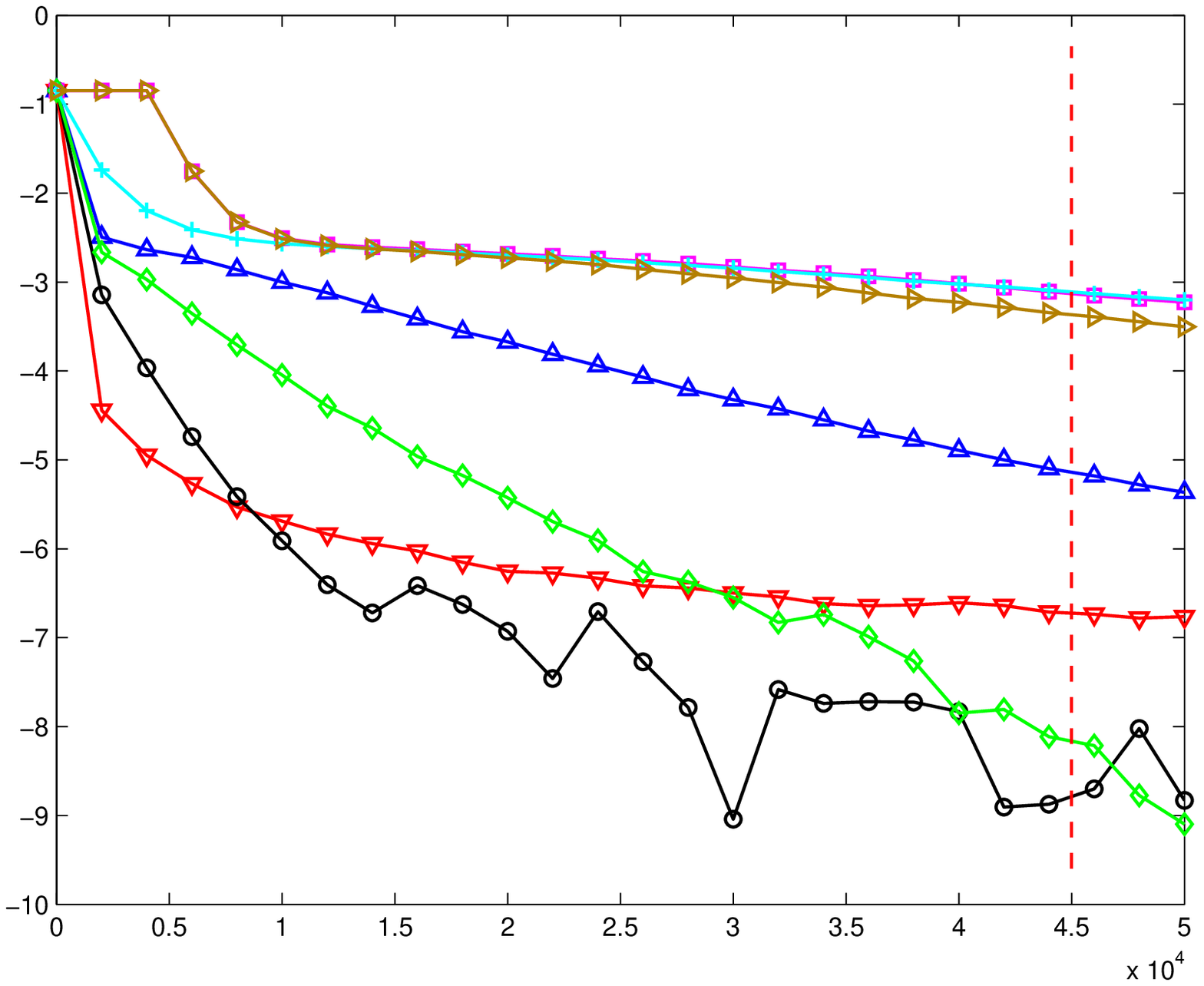}&
            \includegraphics[width=0.28\linewidth]{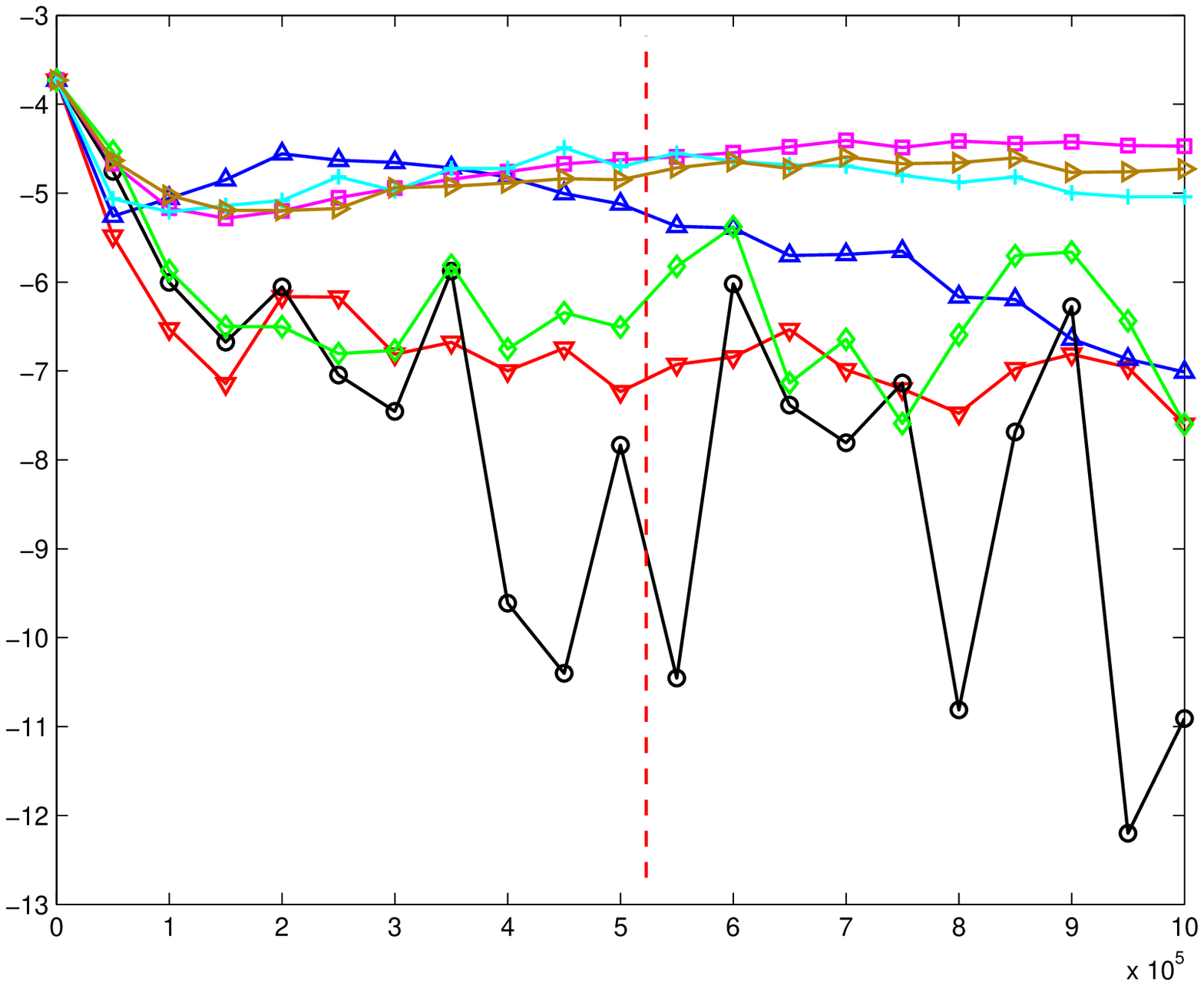} & 
            \\ 
            3. {\sc ijcnn1} &
            4. {\sc covtype} & 
              \\
	  \end{tabular}
          \caption{Suboptimality on the expected risk with regularizer
          $\lambda = n^{-\frac{3}{4}}$}
          \label{fig:results_small_regularizer_test}
	\end{center}
	
\end{figure}

\newpage
\comment{
\begin{figure*}
	\begin{center}
          \begin{tabular}{@{}c@{\hspace{5mm}}c@{\hspace{5mm}}c@{}}
            \includegraphics[width=0.3\linewidth]{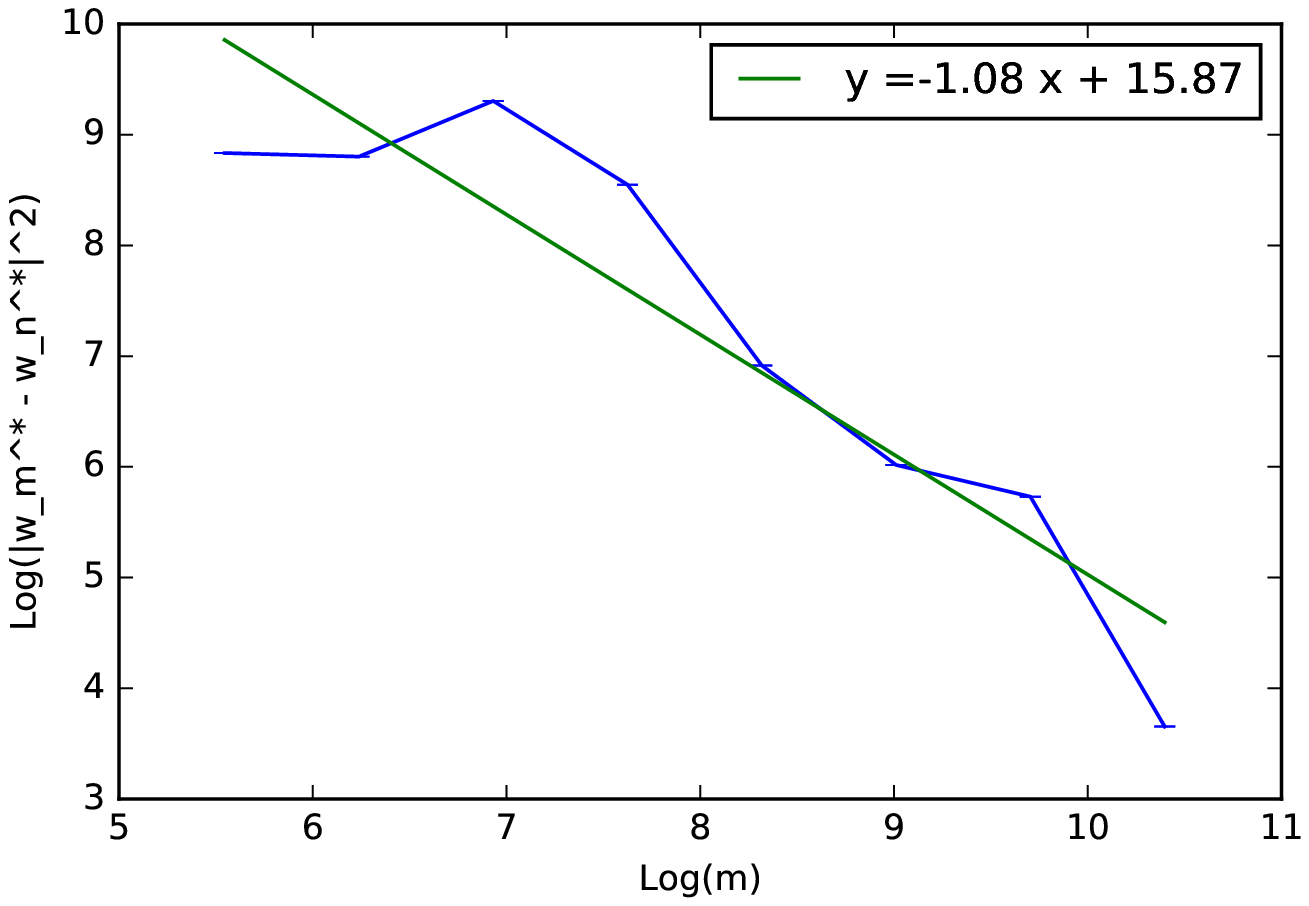} & 
            \includegraphics[width=0.3\linewidth]{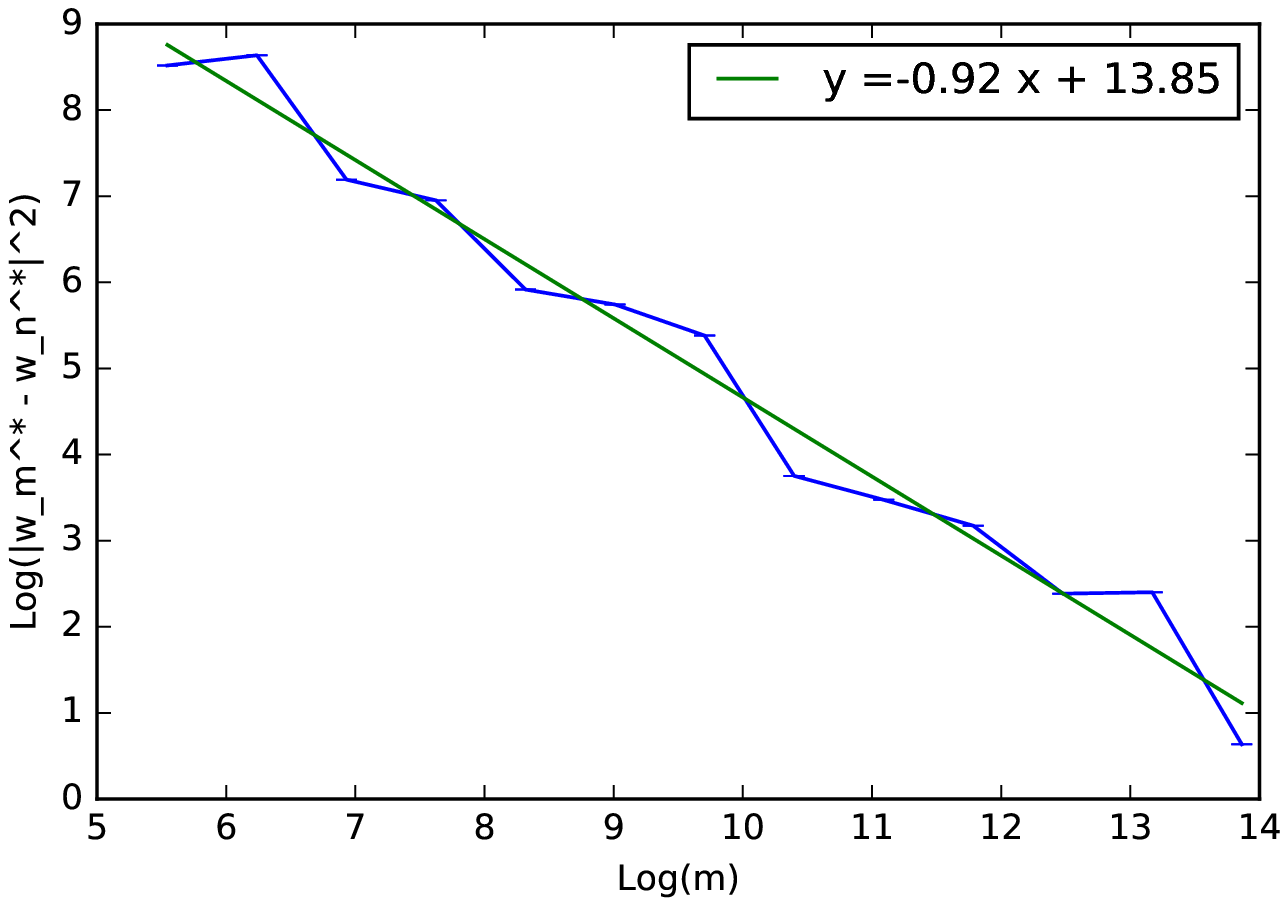} &
            \includegraphics[width=0.3\linewidth]{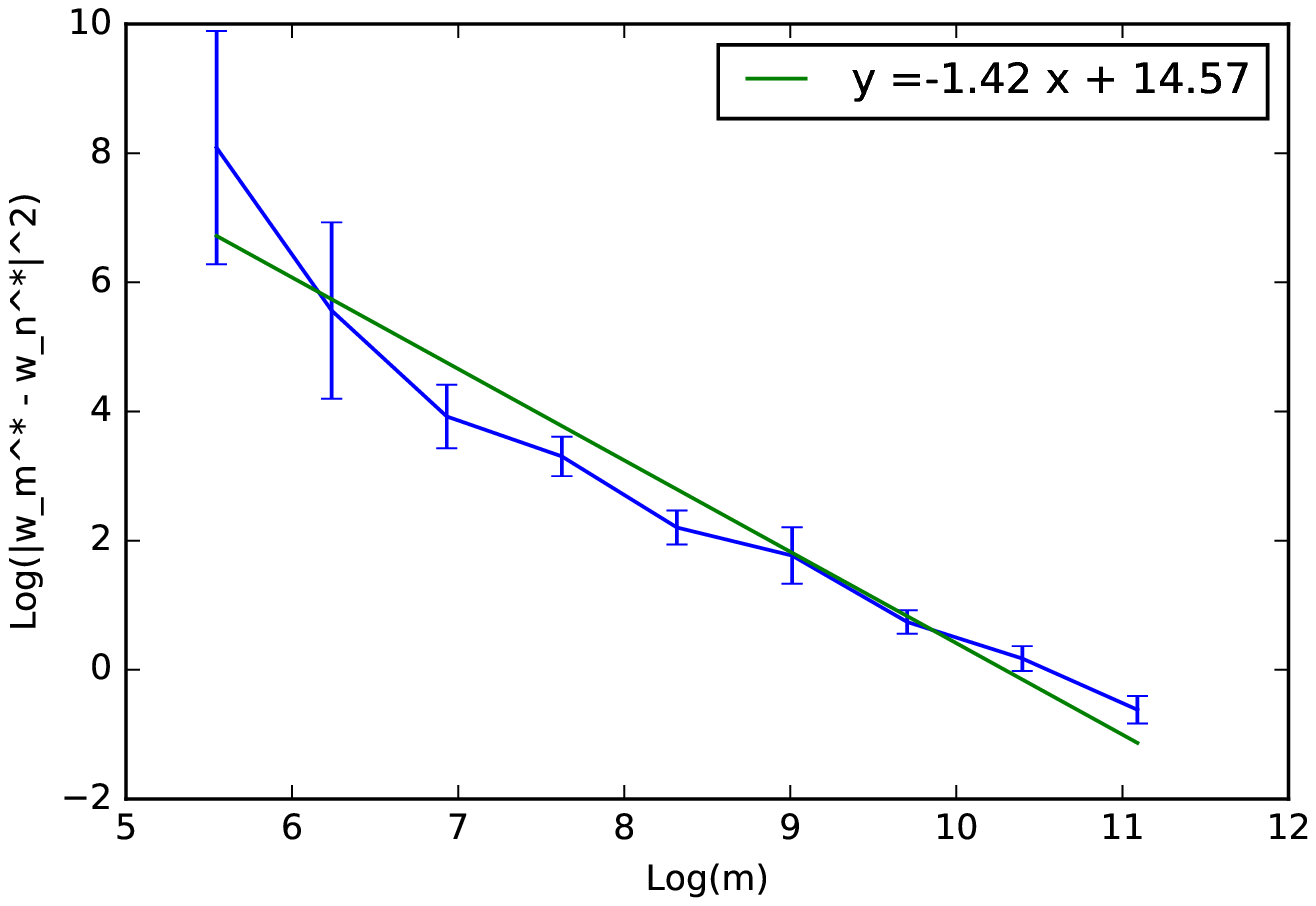} \\
            1. {\sc a9a}  &
            2. {\sc covtype} &
            3. {\sc ijcnn1} 
	  \end{tabular}
          \caption{{\it Sampling test}. We show the suboptimality as we increase the sample size. The results averaged over 10 runs confirm that the rate is close to $1/m$.}
          \label{fig:sampling}
	\end{center}
\end{figure*}
}

\newpage
\subsection{Details of Experiments}

The various parameters of all baselines and \methodname~are represented in
Table~\ref{table:exp_set}.
 \begin{table}[h]
\caption{Experimental setting}
\label{table:exp_set}
\vskip 0.15in
\begin{center}
\begin{small}
\begin{sc}
\begin{tabular}{llcc} 
\hline
\abovespace\belowspace
Method & Parameter & Notation & Value 
\\
\hline
\abovespace
SGD & step size &  $\eta_t$ & $\frac{0.1}{0.1 + \mu t }$\\ 
SAGA  & step size &  $\eta$ & $\frac{0.3}{L + \mu n}$\\
SSVRG and SGD/SVRG& factor for increasing sample size & $b$  & $3$ \\ 
				 & a constant parameter & $p$ & $2$ \\ 
				& step size & $\eta$ & $\frac{1}{10 b^{p}}$ \\ 
				 & initial batch size & $k_0$ & $\kappa$ \\ 
				& number of steps on each batch size & $m$ & $\frac{\kappa}{\eta}$  \\ 

SGD:$0.05$  & step size & $\eta$ & $0.05$ \\ 
SGD:$0.005$  & step size & $\eta$ & $0.005$ \\ 
\methodname & step size for sample size $m$ & $\eta(m)$ & $\frac{0.3}{L + \mu
m}$
\\ 
  & initial batch size & $k_0$& $\kappa$ \\
  & number of iterations on sample size $m$ & $t(m)$ & 2 \\
\belowspace
\\
\hline
\end{tabular}
\end{sc}
\end{small}
\end{center}
\vskip -0.1in
\end{table}

\end{document}